\documentclass{article}

\usepackage{microtype}
\usepackage{graphicx}
\usepackage{subfigure}
\usepackage{booktabs}

\usepackage{cdann_definitions}

\usepackage[utf8]{inputenc}
\usepackage{graphicx}
\usepackage[mathscr]{euscript}

\definecolor{DarkRed}{rgb}{0.75,0,0}
\definecolor{DarkGreen}{rgb}{0,0.5,0}
\definecolor{DarkPurple}{rgb}{0.5,0,0.5}
\definecolor{DarkBlue}{rgb}{0,0,0.7}
\usepackage[bookmarks, colorlinks=true, plainpages = false, citecolor = DarkBlue, urlcolor = DarkBlue, filecolor = black, linkcolor =DarkRed]{hyperref}
\usepackage{breakurl}
\usepackage[ruled, vlined, linesnumbered]{algorithm2e}

\SetCommentSty{mycommfont}

\usepackage{url}

\usepackage[utf8]{inputenc}
\usepackage{amsmath}
\usepackage{graphicx}
\usepackage{upgreek}
\usepackage{amsfonts}
\usepackage{amssymb}
\usepackage{amsthm}
\usepackage[mathscr]{euscript}
\usepackage{mathtools}
\usepackage{xcolor}
\usepackage{xr}
\usepackage{chngcntr}
\usepackage{apptools}
\usepackage[page, header]{appendix}
\usepackage{titletoc}
\usepackage{enumitem}
\usepackage{multicol}
\setlist{itemsep=0.8mm, topsep=0mm, partopsep=0mm, parsep=0mm}
\setlist[itemize]{leftmargin=*}
\setlist[enumerate]{leftmargin=*}
\makeatletter
\newcommand{\removelatexerror}{\let\@latex@error\@gobble}
\makeatother

\usepackage{tikz, pgf,pgfplots}
\usetikzlibrary{calc,3d, patterns}
\usetikzlibrary{arrows,shapes,automata}
\usetikzlibrary{decorations.markings, positioning, fit}

\allowdisplaybreaks[2]
\newcommand{\todo}[1]{}
\newcommand{\prob}{\mathbb P}
\newcommand{\Ex}{\mathbb E}

\newcommand{\statespace}{\mathcal S}
\newcommand{\actionspace}{\mathcal A}
\newcommand{\saspace}{\statespace \times \actionspace}

\newcommand{\numS}{S}
\newcommand{\numA}{A}
\newcommand{\wmin}{w_{\min}}

\newcommand{\polylog}{\operatorname{polylog}}

\newcommand{\reals}{\mathbb R}

\newcommand{\llnp}{\operatorname{llnp}}
\newcommand{\defeq}{:=}
\usepackage{xspace}

\newcommand{\ubev}{\texttt{UBEV}\xspace}
\newcommand{\ofulsic}{\texttt{ORLC-SI}\xspace}
\newcommand{\ulcr}{\texttt{ORLC}\xspace}

\newcommand{\fp}{^{(p)}}
\newcommand{\fr}{^{(r)}}

\newcommand{\range}{\operatorname{rng}}

\usepackage{accents}
\newcommand{\Vub}{\tilde V}
\newcommand{\Vlb}{\underaccent{\tilde}{V}}
\newcommand{\Qub}{\tilde Q}
\newcommand{\Qlb}{\underaccent{\tilde}{Q}}
\newcommand{\psiub}{\tilde \psi}
\newcommand{\psilb}{\underaccent{\tilde}{\psi}}

\usepackage[nohyperref, accepted]{icml2019}

\icmltitlerunning{Policy Certificates: Towards Accountable Reinforcement Learning}
\begin{document}

\twocolumn[
\icmltitle{Policy Certificates: Towards Accountable Reinforcement Learning}

\icmlsetsymbol{equal}{*}

\begin{icmlauthorlist}
\icmlauthor{Christoph Dann}{cmu}
\icmlauthor{Lihong Li}{goo}
\icmlauthor{Wei Wei}{goo}
\icmlauthor{Emma Brunskill}{stan}
\end{icmlauthorlist}

\icmlaffiliation{cmu}{Carnegie Mellon University}
\icmlaffiliation{goo}{Google Research}
\icmlaffiliation{stan}{Stanford University}

\icmlcorrespondingauthor{Christoph Dann}{cdann@cdann.net}

\icmlkeywords{certificates, reinforcement learning, accountability}

\vskip 0.3in
]

\printAffiliationsAndNotice{}

\begin{abstract}
The performance of a reinforcement learning algorithm can vary drastically during learning because of exploration. Existing algorithms provide little information about the quality of their current policy before executing it, and thus have limited use in high-stakes applications like healthcare. We address this lack of \emph{accountability} by proposing that algorithms output \emph{policy certificates}. These certificates bound the sub-optimality and return of the policy in the next episode, allowing humans to intervene when the certified quality is not satisfactory. We further introduce two new algorithms with certificates and present a new  framework for theoretical analysis that guarantees the quality of their policies and certificates. For tabular MDPs, we show that computing certificates can even improve the sample-efficiency of optimism-based exploration. As a result, one of our algorithms is the first to achieve minimax-optimal PAC bounds up to lower-order terms, and this algorithm also matches (and in some settings slightly improves upon) existing minimax regret bounds.
\end{abstract}

\section{Introduction}

There is increasing excitement around applications of machine learning, but also growing awareness and concerns about fairness, accountability and transparency. Recent research aims to address these concerns but most work focuses on supervised learning and only few results~\citep{jabbari2016fair,joseph2016fairness,kannan2017fairness,raghavan2018externalities} exist on reinforcement learning (RL). 

One challenge when applying RL in practice is that, unlike in supervised learning, the performance of an RL algorithm is typically not monotonically increasing with more data due to the trial-and-error nature of RL that necessitates exploration. Even sharp drops in policy performance during learning are common, e.g., when the agent starts to explore a new part of the state space.  Such unpredictable performance fluctuation has limited the use of RL in high-stakes applications like healthcare, and calls for more \emph{accountable} algorithms that can quantify and reveal their performance online during learning. 

To address this lack of accountability, we propose that RL algorithms output
\emph{policy certificates} in episodic RL. Policy certificates consist of (1) a confidence interval of the algorithm's expected sum of rewards (return) in the next episode (policy return certificates) and (2) a bound on how far from the optimal return the performance can be (policy optimality certificates). Certificates make the policy's performance 
more transparent and accountable, and allow designers to intervene if necessary. 
For example, in medical applications, one would need to intervene unless the policy achieves a certain minimum treatment outcome; in financial applications, policy optimality certificates can be used to assess the potential loss when learning a trading strategy. 
In addition to accountability, we also want RL algorithms to be sample-efficient and quickly achieve good performance. 
To formally quantify accountability and sample-efficiency of an algorithm, we introduce a new framework for theoretical analysis called IPOC. IPOC bounds guarantee that certificates indeed bound the algorithm's expected performance in an episode, and prescribe the rate at which the algorithm's policy and certificates improve with more data. IPOC is stronger than other 
frameworks like regret~\citep{jaksch2010near}, PAC \citep{kakade2003sample} and Uniform-PAC \citep{dann2017unifying}, that only guarantee the cumulative performance of the algorithm, but do not provide bounds for \emph{individual} episodes during learning. IPOC also provides stronger bounds and more nuanced guarantees on per episode performance than KWIK~\cite{li2008knows}. 

A natural way to create accountable and sample-efficient RL algorithms is to combine existing sample-efficient algorithms with off-policy policy evaluation approaches to estimate the return (expected sum of rewards) of the algorithm's policy before each episode.
Existing policy evaluation approaches estimate the return of a fixed policy from a batch of data \citep[e.g.,][]{thomas2015high, jiang2016doubly, thomas2016data}. They provide little to no guarantees when the policy is not fixed but computed from that same batch of data, as is here the case. They also do not reason about the return of the unknown optimal policy which is necessary for providing policy optimality certificates.
We found that by focusing on optimism-in-the-face-of-uncertainty (OFU) based RL algorithms for updating the policy and model-based policy evaluation techniques for estimating the policy returns, we can create sample-efficient algorithms that compute policy certificates on both the current policy's return and its difference to the optimal return. The main insight is that OFU algorithms compute an upper confidence bound on the optimal return from an empirical model when updating the policy. Model-based policy evaluation can leverage the same empirical model to compute a confidence interval on the policy return, even when the policy depends on the data. We illustrate this approach with new algorithms for two different episodic settings. 

Perhaps surprisingly, we show that in tabular Markov decision processes (MDPs) it can be beneficial to explicitly leverage the combination of OFU-based policy optimization and model-based policy evaluation to improve either component. Specifically, computing the certificates can directly improve the underlying OFU approach and knowing that the policy converges to the optimal policy at a certain rate improves the accuracy of policy return certificates. 
As a result, the guarantees for our new algorithm improve state-of-the-art regret and PAC bounds in problems with large horizons and are minimax-optimal up to lower-order terms.

The second setting we consider are finite MDPs with linear side information (context)~\citep{abbasi2014online,hallak2015contextual, modi2018markov}, which is of particular interest in practice. For example, in a drug treatment optimization task where each patient is one episode, context is the background information of the patient which influences the treatment outcome. While one expects the algorithm to learn a good policy quickly for frequent contexts, the performance for unusual patients may be significantly more variable due to the limited prior experience of the algorithm. Policy certificates allow humans to detect when the current policy is good for the current patient and intervene if a certified performance is deemed inadequate. For example, for this health monitoring application, a human expert could intervene to either directly specify the policy for that episode, or in the context of automated customer service, the service could be provided at reduced cost to the customer. 

To summarize, We make the following main contributions:
\begin{enumerate}
    \item We introduce policy certificates and the IPOC framework for evaluating RL algorithms with certificates. Similar to existing frameworks like PAC, it provides formal requirements to be satisfied by the algorithm, here requiring the algorithm to be an efficient learner and to quantify its performance online through policy certificates. 
    \item 
    We provide a new  RL algorithm for finite, episodic MDPs that satisfies this definition, 
    and show that it has stronger, minimax regret and PAC guarantees than prior work. Formally, 
    our sample complexity bound is $\tilde O(\numS \numA H^2 / \epsilon^2 + \numS^2 \numA H^3 / \epsilon)$ vs. prior $\tilde O(\numS \numA H^4 / \epsilon^2 + \numS^2 \numA H^3 / \epsilon)$ \citep{dann2017unifying}, and our regret bound $\tilde O(\sqrt{\numS \numA H^2 T} + \numS^2 \numA H^3)$ improves  
    prior work \citep{azar2017minimax} since it has minimax rate up to log-terms in the dominant term even for long horizons $H > \numS \numA$. 
    \item We introduce a new RL algorithm for finite, episodic MDPs with linear side information that has a cumulative IPOC bound, which is tighter than past results~\cite{abbasi2014online} by a factor of $\sqrt{\numS \numA H}$.
\end{enumerate}

\section{Setting and Notation}
We consider episodic RL problems where the agent interacts with the environment in episodes of a certain length. While the framework for policy certificates applies more breadly, we focus on finite MDPs with linear side information~\citep{modi2018markov,hallak2015contextual,abbasi2014online} for concreteness.  This setting includes tabular MDPs as a special case but is more general and can model variations in the environment across episodes, e.g., because different episodes correspond to treating different patients in a healthcare application. Unlike the tabular special case, function approximation is necessary for efficient learning.
\vspace{-3mm}

\paragraph{Tabular MDPs}
The agent interacts with the MDP in episodes indexed by $k$. Each episode is a sequence $(s_{k, 1}, a_{k, 1}, r_{k, 1}, \dots, s_{k, H}, a_{k, H}, r_{k, H})$ of $H$ states $s_{k,h} \in \statespace$, actions $a_{k,h} \in \actionspace$ and scalar rewards  $r_{k,h} \in [0,1]$. For notational simplicity, we assume that the initial state $s_{k, 1}$ is deterministic. The actions are taken as prescribed by the agent's policy $\pi_k$ and we here focus on deterministic time-dependent policies, i.e., $a_{k,h} = \pi_k(s_{k,h}, h)$ for all time steps $h \in [H] := \{1, 2, \dots H\}$. The successor states and rewards are sampled from the MDP as $s_{k,h+1} \sim P(s_{k,h}, a_{k,h})$ and $r_{k,h} \sim P_R(s_{k,h}, a_{k,h})$. In tabular MDPs the size of the state space $\numS = |\statespace|$ and action space $\numA = |\actionspace|$ are finite.
\vspace{-3mm}

\paragraph{Finite MDPs with linear side information.}
We assume that state- and action-space are finite as in tabular MDPs,
but here the agent essentially interacts with a family of infinitely many tabular MDPs that is parameterized by linear contexts.  At the beginning of episode $k$, two contexts, $x\fr_k \in \RR^{d\fr}$ and $x\fp_k \in \RR^{d\fp}$, are observed and the agent interacts in this episode with a tabular MDP, whose dynamics and reward function depend on the contexts in a linear fashion.  Specifically, it is assumed that the rewards are sampled from $P_R(s,a)$ with means
$r_k(s,a) = (x\fr_k)^\top \theta\fr_{s,a}$ and transition probabilities are $ P_k(s' | s,a) = (x\fp_k)^\top \theta\fp_{s', s,a}$ where $\theta\fr_{s,a} \in \RR^{d\fr}$ and $\theta\fp_{s', s, a} \in \RR^{d\fp}$ are unknown parameter vectors for each $s, s' \in \statespace, a \in \actionspace$.  As a regularity condition, we assume bounded parameters, i.e., $\|\theta\fr_{s,a}\|_2 \leq \xi_{\theta\fr}$ and $\|\theta\fp_{s', s, a}\|_2 \leq \xi_{\theta\fp}$ as well as bounded contexts $\|x\fr_k\|_2 \leq \xi_{x\fr}$ and  $\|x\fp_k\|_2 \leq \xi_{x\fp}$.  We allow $x_k\fr$ and $x_k\fp$ to be different, and use $x_k$ to denote $(x_k\fr, x_k\fp)$ in the following.
Note that our framework and algorithms can handle adversarially chosen contexts.
\vspace{-3mm}

\paragraph{Return and optimality gap.}
The quality of a policy $\pi$ in any episode $k$ is evaluated by the \emph{total expected reward} or \emph{return}:
$    
    \rho_k(\pi) \defeq \Ex\left[ \sum_{h=1}^H r_{k,h} \big| a_{k,1:H} \sim \pi \right]
    $,
where this notation means that all actions in the episode are taken as prescribed by a policy $\pi$. Optimal policy and return $\rho^\star_k = \max_{\pi} \rho_k(\pi)$ may depend on the episode's contexts. The difference of achieved and optimal return is called \emph{optimality gap} $\Delta_k = \rho^\star_k - \rho_k(\pi_k)$ for each episode $k$ where $\pi_k$ is the algorithm's policy in that episode. 
\vspace{-3mm}

\paragraph{Additional notation.} 
We denote the largest possible optimality gap by $\Delta_{\max} = H$, and the value functions of $\pi$ in episode $k$ by $Q^{\pi_{k}}_{h}(s,a) = \Ex[\sum_{t=h}^H r_{k,t} | a_{k,h} = a, a_{k,h+1:H} \sim \pi]$ and $V^{\pi_k}_{h}(s) = Q^{\pi_{k}}_{h}(s, \pi(s, h))$. 
Optimal versions are marked by superscript $\star$ and subscripts are omitted when unambiguous. We treat $P(s,a)$ as a linear operator, that is, $P(s,a)f = \sum_{s' \in \statespace} P(s' | s,a) f(s')$ for any $f: \statespace \rightarrow \RR$. We also use $\sigma_{q}(f) = \sqrt{q(f - qf)^2}$ for the standard deviation of $f$ with respect to a state distribution $q$ and $V^{\max}_{h} = (H-h+1)$ for all $h \in [H]$.
We also use the common short hand notation $a \vee b = \max\{a, b\}$ and $a \wedge b = \min\{a, b\}$ as well as $\tilde O(f) = O(f \cdot \operatorname{poly}(\log(f)))$.

\section{The IPOC Framework}
\label{sec:framework}

During execution, the optimality gaps $\Delta_k$ are hidden and the algorithm only observes the sum of rewards which is a sample of $\rho_k(\pi_k)$. This causes risk as one does not know whether the algorithm is playing a good or potentially bad policy.
We introduce a new learning framework that mitigates this limitation. This framework forces the algorithm to output its current policy $\pi_k$ as well as certificates $\epsilon_k \in \RR_+$ and $\Ical_k \subseteq \RR$ before each episode $k$. The \emph{return certificate} $\Ical_k$ is a confidence interval on the return of the policy, while the \emph{optimality certificate} $\epsilon_k$ informs the user how sub-optimal the policy can be for the current context, i.e., $\epsilon_k \geq \Delta_k$. Certificates allow one to intervene if needed. For example, in automated customer services, one might reduce the service price in episode $k$ if certificate $\epsilon_k$ is above a certain threshold, since the quality of the provided service cannot be guaranteed. When there is no context, an optimality certificate upper bounds the sub-optimality of the current policy in any episode which makes algorithms anytime interruptable~\citep{zilberstein1996optimal}: one is guaranteed to always know a policy with improving performance. 
Our learning framework is formalized as follows: 

\begin{definition}[Individual Policy Certificates (IPOC) Bounds]
\label{def:ipoc}
An algorithm satisfies an individual policy certificate (IPOC) bound $F$ if for a given $\delta \in (0,1)$ it outputs the current policy $\pi_k$, a \emph{return certificate}  $\Ical_k \subseteq \RR$ and an \emph{optimality certificate} $\epsilon_k$ with $\epsilon_k \geq | \Ical_k|$ before each episode $k$ (after observing the contexts) so that with probability at least $1-\delta$:
\begin{itemize}
\item[1.] all return certificates contain the return policy $\pi_k$ played in episode $k$ and
all optimality certificates are upper bounds on the sub-optimality of $\pi_k$, i.e., $\forall ~k \in \NN: \quad \epsilon_k \geq \Delta_k$ and $\rho_k(\pi_k) \in \Ical_k$ 
; and either
\item[2a.] for all number of episodes $T$ the cumulative sum of certificates is bounded $\sum_{k=1}^T \epsilon_k \leq F(W, T, \delta)$ (Cumulative Version), or
\item[2b.] for any threshold $\epsilon$, the number of times certificates can exceed the threshold is bounded as
$\sum_{k=1}^\infty \one\{ \epsilon_k > \epsilon\} \leq F(W, \epsilon, \delta)$ (Mistake Version).
\end{itemize}
Here, $W$ can be (known or unknown) properties of the environment.
If conditions 1 and 2a hold, we say the algorithm has a cumulative IPOC bound and if conditions 1 and 2b hold, we say the algorithm has a mistake IPOC bound.
\end{definition}

Condition~1 alone would be trivial to satisfy with $\epsilon_k = \Delta_{\max}$ and $\Ical_k = [0, \Delta_{\max}]$, but condition 2 prohibits this by controlling the size of $\epsilon_k$ (and therefore the size of $|\Ical_k| \leq \epsilon_k$). Condition 2a bounds the cumulative sum of optimality certificates (similar to regret bounds), 
and condition~2b bounds the size of the superlevel sets of $\epsilon_k$ (similar to PAC bounds). We allow both alternatives as condition~2b is stronger but one sometimes can only prove condition~2a (see Appendix~\ref{sec:nomistakebound}\footnote{\label{foot:appendixloc}See full version of this paper at \url{https://arxiv.org/abs/1811.03056}.}).
An IPOC bound controls simultaneously the quality of certificates (how big $\epsilon_k - \Delta_k$ and $|\Ical_k|$ are) as well as the optimality gaps  $\Delta_k$ themselves and, hence, an IPOC bound not only guarantees that the algorithm improves its policy but also becomes better at telling us how well the policy performs. Note that  the condition $\epsilon_k \geq |\Ical_k|$ in Definition~\ref{def:ipoc} is natural as any upper bound on $\rho_k^\star$ is also an upper bound on $\rho_k(\pi_k)$ and is made for notational convenience. 

We would like to emphasize that we provide certificates on the return, the \emph{expected} sum of rewards, in the next episode. Due to the stochasticity in the environment, one in general cannot hope to accurately predict the sum of rewards directly. Since return is the default optimization criteria in RL, certificates for it are a natural starting point and relevant in many scenarios. Nonetheless, certificates for other properties of the sum-of-reward distribution of a policy are an interesting direction for future work. For example, one might want certificates on properties that take into account the variability of the sum of rewards (e.g., conditional value at risk) in high-stakes applications which are often the objective in risk-sensitive RL.
\subsection{Relation to Existing Frameworks}
\label{sec:existing_frameworks}
Unlike IPOC, existing frameworks for RL only guarantee sample-efficiency of the algorithm over multiple episodes and do not provide performance bounds for single episodes during learning. The common existing frameworks are:
\begin{itemize}
    \item \emph{Mistake-style PAC bounds} \citep{strehl2006pac,strehl2009reinforcement,szita2010model,lattimore2012pac,dann2015sample}
     bound the number of $\epsilon$-mistakes, that is, the size of the set $\{k \in \NN ~:~  \Delta_k > \epsilon\}$ with high probability, but do not tell us when mistakes happen. The same is true for the stronger Uniform-PAC bounds~\citep{dann2017unifying} which hold for all $\epsilon$ jointly.
   \item \emph{Supervised-learning style PAC bounds} \citep{kearns2002near,jiang2017contextual, dann2018oracle} ensure that the algorithm outputs an $\epsilon$-optimal policy for a given $\epsilon$, i.e., they ensure $\Delta_k \leq \epsilon$ for $k$ greater than the bound. Yet, they need to know $\epsilon$ ahead of time and tell us nothing about $\Delta_k$ during learning (for $k$ smaller than the bound).
    \item \emph{Regret bounds} \citep{osband2013more, osband2016generalization,azar2017minimax,jin2018q} control the cumulative sum of optimality gaps $\sum_{k=1}^T \Delta_k$ (regret) which does not yield any nontrivial guarantee for individual $\Delta_k$ because it does not reveal which optimality gaps are small.
\end{itemize}

We show that mistake IPOC bounds are stronger than any of the above guarantees, i.e., they imply Uniform PAC, PAC, and regret bounds. Cumulative IPOC bounds are slightly weaker but still imply regret bounds. Both versions of IPOC also ensure that the algorithm is anytime interruptable, i.e., it can be used to find better and better policies that have small $\Delta_k$ with high probability $1 - \delta$. That means IPOC bounds imply supervised-learning style PAC bounds for all $\epsilon$ jointly.
These claims are formalized as follows:
\begin{proposition}
\label{prop:ipoc_properties_cum}
Assume an algorithm has a cumulative IPOC bound $F(W, T, \delta)$.
\begin{enumerate}
    \item Then it has a regret bound of same order, i.e., with probability at least $1 - \delta$, for all $T$ the regret $R(T) := \sum_{k=1}^T \Delta_k$ is bounded by $F(W, T, \delta)$.
    \item If $F$ has the form  $\sum_{p=0}^N (C_p(W, \delta) T)^{\frac{p}{p+1}}$ for appropriate functions $C_p$, then with probability at least $1-\delta$ for any $\epsilon$, it outputs a certificate $\epsilon_k \leq \epsilon$ within
    \vspace{-2mm}
    \begin{align}
        \sum_{p=0}^N \frac{C_{p}(W, \delta)^{p} (N+1)^{p+1}}{\epsilon^{p+1}}
    \end{align}
    episodes. Hence, for settings without context, the algorithm outputs an $\epsilon$-optimal policy within that number of episodes (supervised learning-style PAC bound).
\end{enumerate}
\end{proposition}

\begin{proposition}
\label{prop:ipoc_properties_mis}
If an algorithm has a mistake IPOC bound $F(W, \epsilon, \delta)$, then
\begin{enumerate}
    \item it has a uniform PAC bound $F(W, \epsilon, \delta)$, i.e., with probability at least $1-\delta$, the number of episodes with $\Delta_k \geq \epsilon$ is at most $F(W, \epsilon, \delta)$ for all $\epsilon > 0$;
    \item with probability $ \geq 1-\delta$ for all $\epsilon$, it outputs a certificate $\epsilon_k \leq \epsilon$ within $F(W, \epsilon, \delta) + 1$ episodes. For settings without context, that means the algorithm outputs an $\epsilon$-optimal policy within that many episodes (supervised learning-style PAC).
    \item if $F$ has the form $\sum_{p=1}^N \frac{C_p(W, \delta)}{\epsilon^p} \left(\ln \frac{\tilde C(W, \delta)}{\epsilon}\right)^{np}$ with $C_p(W, \delta) \geq 1$ and constants $N, n \in \NN$, it also has a cumulative IPOC bound of order
    \vspace{-2mm}
    \[
        \!\!\!\!\!\!\!\tilde O\left(\sum_{p=1}^{N} C_{p}(W, \delta)^{1/p} T^{\frac{p-1}{p}} \polylog (\Delta_{\max}, \tilde C(W, \delta), T) \right).
    \]
\end{enumerate}
\vspace{-3mm}
\end{proposition}
The functional form in part~2 of Proposition~\ref{prop:ipoc_properties_cum} includes common polynomial bounds like $O(\sqrt{T})$ or $O(T^{2/3})$ with appropriate factors and similarly for part~3 of Proposition~\ref{prop:ipoc_properties_mis} which covers for example $\tilde O(1/\epsilon^2)$.

Our IPOC framework is similar to KWIK \citep{li2008knows}, in that the algorithm is required to declare how well it will perform. 
Hower, KWIK only requires an algorithm to declare whether the output will perform better than a single pre-specified input threshold. Existing KWIK for RL methods only provide such a binary classification, and have less strong learning guarantees. In a sense IPOC is a  generalization of KWIK.

\section{Algorithms with Policy Certificates}
\label{sec:algo}
A natural path to obtain RL algorithms with IPOC bounds is to combine 
existing provably efficient online RL algorithms with 
an off-policy policy evaluation method to compute a confidence interval on the online RL algorithm's policy for the 
current episode. This yields policy return certificates, but not necessarily policy optimality certificates -- bounds on the difference of the optimal and current policy's return. Estimating the optimal return using off-policy evaluation algorithms in order to compute optimality certificates would
require a significant computational burden, e.g. evaluating all  (exponentially many) policies.

However optimism in the face of uncertainty (OFU) algorithms 
can be modified to provide both policy return 
certificates and optimality certificates without the need 
for a separate off-policy policy optimization step. Specifically, we here consider 
OFU algorithms that maintain an 
upper confidence bound (for a potentially changing confidence level) on the optimal value function $Q^\star_{k,h}$ and therefore optimal return $\rho^\star_k$. This bound is also an upper bound on the return of the current policy which is chosen to maximize this bound. 
Many OFU methods explicitly 
maintain a confidence set of the MDP model to compute the upper confidence bound on $Q^\star_{k,h}$. 
These same confidence sets of the model can be used to compute a lower bound on the value function of the current policy. In doing so, OFU algorithms can be modified with little computational overhead to provide policy return and optimality certificates. 

For these reasons, we focus on OFU methods, introducing two new algorithms with policy certificates, one for tabular MDPs and and one for the more general MDPs with linear side information setting. Both approaches have a similar structure, but leverage different confidence sets and model estimators. In the first case, we show that maintaining lower bounds on the current policy's value has significant benefits beyond enabling policy certificates: lower bounds help us to derive a tighter bound on our uncertainty over the range of future values. Thus we are able to provide the strongest, to our knowledge, PAC and regret bounds for tabular MDPs. It remains an intriguing but non-trivial question if we can create confidence sets that leverage explicit upper and lower bounds for the linear side information setting.

\subsection{Tabular MDPs}

\begin{figure*}[t]
\begingroup
\removelatexerror
\begin{algorithm}[H]
\SetInd{0.7em}{0.5em}
\SetKwInOut{Inputa}{Input}
\Inputa{failure tolerance $\delta \in (0,1]$}
    $\phi(n) = 1 \wedge \sqrt{\frac{0.52}{n}\left( 1.4 \ln \ln(e \vee n) + \ln\frac{26\numS \numA (H + 1 + \numS)}{\delta}\right)};
    \qquad
    \Vub_{k,H+1}(s) = 0; \quad  \Vlb_{k,H+1}(s) = 0 \quad \forall s \in \statespace, k \in \NN
    $\;
    
\For{$k=1, 2, 3, \dots$}{

    \For(\tcp*[f]{update empirical model and number of observations}){$s',s \in \statespace, a \in \actionspace$}
    {
        $n_k(s,a) = \sum_{i=1}^{k-1} \sum_{h=1}^H \one\{s_{i,h} = s, a_{i,h} = a\}$ \tcp*{number of times (s,a) was observed}
        $\hat r_k(s,a) = \frac{1}{n_k(s,a)} \sum_{i=1}^{k-1} \sum_{h=1}^H r_{i,h} \one\{s_{i,h} = s, a_{i,h} = a\}$
        \tcp*{avg. reward observed for (s,a)}
        $\hat P_k(s' | s, a) = \frac{1}{n_k(s,a)} \sum_{i=1}^{k-1} \sum_{h=1}^H \one\{s_{i,h} = s, a_{i,h} = a, s_{i,h+1} = s'\}$
    }
    \For
    (\tcp*[f]{optimistic planning with upper and lower confidence bounds})
    {$h=H$ \KwTo $1$ \textbf{and} $s \in \statespace$}
    {

            \For{$a \in \actionspace$}
            {
            $\psi_{k,h}(s,a) = (1 \!+\! \sqrt{12}\sigma_{\hat P_k(s,a)}(\Vub_{k,h+1}))\phi(n_k(s,a)) 
            + 45 \numS H^2 \phi(n_k(s,a))^2
            \textcolor{DarkRed}{ + \frac 1 {H} \hat P(s,a)( \Vub_{k,h+1} \!-\! \Vlb_{k,h+1})}$\label{lin:psiterm}\;
            
                $\Qub_{k,h}(s,a) = 0 \vee \,\,( \hat r_k(s,a) 
                + \hat P_k(s,a)\Vub_{k,h+1} + \psi_{k,h}(s,a)) \,\,\wedge V^{\max}_h$
                \label{lin:ulcr_ub}
                \tcp*{UCB of $Q^\star_{h+1}$}
                \textcolor{DarkRed}{$\Qlb_{k,h}(s,a) = 0 \vee \,\,(  \hat r_k(s,a) 
                + \hat P_k(s,a)\Vlb_{k,h+1} - \psi_{k,h}(s,a))\,\,\wedge V^{\max}_h$}
                \label{lin:ulcr_lb}
                \tcp*{LCB of $Q^{\pi_k}_{h+1}$}
             }
 $\pi_k(s, h) = \argmax_{a} \Qub_{k,h}(s, a);
 \qquad \Vub_{k,h}(s) = \Qub_{k,h}(s, \pi_k(s, h));
 \qquad \textcolor{DarkRed}{\Vlb_{k,h}(s) = \Qlb_{k,h}(s, \pi_k(s, h))} $\;
}
\textbf{output} policy $\pi_k$  \textcolor{DarkRed}{with certificates $\Ical_k = [\Vlb_{k,1}(s_{1,1}), \Vub_{k,1}(s_{1,1})]$ and $\epsilon_k = |\Ical_k|$}\;
\textbf{sample episode} $k$ with policy $\pi_k$\tcp*{Observe $s_{k,1}, a_{k, 1}, r_{k, 1}, s_{k,2}, \dots, s_{k, H}, a_{k, H}, r_{k, H}$}
}
\caption{\ulcr ({\bf O}ptimistic {\bf R}einforcement {\bf L}earning with  {\bf C}ertificates)}
\label{alg:ulcr}
\end{algorithm}
\vspace{-2mm}
\endgroup
\end{figure*}

\begin{table*}
\begin{center}
\begin{tabular}{|c| c | c | c| }\hline
 Algorithm & Regret & PAC & Mistake IPOC \\ \hline \hline
 \texttt{UCBVI-BF} {\tiny\citep{azar2017minimax}} & $\tilde O(\sqrt{\numS \numA H^2 T} + \sqrt{H^3 T} + \numS^2 \numA H^2)$ & - & - \\  
 Q-l. w/ UCB\textsuperscript{\ref{foot:Hdep}}  {\tiny\citep{jin2018q}}& $\tilde O(\sqrt{\numS \numA H^4 T} + \numS^{1.5} \numA^{1.5} H^{4.5})$  & - & - \\
 \texttt{UCFH} {\tiny \citep{dann2015sample}} & - & $\tilde O\left(\frac{\numS^2 \numA H^2}{\epsilon^2} \right)$& - \\
 \ubev \textsuperscript{\ref{foot:Hdep}} {\tiny \citep{dann2017unifying}} & $\tilde O(\sqrt{\numS \numA H^4 T} + \numS^2 \numA H^3)$ & $\tilde O\left(\frac{\numS \numA H^4}{\epsilon^2} + \frac{\numS^2 \numA H^3}{\epsilon} \right)$ & - \\
 \ulcr (this work)& $\tilde O(\sqrt{\numS \numA H^2 T} + \numS^2 \numA H^3)$ &
 $\tilde O\left(\frac{\numS \numA H^2}{\epsilon^2} + \frac{\numS^2 \numA H^3}{\epsilon} \right)$ & $\tilde O\left(\frac{\numS \numA H^2}{\epsilon^2} + \frac{\numS^2 \numA H^3}{\epsilon} \right)$\\\hline
 Lower bounds & $\Omega\left(\sqrt{\numS \numA H^2 T} \right)$ & $\Omega\left(\frac{\numS \numA H^2}{\epsilon^2} \right)$ & $\Omega\left(\frac{\numS \numA H^2}{\epsilon^2}  \right)$\\\hline
\end{tabular}
\vspace{-2mm}
\end{center}
\caption{Comparison of the state of the art and our bounds for episodic RL in tabular MDPs. A dash means that the algorithm does not satisfy a non-trivial bound without modifications. $T$ is the number of episodes and $\ln(1/\delta)$ factors are omitted for readability. For an empirical comparison of the sample-complexity of these approaches, see Appendix~\ref{sec:tabular_experiments} available at \url{https://arxiv.org/abs/1811.03056}.}
\label{tab:tabbounds}
\end{table*}
We present the \ulcr (optimistic RL with certificates) Algorithm shown in Algorithm~\ref{alg:ulcr} (see the appendix for a version with empirically tighter confidence bounds but same theoretical guarantees). It shares similar structure with recent OFU algorithms like \ubev \citep{dann2017unifying} and \texttt{UCBVI-BF} \citep{azar2017minimax}
but has some significant differences highlighted in red.
Before each episode $k$, Algorithm~\ref{alg:ulcr} computes an optimistic estimate $\Qub_{k,h}$ of $Q^\star_h$ in Line~\ref{lin:ulcr_ub} by dynamic programming on the empirical model $(\hat P_k, \hat r_k)$ with confidence intervals $\psi_{k,h}$. Importantly, it also computes $\Qlb_{k,h}$, a pessimistic estimate of $Q^{\pi_k}_h$ in similar fashion in Line~\ref{lin:ulcr_lb}. The optimistic and pessimistic estimates $\Qlb_{k,h}, \Qub_{k,h}$ (resp. $\Vlb_{k,h}, \Vub_{k,h}$) allow us to compute the certificates $\epsilon_k$ and $\Ical_{k}$ and enables more sample-efficient learning. Specifically, Algorithm~\ref{alg:ulcr} uses a novel form of confidence intervals $\psi$ that explicitly depends on this difference. These confidence intervals are key for proving the following IPOC bound:

\begin{theorem}[Mistake IPOC Bound of Alg.~\ref{alg:ulcr}]
\label{thm:ulcripoc}
For any given $\delta \in (0, 1)$, Alg.~\ref{alg:ulcr} satisfies in any tabular MDP with $\numS$ states, $\numA$ actions and horizon $H$, the following Mistake IPOC bound: For all $\epsilon > 0$, the number of episodes where Alg.~\ref{alg:ulcr} outputs a certificate $|\Ical_k| = \epsilon_k > \epsilon$ is 
\begin{align}
\tilde O\left(\left(\frac{\numS \numA H^2}{\epsilon^2} + \frac{\numS^2 \numA H^3}{\epsilon}\right)\ln \frac 1 \delta \right).
\label{eqn:ipocbound_ulcr}
\end{align}
\vspace{-6mm}
\end{theorem}

By Proposition~\ref{prop:ipoc_properties_mis}, this implies a Uniform-PAC bound of same order as well as the regret and PAC bounds listed in Table~\ref{tab:tabbounds}. This table also contains previous state of the art bounds of each type\footnote{These model-free and model-based methods have the best known bounds in our problem class. Q-learning with UCB and \ubev allow time-dependent dynamics. One might be able to improve their regret bound by $\sqrt{H}$ when adapting them to our setting. Note that by augmenting our state space with a time index, our algorithm also achieves minimax optimality with $\tilde O(\sqrt{\numS \numA H^3 T})$ regret up to lower order terms in their setting.\label{foot:Hdep}} as well as lower bounds.
The IPOC lower bound follows from the PAC lower bound by \citet{dann2015sample} and Proposition~\ref{prop:ipoc_properties_mis}.
For $\epsilon$ small enough ($ \leq O(1/(\numS H))$ specifically), our IPOC bound is minimax, i.e., the best achievable,  up to log-factors. This is also true for the Uniform-PAC and PAC bounds implied by Theorem~\ref{thm:ulcripoc} as well as the implied regret bound when the number of episodes $T = \Omega(\numS^3 \numA H^4)$ is large enough. \ulcr is the first algorithm to achieve this minimax rate for PAC and Uniform-PAC. While \texttt{UCBVI-BF} achieves minimax regret for problems with small horizon, their bound is suboptimal when $H > \numS \numA$. The lower-order term in our regret bound $\tilde O(\numS^2 \numA H^3)$ has a slightly worse dependency on $H$ than \citet{azar2017minimax} but we can trade-off a factor of $H$ for $\numA$ (see appendix) and believe that this term can be further reduced by a more involved analysis.

We defer details of our IPOC analysis to the appendix availble at \url{https://arxiv.org/abs/1811.03056} but the main advances 
leverage that $[\Qlb_{k,h}(s,a), \Qub_{k,h}(s,a)]$ is an \emph{observable} confidence interval for both $Q^\star_{h}(s,a)$ and $Q^{\pi_k}_{h}(s,a)$. Specifically, our main novel insights are:
\begin{itemize}
    \item While prior works~\citep[e.g.][]{lattimore2012pac, dann2015sample} control the suboptimality $Q^\star_{h} - Q^{\pi_k}_h$ of the policy by recursively bounding $\Qub_{k,h} - Q^{\pi_k}_h$, we instead recursively bound $\Qub_{k,h} -\Qlb_{k,h} \leq 2 \psi_{k,h} + \hat P_k (\Vub_{k,h+1} -\Vlb_{k,h+1})$ which is not only simpler but also controls both the suboptimality of the policy and the size of the certificates simultaneously.
    
    \item As existing work \citep[e.g.][]{azar2017minimax, jin2018q}, we use empirical Bernstein-type concentration inequalities to construct $\Qub_{k,h}(s,a)$ as an upper bound to $Q^\star_{h}(s,a) = r(s,a) + P(s,a) V^\star_{h+1}$. This results in a dependency of the upper bound on the variance of the optimal next state value $\sigma_{\hat P_k(s,a)}(V^\star_{h+1})^2$ under the empirical model. Since $V^\star_{h+1}$ is unknown this has to be upper-bounded by $\sigma_{\hat P_k(s,a)}(\Vub_{k,h+1})^2 + B$ with an additional bonus $B$ to account for the difference between the values, $\Vub_{k,h+1} - V^\star_{h+1}$,  which is again unobservable. \citet{azar2017minimax} now constructs an observable bound on $B$ through an intricate regret analysis that involves additional high-probability bounds on error terms (see their $\Ecal_{fr}$/$\Ecal_{az}$ events) which causes the suboptimal $\sqrt{H^3 T}$ term in their regret bound. Instead, we use the fact that  $\Vub_{k,h+1} - \Vlb_{k,h+1}$ is an observable upper bound on $\Vub_{k,h+1} - V^\star_{h+1}$ which we can directly use in our confidence widths $\psi_{k,h}$ (see the last term in Line~\ref{lin:psiterm} of Alg.~\ref{alg:ulcr}). Hence, availability of lower bounds through certificates improves also our upper confidence bounds on $Q^\star_{h}$ and yields more sample-efficient exploration with improved performance bounds as we avoid additional high-probability bounds of error terms.
    
    \item As opposed to the upper bounds, we cannot simply apply concentration inequalities to construct $\Qlb_{k,h}(s,a)$ as a lower bound to $Q^{\pi_k}$ because the estimation target $Q^{\pi_k}(s,a) = r(s,a) + P(s,a)V^{\pi_k}_{h+1}$ is itself random. The policy $\pi_k$ depends in highly non-trivial ways on all samples from which we also estimate the empirical model $\hat P_k, \hat r_k$. A prevalent approach in model-based policy evaluation~\citep[e.g.]{strehl2008analysis, ghavamzadeh2016safe} to deal with this challenge is to instead apply a concentration argument on the $\ell_1$ distance of the transition estimates $\|P(s,a) - \hat P_k(s,a)\|_1 \leq \sqrt{\numS}\phi(n_k(s,a))$. This yields confidence intervals that shrink at a rate of $H\sqrt{\numS}\phi(n_k(s,a))$. Instead, we can exploit that $\pi_k$ is generated by a sample-efficient algorithm and construct $\Qlb_{k,h}$ as a lower bound to the non-random quantity $r(s,a) + P(s,a) V^\star_{h+1}$. We account for the difference $P(s,a)(V^\star_{h+1} - V^{\pi_k}_{h+1}) \leq P(s,a)(\Vub_{k,h+1} - \Vlb_{k,h+1})$ explicitly, again through a recursive bound. This allows us to achieve confidence intervals that shrink at a faster rate of $\psi_{k,h} \approx H \phi(n_k(s,a)) + \numS H^2 \phi(n_k(s,a))^2$ without the $\sqrt{\numS}$ dependency in the dominating $\phi(n_k(s,a))$ term (recall  $\phi(n_k(s,a)) \leq 1$ and goes to 0). Hence, by leveraging that $\pi_k$ is computed by a sample-efficient approach, we improve the tightness of the certificates.
\end{itemize}

\subsection{MDPs With Linear Side Information}

\begin{figure*}[t]
\begingroup
\removelatexerror
\begin{algorithm}[H]
\SetInd{0.7em}{0.5em}
\SetKwInOut{Inputa}{Input}
\Inputa{failure prob. $\delta \in (0,1]$, regularizer $\lambda > 0$}
$\xi_{\theta\fr} = \sqrt{d}; ~ \xi_{\theta\fp} = \sqrt{d}; \qquad
    \Vub_{k,H+1}(s) = 0; \quad  \Vlb_{k,H+1}(s) = 0 \quad \forall s \in \statespace, k \in \NN$\;
$\phi(N, x, \xi) \defeq \left[\sqrt{\lambda}\xi + \sqrt{\frac 1 2 \ln \frac{\numS (\numS\numA + \numA + H)}{\delta} + \frac 1 4 \ln \frac{\det N}{\det(\lambda I)}}\right]\|x\|_{N^{-1}}
$
\label{lin:confsize_context}
\;

\For{$k=1, 2, 3, \dots$}{
    Observe current contexts $x_k\fr$ and $x_k\fp$\;
    \For
    (\tcp*[f]{estimate model with least-squares})
    {$s,s' \in \Scal, a \in \Acal$}{
    
        $N_{k,s,a}^{(q)} = \lambda I + \sum_{i=1}^{k-1} \sum_{h=1}^H \one\{s_{i,h} = s, a_{i,h} = a\} x_k^{(q)} (x_k^{(q)})^\top \qquad \textrm{for } q \in \{r, p\}$\;
    $\hat \theta\fr_{k,s,a} = (N\fr_{k,s,a})^{-1} \sum_{i=1}^{k-1}\sum_{h=1}^H \one\{s_{i,h} = s, a_{i,h} = a\} x_k\fr r_{i,h};
    \qquad \qquad \hat r_k(s,a) = 0 \vee (x_k\fr)^\top \hat \theta\fr_{k,s,a} \wedge  1$\;
    $\hat \theta\fp_{s',s,a} = (N\fp_{k,s,a})^{-1} \sum_{i=1}^{k-1} \sum_{h=1}^H \one\{s_{i,h} = s, a_{i,h} = a, s_{i,h+1} = s'\} x_k\fp$\;
     $\hat P_k(s'|s,a) = 0 \vee (x_k\fp)^\top \hat \theta\fp_{k,s',s,a} \wedge 1$\;
    }
    \For
    (\tcp*[f]{optimistic planning with ellipsoid confidence bounds})
    {$h=H$ \KwTo $1$ \textrm{\textbf{and}} $s \in \statespace$}
        {
            \For{$a \in \actionspace$}
            {
            $\psi_{k,h}(s,a) = \|\Vub_{k,h+1}\|_1 \phi(N\fp_{k,s,a},x\fp_k, \xi_{\theta\fp}) + \phi(N\fr_{k, s,a},x\fr_k, \xi_{\theta\fr})$\;
            $\Qub_{k,h}(s,a) = 0 \vee \,\,\, (\hat r_k(s,a) + \hat P_k(s,a)\Vub_{k,h+1} + \psi_{k,h}(s,a)) \,\,\, \wedge V^{\max}_h$
            \tcp*{UCB of $Q^\star_{h+1}$}
            $\Qlb_{k,h}(s,a) = 0 \vee \,\,\, (\hat r_k(s,a) + \hat P_k(s,a)\Vlb_{k,h+1} - \psi_{k,h}(s,a)) \,\,\, \wedge V^{\max}_h$
            \tcp*{LCB of $Q^{\pi_k}_{h+1}$}
             }
 $\pi_k(s, h) = \argmax_{a} \Qub_{k,h}(s, a);
 \qquad \Vub_{k,h}(s) = \Qub_{k,h}(s, \pi_k(s, t)); 
 \qquad \Vlb_{k,h}(s) = \Qlb_{k,h}(s, \pi_k(s, t)) $\;
}
\textbf{output} policy $\pi_k$ with certificates $\Ical_k = [\Vlb_{k,1}(s_{1,1}), \Vub_{k,1}(s_{1,1})]$ and $\epsilon_k = |\Ical_k|$\;
\textbf{sample episode} $k$ with policy $\pi_k$
\tcp*{Observe $s_{k,1}, a_{k, 1}, r_{k, 1}, s_{k,2}, \dots, s_{k, H}, a_{k, H}, r_{k, H}$}
}

\caption{\ofulsic ({\bf O}ptimistic {\bf R}einforcement {\bf L}earning with  {\bf C}ertificates and {\bf S}ide {\bf I}nformation)}
\label{alg:ofucontext}
\end{algorithm}
\endgroup
\end{figure*}

We now present an algorithm for the more general setting with side information, which, for example, allows us to take background information about a customer into account and generalize across different customers.
Algorithm~\ref{alg:ofucontext} gives an extension, called \ofulsic, of the OFU algorithm by \citet{abbasi2014online}. Its overall structure is the same as the tabular Algorithm~\ref{alg:ulcr} but here the empirical model are least-squares estimates of the model parameters evaluated at the current contexts. Specifically, the empirical transition probability $\hat P_k(s' | s, a)$ is $(x\fp_k)^\top \hat \theta_{s', s, a}$ where $\hat \theta_{s',s,a}$ is the least squares estimate of model parameter $\theta_{s',s,a}$. Since transition probabilities are normalized, this estimate is then clipped to $[0,1]$. 
This model is estimated separately for each $(s', s, a)$-triple, but generalizes across different contexts. The confidence widths $\psi_{k,h}$ are derived using ellipsoid confidence sets on model parameters.
We show the following IPOC bound:

\begin{theorem}[Cumulative IPOC Bound for Alg.~\ref{alg:ofucontext} ]
\label{thm:ufolsic_certs}
For any $\delta \in (0,1)$ and regularizer $\lambda >0$, Alg.~\ref{alg:ofucontext} satisfies the following cumulative IPOC bound in any MDP with contexts of dimensions $d\fr$ and $d\fp$ and bounded parameters $\xi_{\theta\fr} \leq \sqrt{d\fp}, \xi_{\theta\fp} \leq \sqrt{d\fp}$. With prob. at least $1 - \delta$ all return certificates contain the return of $\pi_k$ and optimality certificates are upper bounds on the optimality gaps  and their total sum after $T$ episodes is bounded for all $T$ by
\vspace{-2mm}
\begin{align}
    \tilde O \left(\sqrt{\numS^3 \numA H^4 T} \lambda ( d\fp + d\fr)
    \log \frac{\xi^2_{x\fp}  + \xi^2_{x\fr}}{\lambda \delta}\right).
\end{align}
\vspace{-6mm}
\end{theorem}

By Proposition~\ref{prop:ipoc_properties_cum}, this IPOC bound implies a regret bound of the same order which improves on the $\tilde O(\sqrt{d^2 \numS^4 \numA H^5 T \log 1 / \delta})$ regret bound of \citet{abbasi2014online} with $d = d\fp + d\fr$ by a factor of $\sqrt{\numS \numA H}$. While they make a different modelling assumption (generalized linear instead of linear), we believe at least our better $\numS$ dependency is due to using improved least-squares estimators for the transition dynamics \footnote{They estimate $\theta_{s',s,a}$ only from samples where the transition $s, a \rightarrow s'$ was observed instead of all occurrences of $s, a$ (no matter whether $s'$ was the next state).} and can likely be transferred to their setting. 
The mistake-type PAC bound by \citet{modi2018markov} is not comparable because our cumulative IPOC bound does not imply a mistake-type PAC bound.\footnote{An algorithm with a sub-linear cumulative IPOC bound can output a certificate larger than a threshold $\epsilon_k \geq \epsilon$ infinitely often as long as it does so sufficiently less frequently (see Section~\ref{sec:nomistakebound}).}  Nonetheless, loosely translating our result to a PAC-like bound yields $\tilde O\left( \frac{d^2 \numS^3 \numA H^5}{\epsilon^2}\right)$ which is much smaller than their $\tilde O\left( \frac{d^2 \numS \numA H^4}{\epsilon^5}\max\{d^2, \numS^2\} \right)$ bound for  small $\epsilon$.

The confidence bounds in Alg.~\ref{alg:ofucontext} are more general but looser than those for the tabular case of Alg.~\ref{alg:ulcr}. Instantiating the IPOC bound for Alg.~\ref{alg:ofucontext} from Theorem~\ref{thm:ufolsic_certs} for tabular MDPs ($x\fr_k = x\fp_k = 1$) yields $\tilde O(\sqrt{\numS^3 \numA H^4 T})$ which is worse than the cumulative IPOC bound $\tilde O(\sqrt{\numS \numA H^2 T} + \numS^2 \numA H^3)$ of Alg.~\ref{alg:ulcr} implied by Thm.~\ref{thm:ulcripoc} and Prop.~\ref{prop:ipoc_properties_mis}. 

By Prop.~\ref{prop:ipoc_properties_mis}, a mistake IPOC bound is stronger than the cumulative version we proved for Algorithm~\ref{alg:ofucontext}. One might wonder if Alg.~\ref{alg:ofucontext} also satisfies a mistake bound, but in Appendix~\ref{sec:nomistakebound} (at \url{https://arxiv.org/abs/1811.03056}) we show that this is not the case because of its non-decreasing ellipsoid confidence sets. There could be other algorithms with mistake IPOC bounds for this setting, but they they would likely require entirely different confidence sets.

\section{Simulation Experiment}
\label{sec:experiments}

\begin{figure}[t]
    \centering
    \includegraphics[width=.9\linewidth]{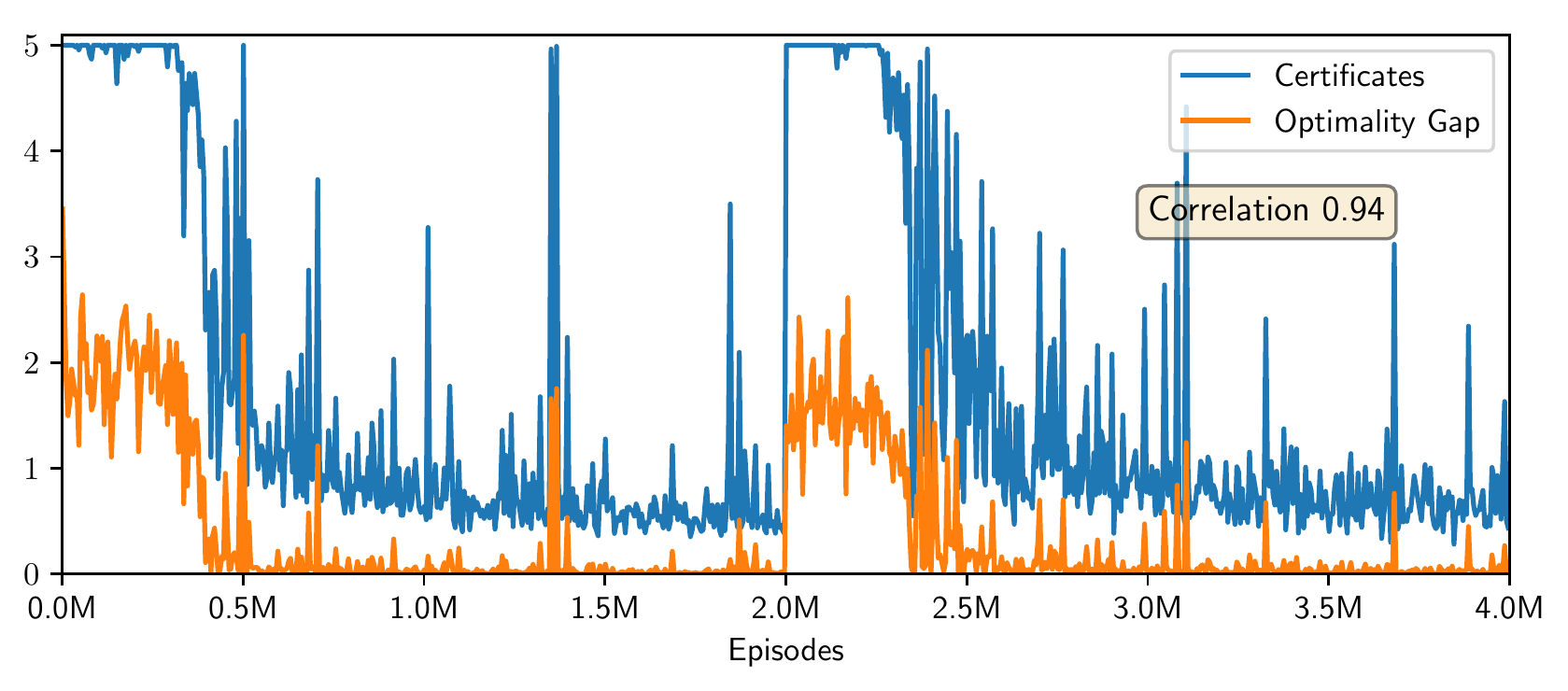}
    \vspace{-4mm}
    \caption{Certificates and (unobserved) optimality gaps of Algorithm~\ref{alg:ofucontext} for 4M episodes on an MDP with context distribution shift after 2M (episodes sub-sampled for better visualization)}
    \label{fig:mdp_res_new}
\end{figure}

One important use case for certificates is to detect sudden performance drops when the distribution of contexts changes. For example, in a call center dialogue system, there can be a sudden increase of customers calling due to a certain regional outage. 
We demonstrate that certificates can identify such performance drops caused by context shifts. We consider a simulated MDP with $10$ states, $40$ actions and horizon $5$ where rewards depend on a $10$-dimensional context and let the distribution of contexts change after $2$M episodes. As seen in Figure~\ref{fig:mdp_res_new}, this causes a spike in optimality gap  as well as in the optimality certificates. 
While our certificates need to upper bound the optimality gap / contain the return in each episode up to a small failure probability, even for the worst case, our algorithm reliably can detect this sudden decrease of performance. In fact, the optimality certificates have a very high correlation of $0.94$ with the unobserved optimality gaps.

One also may wonder if our algorithms leads to improvements over prior approaches in practice or only in the theoretical bounds. To help answer this, we present results in Appendix~\ref{sec:experiemntalsetup} at \url{https://arxiv.org/abs/1811.03056} both on analyzing the policy certificates provided, and examining \ulcr's performance in tabular MDPs versus other recent papers with similar regret~\cite{azar2017minimax} or PAC~\cite{dann2017unifying} bounds. Encouragingly in the small simulation MDPs considered, we find that our algorithms lead to faster learning and better performance.
Therefore while our primary contribution is theoretical results, these simulations suggest the potential benefits of the ideas underlying our proposed framework and algorithms.

\section{Related Work}
\label{sec:relatedwork}
The connection of IPOC to other frameworks is formally discussed in Section~\ref{sec:framework}. 
Our algorithms essentially compute confidence bounds as in OFU methods, and then use those in model-based policy evaluation to obtain policy certificates. 
There are many works on off-policy policy evaluation \citep[e.g.,][]{jiang2016doubly, thomas2016data,mahmood2017multi}, some of which provide non-asymptotic confidence intervals \citep[e.g.,][]{thomas2015high,thomas2015highimp,sajed2018high}. However, these methods focus on the batch setting where 
a set of episodes sampled by fixed policies is given.
Many approaches rely on importance weights that require stochastic  data-collecting policies but most sample-efficient algorithms for which we would like to provide certificates deploy deterministic policies. One could treat previous episodes to be collected by one stochastic data-dependent policy but that introduces bias in the importance-weighting estimators that is not accounted for in the analyses.

Interestingly, there is very recent work~\citep{zanette2019tighter} that also observed the benefits of using lower bounds in optimism-based exploration in tabular episodic RL. Though both their and our work obtain improved theoretical results, the specific forms of the optimistic bonuses are distinct and the analyses differ in many parts (e.g., we provide (Uniform-)PAC and regret bounds instead of only regret bounds). Most importantly, our work provides policy certificate guarantees as a main contribution whereas that work focuses on problem-dependent regret bounds.  

Approaches on safe exploration \citep[][]{kakade2002approximately,pirotta2013safe,thomas2015highimp, ghavamzadeh2016safe} guarantee monotonically increasing performance by operating in a batch loop. Our work is orthogonal, as we are not restricting exploration but rather exposing its impact to the users and give them the choice to intervene.

\section{Conclusion and Future Work}
\label{sec:futurework}

We  introduced policy certificates to improve accountability in RL by enabling users to intervene if the guaranteed performance is deemed inadequate. Bounds in our new theoretical framework IPOC ensure that certificates indeed bound the return and suboptimality in each episode and prescribe the rate at which certificates and policy improve. By combining optimism-based exploration with model-based policy evaluation, we have created two algorithms for RL with policy certificates, including for tabular MDPs with side information. For tabular MDPs, we demonstrated that policy certificates help optimism-based policy learning and vice versa. 
As a result, our new algorithm is the first to achieve minimax-optimal PAC bounds up to lower-order terms for tabular episodic MDPs, and, also the first to have both, minimax PAC and regret bounds, for this setting.

Future areas of interest include scaling up these ideas to continuous state spaces, extending them to model-free RL, and to provide per-episode risk-sensitive guarantees on the reward obtained.

\section*{Acknowledgements}
Part of this work were completed while Christoph Dann was an intern at Google. 
We appreciate support from a Microsoft Faculty Fellowship and a NSF Career award.
\bibliographystyle{icml2019}
\bibliography{thesis_manual}
\appendix
\onecolumn
\appendixpage

\section{Proofs of Relationship of IPOC Bounds to Other Bounds}

\subsection{Proof of Proposition~\ref{prop:ipoc_properties_cum}}

\begin{proof}[Proof of Proposition~\ref{prop:ipoc_properties_cum}]
We prove each part separately:

\textbf{Part 1:} With probability at least $1-\delta$, for all $T$, the regret is bounded as
\begin{align}
    \sum_{k=1}^T \Delta_k \leq \sum_{k=1}^T \epsilon_k \leq F(W, T, \delta)
\end{align}
where the first inequality follows from condition~1 and the second from condition~2a. Hence, the algorithm satisfied a high-probability regret bound $F(W, T, \delta)$ uniformly for all $T$.

\textbf{Part 2:} By assumption, the cumulative sum of certificates is bounded by $F(W, T, \delta) = \sum_{p=0}^N (C_p(W, \delta)T)^{\frac{p}{p+1}}$. Since the minimum is always smaller than the average,  the smallest certificates output in the first $T$ episodes is at most
\begin{align}
    \min_{k \in [T]}\epsilon_k 
    \leq \frac{\sum_{k=1}^T\epsilon_k}{T} 
    \leq \frac{F(W,T, \delta)}{T}
    = \sum_{p=0}^N C_p(W, \delta)^{\frac{p}{p+1}} T^{-\frac{1}{p+1}}.
\end{align}
For $T \geq \frac{C_p(W, \delta)^p (N+1)^{p+1}}{\epsilon^{p+1}}$ we can bound
\begin{align}
C_p(W, \delta)^{\frac{p}{p+1}} T^{-\frac{1}{p+1}}
\leq C_p(W, \delta)^{\frac{p}{p+1}} \left(\frac{C_p(W, \delta)^p (N+1)^{p+1}}{\epsilon^{p+1}}\right)^{-\frac{1}{p+1}} \leq \frac{\epsilon}{N}.
\end{align}
As a result, for $T \geq \sum_{p=0}^N \frac{C_p(W, \delta)^p (N+1)^{p+1}}{\epsilon^{p+1}} \geq \max_{p \in [N] \cup \{0\}} \frac{C_p(W, \delta)^p (N+1)^{p+1}}{\epsilon^{p+1}}$, we can ensure that $\frac{F(W,T, \delta)}{T} \leq \epsilon$, which completes the proof.
\end{proof}

\subsection{Proof of Proposition~\ref{prop:ipoc_properties_mis}}
\begin{proof}[Proof of Proposition~\ref{prop:ipoc_properties_mis}]
We prove each part separately:

\textbf{Part 1:} 

By Definition~\ref{def:ipoc} and the assumption, we have that with probability at least $1 - \delta$ for all $\epsilon > 0$, it holds
\begin{align}
  \sum_{k}^{\infty} \one\{ \Delta_k > \epsilon\} \leq  \sum_{k}^{\infty} \one\{ \epsilon_k > \epsilon\} \leq F(W, \epsilon, \delta),
\end{align}
where the first inequality follows from condition~1 of IPOC and the second from condition~2b. This proves that the algorithm also satisfies a Uniform-PAC bound as defined by \citet{dann2017unifying}.

\textbf{Part 2:} 
Since by definition of IPOC, with probability at least $1-\delta$ for all $\epsilon > 0$, the algorithm can output a certificate $\epsilon_k > \epsilon$ at most $F(W, \epsilon, \delta)$ times.  By the pigeon hole principle, the algorithm has to output at least one certificate $\epsilon_k \leq \epsilon$ in the first $F(W, \epsilon, \delta)+1$ episodes.

\textbf{Part 3:} This part of the proof is based on the proof of Theorem~A.1 in \citet{dann2017unifying}. For convenience, we omit the dependency of $\tilde C$ and $C_p$ on $W$ and $\delta$ in the following.
We assume
\begin{align}
    F(W, \epsilon, \delta) 
    = \sum_{p=1}^N \frac{C_p}{\epsilon^p} \left(\ln \frac{\bar C}{\epsilon}\right)^{np} 
    = \sum_{p=1}^N C_p g(\epsilon)^p
\end{align}
where $\bar C$ is chosen so that for all $p\in [N]$ holds $\bar C^p \geq \Delta_{\max} C_p$ as well as $\bar C \geq \tilde C$.
We also defined $g(\epsilon) := \frac 1 \epsilon \left( \ln \frac {\bar C}{\epsilon}\right)^n$. 
Consider now the cumulative sum of certificates after $T$ episodes.
We distinguish two cases: \\ 
\textbf{Case 1: $T \leq \max_{p \in [N]} \frac{e^p}{\bar C^p} N C_p$.}
Note that $e = \exp(1)$ here.
We use the fact that all certificates are at most $\Delta_{\max}$ and bound
\begin{align}
    \sum_{k=1}^T \epsilon_k \leq \Delta_{\max} T \leq \max_{p \in [N]} \frac{e^p}{\bar C^p} N C_p \Delta_{\max} 
    \leq N e^N
\end{align}
where the final inequality leverages the assumption on $\bar C$.

\textbf{Case 2: $T \geq \max_{p \in [N]} \frac{e^p}{\bar C^p} N C_p$.}
  The mistake bound $F(W, \epsilon, \delta)$ is monotonically decreasing for $\epsilon \in (0,
\Delta_{\max}]$.
If $T$ is large enough, we can therefore find an $\epsilon_{\min} \in (0,
\Delta_{\max}]$ such that $F(W, \epsilon, \delta) \leq T$ for all $\epsilon \in (\epsilon_{\min}, \Delta_{\max}]$.
The cumulative sum of certificates can then be bounded as follows
\begin{align}
    \sum_{k=1}^T \epsilon_k 
    \leq &  T \epsilon_{\min} + \int_{\epsilon_{\min}}^{\Delta_{\max}} F(W, \epsilon, \delta) d\epsilon.
    \label{eqn:intbound1}
\end{align}
This bound assumes the worst case where the algorithm first outputs as many $\epsilon_k = \Delta_{\max}$ as allowed and subsequently smaller certificates as controlled by the mistake bound.

Before further simplifying this expression, we claim that 
\begin{align}
\epsilon_{\min} = \frac{\ln \left(\bar C\min_{p \in [N]} \left(\frac{T}{N C_p} \right)^{1/p} \right)^n}{\min_{p \in [N]} \left(\frac{T}{N C_p} \right)^{1/p}}
\end{align} 
satisfies the desired property $F(W, \epsilon_{\min}, \delta) \leq T$.
To see this, it is sufficient to show that $g(\epsilon_{\min}) \leq \min_{p \in [N]} \left(\frac{T}{N C_p} \right)^{1/p}$, as it implies
\begin{align}
    \sum_{p=1}^N C_p g(\epsilon_{\min})^p
    = \sum_{p=1}^N C_p \min_{p \in [N]} \left(\frac{T}{N C_p} \right)^{p/p}
    \leq \sum_{p=1}^N \frac{T}{N} \frac{C_p}{C_p} = T.
\end{align}
To show the bound on $g(\epsilon_{\min})$, we verify that for any $x \geq \exp(1) /\bar C$
\begin{align}
    g\left( \frac{(\ln (\bar C x))^n}{x}\right)
    = x \frac{\ln\left( \frac{\bar C x}{\ln (x \bar C)^n}\right)^n}{\ln(\bar C x)^n}
    = x~\frac{1}{\ln(\bar C x)^n}\left(\ln(\bar C x) - n\ln (\ln (x \bar C))\right)^n
    \leq x. 
\end{align}
Since $\epsilon_{\min}$ has this form for $x = \min_{p \in [N]} \left(\frac{T}{(N)C_p} \right)^{1/p}$ and $ \min_{p \in [N]} \left(\frac{T}{(N)C_p} \right)^{1/p}  \geq \frac{e}{\bar C}$ by case assumption on $T$, $\epsilon_{\min}$ satisfies the desired property $F(W, \epsilon_{\min}, \delta) \leq T$.

We now go back to Equation~\eqref{eqn:intbound1} and simplify it to
\begin{align}
    \sum_{k=1}^T \epsilon_k 
    \leq &  T \epsilon_{\min} + \int_{\epsilon_{\min}}^{\Delta_{\max}} F(W, \epsilon, \delta) d\epsilon.\\
    = & T \epsilon_{\min} + \sum_{p=1}^N C_p \int_{\epsilon_{\min}}^{\Delta_{\max}} g(\epsilon)^p d\epsilon\\
    = & T \epsilon_{\min} + \sum_{p=1}^N C_p \int_{\epsilon_{\min}}^{\Delta_{\max}} \frac{1}{\epsilon^p}
    \ln \left( \frac{\bar C}{\epsilon} \right)^{np} d\epsilon\\
    \leq & T \epsilon_{\min} + \sum_{p=1}^N C_p \ln \left( \frac{\bar C}{\epsilon_{\min}} \right)^{np}\int_{\epsilon_{\min}}^{\Delta_{\max}} \frac{1}{\epsilon^p}
     d\epsilon\\
    = & T \epsilon_{\min} 
    + C_1 \left( \ln \frac {\bar C}{\epsilon_{\min}}\right)^{n} \ln \frac{\Delta_{\max}}{\epsilon_{\min}}
    + \sum_{p=2}^N \frac{C_p}{1 - p}  \left( \ln \frac {\bar C}{\epsilon_{\min}}\right)^{np}
    \left[ \Delta_{\max}^{1-p} - \epsilon_{\min}^{1-p}\right].
    \label{eqn:intboundterms}
\end{align}
For each term in the final expression, we show that it is $\tilde O\left(\sum_{p=1}^{N} C_{p}^{1/p} T^{\frac{p-1}{p}} \polylog (\bar C T) \right)$. Starting with the first, we bound
\begin{align}
    T \epsilon_{\min} =& \frac{T\ln \left(\bar C\min_{p \in [N]} \left(\frac{T}{N C_p} \right)^{1/p} \right)^n}{\min_{p \in [N]} \left(\frac{T}{N C_p} \right)^{1/p}}
    = \ln \left(\bar C\min_{p \in [N]} \left(\frac{T}{N C_p} \right)^{1/p} \right)^n
    \max_{p \in [N]} \frac{T N^{1/p} C_p^{1/p}}{T^{1/p}}\\
    \leq & \ln \left(\bar C \frac{T}{N C_1} \right)^n N
    \max_{p \in [N]} T^{\frac{p-1}{p}} C_p^{1/p}
    \leq \ln \left(\bar C T \right)^n N
    \max_{p \in [N]} T^{\frac{p-1}{p}} C_p^{1/p}\\
    = & \tilde O\left(\sum_{p=1}^{N} C_{p}^{1/p} T^{\frac{p-1}{p}} \polylog (\bar C T) \right)\,.
\end{align}
For the second term, we start with bounding the inverse of $\epsilon_{\min}$ separately leveraging the case assumption on $T$:
\begin{align}
    \frac 1{\epsilon_{\min}} =& {\min_{p \in [N]} \left(\frac{T}{N C_p} \right)^{1/p}}\frac{1}{\ln \left(\bar C\min_{p \in [N]} \left(\frac{T}{N C_p} \right)^{1/p} \right)^n}
    \leq {\min_{p \in [N]} \left(\frac{T}{N C_p} \right)^{1/p}}\frac{1}{\ln \left(\bar C\min_{p \in [N]} \left(\frac{e^p}{\bar C^p} \right)^{1/p} \right)^n}\\
        \leq&\min_{p \in [N]} \left(\frac{T}{N C_p} \right)^{1/p}
        \leq T.
\end{align}
The second term of Equation~\eqref{eqn:intboundterms} can now be upper bounded by:
\begin{align}
    C_1 \left( \ln \frac {\bar C}{\epsilon_{\min}}\right)^{n} \ln \frac{\Delta_{\max}}{\epsilon_{\min}}
    \leq C_1 \ln (\bar C T)^n \ln (\Delta_{\max} T) \leq C_1 \ln(\bar C T)^{n+1}
    = \tilde O\left(\sum_{p=1}^{N} C_{p}^{1/p} T^{\frac{p-1}{p}} \polylog (\bar C T) \right)
\end{align}
where the last inequality leverages the definition of $\bar C$.
Finally, consider the last term of Equation~\eqref{eqn:intboundterms} for $p > 2$:
\begin{align}
    & \frac{C_p}{1 - p}  \left( \ln \frac {\bar C}{\epsilon_{\min}}\right)^{np}
    \left[ \Delta_{\max}^{1-p} - \epsilon_{\min}^{1-p}\right]
    = \frac{C_p}{p-1}  \left( \ln \frac {\bar C}{\epsilon_{\min}}\right)^{np}
    \left[ \epsilon_{\min}^{1-p} - \Delta_{\max}^{1-p}\right]
    \leq \frac{C_p}{p-1}  \ln (\bar C T)^{np} \epsilon_{\min}^{1-p} \\
    = &\frac{C_p}{p-1}  \ln (\bar C T)^{np} (\epsilon_{\min}^{-1})^{p-1}
    \leq \frac{C_p}{p-1}  \ln (\bar C T)^{np} \left(\frac{T}{N C_p} \right)^{(p-1)/p}
    \leq   \ln (\bar C T)^{np}C_p^{1/p} T^{(p-1)/p}\\
    = &\tilde O\left(\sum_{p=1}^{N} C_{p}^{1/p} T^{\frac{p-1}{p}} \polylog (\bar C T) \right)\,.
\end{align}
Combining all bounds above we obtain that
\begin{align}
\sum_{k=1}^T \epsilon_k \leq \tilde O\left(\sum_{p=1}^{N} C_{p}^{1/p} T^{\frac{p-1}{p}} \polylog (\bar C T) \right)
\leq  \tilde O\left(\sum_{p=1}^{N} C_{p}^{1/p} T^{\frac{p-1}{p}} \polylog (\Delta_{\max}, \tilde C, T) \right)\,.
\end{align}
\end{proof}

\section{Theoretical Analysis of Algorithm~\ref{alg:ulcr} for Tabular MDPs}
To ease the presentation, we chose valid but slightly loose confidence widths $\psi_{k,h}$ in Algorithm~\ref{alg:ulcr}. Below is a version of \ulcr with slightly tighter confidence intervals. It uses different width for upper $\psiub_{k,h}$ and lower $\psilb_{k,h}$ confidence widths and is expected to perform better empirically. The IPOC analysis below applies to both algorithms.

\begin{algorithm}[H]
\SetKwInOut{Inputa}{Input}
\Inputa{failure tolerance $\delta \in (0,1]$}
    $\phi(n) = 1 \wedge \sqrt{\frac{0.52}{n}\left( 1.4 \ln \ln(e \vee n) + \ln\frac{26\numS \numA (H + 1 + \numS)}{\delta}\right)};
    \qquad
    \Vub_{k,H+1}(s) = 0; \quad  \Vlb_{k,H+1}(s) = 0 \quad \forall s \in \statespace, k \in \NN
    $\;
    
\For{$k=1, 2, 3, \dots$}{

    \For(\tcp*[f]{update empirical model and number of observations}){$s',s \in \statespace, a \in \actionspace$}
    {
        $n_k(s,a) = \sum_{i=1}^{k-1} \sum_{h=1}^H \one\{s_{i,h} = s, a_{i,h} = a\}$\;
        $\hat r_k(s,a) = \frac{1}{n_k(s,a)} \sum_{i=1}^{k-1} \sum_{h=1}^H r_{i,h} \one\{s_{i,h} = s, a_{i,h} = a\}$\;
        $\hat P_k(s' | s, a) = \frac{1}{n_k(s,a)} \sum_{i=1}^{k-1} \sum_{h=1}^H \one\{s_{i,h} = s, a_{i,h} = a, s_{i,h+1} = s'\}$
    }
    \For
    (\tcp*[f]{optimistic planning leveraging upper and lower confidence bounds})
    {$h=H$ \KwTo $1$ \textbf{and} $s \in \statespace$}
    {

            \For{$a \in \actionspace$}
            {
            $\psiub_{k,h}(s,a) = \min\bigg\{(V^{\max}_{h+1} + 1)\phi(n_k(s,a)), $\\
            $\quad (1 + \sqrt{12}\sqrt{\sigma^2_{\hat P_k(s,a)}(\Vub_{k,h+1}) + \hat P_{k}(s,a)(\Vub_{k,h+1} - \Vlb_{k,h+1})^2}\phi(n_k(s,a)) + 8.13 V^{\max}_{h+1}\phi(n_k(s,a))^2,$\\
            $\quad (1 + \sqrt{12}\sigma_{\hat P_k(s,a)}(\Vub_{k,h+1}))\phi(n_k(s,a)) + \frac 1 {H} \hat P(s,a)( \Vub_{k,h+1} - \Vlb_{k,h+1})$\\
            $\qquad \qquad \qquad \qquad \qquad \qquad\qquad \qquad \qquad + (20.13 H\|\Vub_{k,h+1} - \Vlb_{k,h+1}\|_1) \phi(n_k(s,a))^2
             \bigg\}$\;
             
            $\psilb_{k,h}(s,a) = \min\bigg\{(2\sqrt{\numS} V^{\max}_{h+1} + 1)\phi(n_k(s,a)), $\\
            $\quad
            \left(V^{\max}_{h+1} + 1 + 2 \sqrt{P_k(s,a)} (\Vub_{k,h+1} - \Vlb_{k,h+1}) \right)\phi(n_k(s,a)) + 4.66\|\Vub_{k,h+1} - \Vlb_{k,h+1}\|_1 \phi(n_k(s,a))^2,$\\
            $\quad
            \left(\sqrt{12}\sqrt{\sigma^2_{\hat P_k(s,a)}(\Vub_{k,h+1}) + \hat P_k(s,a)(\Vub_{k,h+1} - \Vlb_{k,h+1})^2} +1 + 2 \sqrt{P_k(s,a)} (\Vub_{k,h+1} - \Vlb_{k,h+1}) \right)\phi(n_k(s,a)) $\\
            $\qquad \qquad \qquad \qquad \qquad \qquad\qquad \qquad 
            + (8.13 V^{\max}_{h+1} + 4.66\|\Vub_{k,h+1} - \Vlb_{k,h+1}\|_1) \phi(n_k(s,a))^2,$\\
            $\quad
            (1 + \sqrt{12}\sigma_{\hat P_k(s,a)}(\Vub_{k,h+1}))\phi(n_k(s,a)) + \frac 1 {H} \hat P_k(s,a)( \Vub_{k,h+1} - \Vlb_{k,h+1})$\\
            $\qquad \qquad \qquad \qquad \qquad \qquad\qquad \qquad+ (8.13 V^{\max}_{h+1} + (32H + 4.66)\|\Vub_{k,h+1} - \Vlb_{k,h+1}\|_1) \phi(n_k(s,a))^2\bigg\}
             $\;
                $\Qub_{k,h}(s,a) = 0 \vee \,\,( \hat r_k(s,a) 
                + \hat P_k(s,a)\Vub_{k,h+1} + \psiub_{k,h}(s,a)) \,\,\wedge V^{\max}_h$
                \;
                $\Qlb_{k,h}(s,a) = 0 \vee \,\,(  \hat r_k(s,a) 
                + \hat P_k(s,a)\Vlb_{k,h+1} - \psilb_{k,h}(s,a))\,\,\wedge V^{\max}_h$
                \;
             }
 $\pi_k(s, h) = \argmax_{a} \Qub_{k,h}(s, a);
 \qquad \Vub_{k,h}(s) = \Qub_{k,h}(s, \pi_k(s, h));
 \qquad \Vlb_{k,h}(s) = \Qlb_{k,h}(s, \pi_k(s, h))$\;
}
\textbf{output} policy $\pi_k$ with certificate $\epsilon_k = \Vub_1(s_{k,1}) - \Vlb_1(s_{k,1})$\;
\textbf{sample episode} $k$ with policy $\pi_k$\;
}
\caption{\ulcr ({\bf O}ptimistic {\bf R}einforcement {\bf L}earning with  {\bf C}ertificates)}
\label{alg:ulcr_morecomplicated}
\end{algorithm}

To ease the notation in the analysis of \ulcr, we first introduce several helpful definitions:
\begin{align}
    w_{k,h}(s,a) =& \Ex\left[ \one\{s_{k,h}=s, a_{k,h}=a\} ~\bigg| a_{k,1:h} \sim \pi_k, s_{k,1} = s_{k,1}\right]\\
    w_k(s,a) =& \sum_{h=1}^H w_{k,h}(s,a) \\
    \wmin =&  \frac{\epsilon c_\epsilon}{\numS (\numA \wedge H) H} \quad\text{where $c_\epsilon = e^{-6}/4$} \\
    L_{k} =& \{ (s,a) \in \saspace \, : \, w_{k}(s,a) \geq \wmin \}\\
    \llnp(x) =& \ln (\ln (\max\{x, e\}))\\
    \range(x) =& \max(x) - \min(x) \quad\text{for vector $x$} \\
    \delta' =& \frac{\delta}{5 \numS \numA H + 4 \numS \numA + 4 \numS^2 \numA} \\
    \phi(n) =& 1 \wedge 
    \sqrt{ \frac{0.52}{n} \left( 1.4 \llnp\left(2n \right) + \log \frac{5.2}{\delta'}\right)}
    \,.
\end{align}
The proof proceeds in four main steps. First, we define all concentration arguments needed in the form of a failure event and gives an upper bound for its probability. 
We then prove that all value estimates $\Qub$ and $\Qlb$ are indeed optimistic / pessimistic outside the failure event. In a third step, we show a bound on the certificates in the form of a weighted sum of decreasing terms and finally apply a refined pigeon hole argument to bound the number of times this bound can exceed a given threshold.
\subsection{Failure event and all probabilistic arguments}
The failure event is defined as
$F = F^N  \cup F^P \cup F^{PE} \cup F^V \cup F^{VE}  \cup F^{L1} \cup F^{R}$
where
\begin{align}
    F^{R} =& \left\{ \exists ~k, s, a: \, 
        |\tilde r_k(s, a) - r(s, a)| 
        \geq \phi(n_k(s,a))
        \right\}\\
    F^{V} =& \left\{ \exists k ,s, a, h: \, 
        |(\hat P_k(s, a) - P(s, a)) V^{\star}_{h+1}| 
        \geq \range(V^{\star}_{h+1}) \phi(n_k(s,a)) \right\}\\
    F^{VE} =& \bigg\{ \exists k, s, a, h: \, 
     |(\hat P_k(s, a) - P(s, a)) V^{\star}_{h+1}| 
    \geq \sqrt{4 \hat P_k(s,a)[( V^\star_{h+1} - P(s,a)V^\star_{h+1})^2]}
    \phi(n_k(s,a))\\
    &\qquad \qquad \qquad \qquad\qquad \qquad \qquad \qquad \qquad \qquad+
     4.66 \range(V^{\star}_{h+1}) \phi(n_k(s,a))^2\bigg\}\\
    F^{P} =& \bigg\{ \exists ~k, s, s', a: \, | \hat P_k(s' | s, a) - P(s' | s, a)| 
        \geq 
        \sqrt{4 P(s' | s,a)} \phi(n_k(s,a)) + 1.56 \phi(n_k(s,a))^2\bigg\}\\
    F^{PE} =& \bigg\{ \exists ~k, s, s', a: \, | \hat P_k(s' | s, a) - P(s' | s, a)| 
        \geq 
        \sqrt{4 \hat P_k(s' | s,a)} \phi(n_k(s,a)) + 4.66 \phi(n_k(s,a))^2\bigg\}\\
    F^{L1} =& \left\{ \exists ~k, s, a: \,  \|\hat P_k(s, a) - P(s, a)\|_1 
        \geq 2 \sqrt{\numS}\phi(n_k(s,a)) \right\}\\   
    F^{N} =& \left\{ \exists ~k, s, a: \, n_{k}(s, a) < \frac 1 2 \sum_{i < k} w_{i}(s, a) - H\ln \frac{\numS \numA H}{\delta'} \right\}.
\end{align}
The following lemma shows that $F$ has low probability.

\begin{lemma}
    For any parameter $\delta' > 0$, the probability of each failure event is bounded as
    \begin{align}
        \prob\left(  F^V \right) \leq& 2 \numS \numA H \delta' &
        \prob(F^{VE})            \leq& 2 \numS \numA H \delta' &
        \prob(F^{R})             \leq& 2 \numS \numA   \delta' &
        \prob(F^{P})             \leq& 2 \numS^2 \numA \delta' \\
        \prob(F^{PE})            \leq& 2 \numS^2 \numA \delta' &
        \prob(F^{L1})            \leq& 2 \numS   \numA \delta' &
        \prob(F^{N })            \leq& \numS \numA H \delta'.
    \end{align}
    The failure probability is thus bounded by $\prob(F) \leq \delta' (5 \numS \numA H + 4 \numS \numA + 4 \numS^2 \numA ) = \delta$, since we set $\delta' = \frac{\delta}{5 \numS \numA H + 4 \numS \numA + 4 \numS^2 \numA}$.
    \label{lem:failureprob}
\end{lemma}

\begin{proof}
When proving that these failure events indeed have low probability, we need to consider sequences of random variables whenever a particular state and action pair $(s,a)$ was observed. Since the number of times a particular $(s,a)$ was observed as well as in which episodes, is random, we have to treat this carefully. To that end, we first define  $\sigma$-fields $\Gcal^{s,a}_{i}$ which correspond to all observations up to exactly $i$ observations of that $(s,a)$-pair. 
    
Consider a fixed $(s,a) \in \statespace \times \actionspace$, and
denote by $\mathcal F_{(k-1)H+h}$ the sigma-field induced by the first $k-1$
episodes and the $k$-th episode up to $s_{k,h}$ and $a_{k,h}$ but not
$s_{k,h+1}$. Define 
\begin{align}
    \tau_i = \inf\left\{ (k-1)H + h ~:~ \sum_{j=1}^{k} \sum_{t=1}^H \one\{ s_{j,t} = s, a_{j,t} = a\} + \sum_{t=1}^h \one\{ s_{k,t} = s, a_{k,t} = a\} \geq i \right\}
\end{align}
to be the index where $(s,a)$
was observed the $i$th time. Note that $\tau_i$ are
stopping times with respect to $\mathcal F_i$.  Hence, the stopped version $\Gcal^{s,a}_i = \mathcal F_{\tau_i} = \{
A \in \mathcal F_\infty \, : \, A \cap \{\tau_i \leq t \} \in \mathcal F_t \,\,\forall\, t \geq 0\}$
is a filtration as well.
We are now ready to bound the probability of each failure event.

\textbf{Failure event $F^V$: }
For a fixed $s \in \statespace, a \in \actionspace, h \in [H]$, we define $X_i = \frac{1}{\range(V^\star_{h+1})} (V^{\star}_{h+1}(s'_{i}) -  P(s, a) V^{\star}_{h+1}) \one\{\tau_i <
    \infty\}$ where $s'_{i}$ is the value of the successor state when $(s,a)$ was observed the $i$th time (formally $s_{k,j+1}$ with $k = \lfloor \tau_i / H \rfloor$ and $j= \tau_i \mod H$) or arbitrary, if $\tau_i = \infty$). 

    By the Markov property of the MDP, $X_i$ is a martingale
    difference sequence with respect to the filtration $\mathcal G_i^{s,a}$, that is, $\Ex[X_i | \mathcal G^{s,a}_{i-1}] = 0$.
    Furthermore,  it is bounded as 
    \begin{align}
        X_i \in \left[ 
        \frac{\min V^{\star}_{h+1} - P(s, a) V^{\star}_{h+1}}{\range(V^\star_{h+1})}
        , 
        \frac{\max V^{\star}_{h+1} -  P(s, a) V^\star_{h+1}}{\range(V^\star_{h+1})}
        \right]
    \end{align} where
    the range is 
    \begin{align}
        \frac{\max V^{\star}_{t+1} -  P(s, a) V^\star_{t+1}}{\range(V^\star_{t+1})} - \frac{\min V^{\star}_{t+1} -P(s, a) V^{\star}_{t+1}}{\range(V^\star_{t+1})} = \frac{\range V^\star_{t+1}}{\range V^\star_{t+1}} = 1. 
    \end{align}
    Hence, $S_j = \sum_{i=1}^j X_i$ with $V_j = j/4$ satisfies Assumption~1 by \citet{howard2018uniform} (see Hoeffding I entry in Table~2 therein) with any sub-Gaussian boundary $\psi_G$. The same is true for the sequence $-S_k$. Using the sub-Gaussian boundary from Corollary~\ref{cor:subgammaboundary}, we get that with probability at least $1 - 2 \delta'$ for all $n \in \NN$ 
    \begin{align}
        \left|\sum_{i=1}^n X_i \right| \leq  1.44 \sqrt{ \frac{n}{4} \left( 1.4 \llnp(n/2) + \log \frac{5.2}{\delta'}\right)} . \label{eqn:asdad}
    \end{align}
    Since that holds after each observation, this is in particular true before each episode $k+1$ where $(s,a)$ has been observed $n_k(s,a)$ times so far. We can now rewrite the value of the martingale as 
    \begin{align}
        \left|\sum_{i=1}^{n_k(s,a)} X_i \right|
        =&  \frac{1}{\range(V^\star_{h+1})}\left|\sum_{i=1}^{n_k(s,a)}  (V^{\star}_{h+1}(s'_{i}) -  P(s, a) V^{\star}_{h+1}) \right|
        =  \frac{n_k(s,a)}{\range(V^\star_{h+1})}\left|\frac{\sum_{i=1}^{n_k(s,a)}  V^{\star}_{h+1}(s'_{i})}{n_k(s,a)} -  P(s, a) V^{\star}_{h+1} \right|\\
        =&\frac{n_k(s,a)}{\range(V^\star_{h+1})}\left|\hat P_k(s, a) V^{\star}_{h+1} -  P(s, a) V^{\star}_{h+1} \right|
        \label{eq:XVid}
    \end{align}
    and combine this equation with Equation~\eqref{eqn:asdad} to realize that for all $k$
    \begin{align}
        |(\hat P_k(s, a) - P(s, a)) V^{\star}_{h+1}| 
        \leq &\range(V^{\star}_{h+1})  \sqrt{ \frac{0.52}{n_k(s,a)} \left( 1.4 \llnp\left(\frac{n_k(s,a)}{2}\right) + \log \frac{5.2}{\delta'}\right)}\\
        \leq & \range(V^{\star}_{h+1})  \sqrt{ \frac{0.52}{n_k(s,a)} \left( 1.4 \llnp\left(2 n_k(s,a)\right) + \log \frac{5.2}{\delta'}\right)}
    \end{align}
    holds with probability at least $1 - 2 \delta'$. Since in addition $|(\hat P_k(s, a) - P(s, a))^\top V^{\star}_{h+1}| 
        \leq \range(V^{\star}_{h+1})$ at all times, we can bound
        $|(\hat P_k(s, a) - P(s, a))^\top V^{\star}_{h+1}| 
        \leq \range(V^{\star}_{h+1}) \phi(n_k(s,a))$ which shows that $F^V$ has low probability for a single $(s,a,h)$ triple. Applying a union bound over all $h \in [H]$ and $s,a \in \saspace$, we can conclude that $\prob(F^{V}) \leq 2 \numS \numA H \delta'$.

\textbf{Failure event $F^{VE}$: }
    As an alternative to the Hoeffding-style bound above, we can use Theorem~5 by \citet{howard2018uniform} with the sub-exponential bound from Corollary~\ref{cor:subgammaboundary} and the predictable sequence
    $\hat X_i = 0$.
    This gives that with $V_n = \sum_{i=1}^n(X_i - \hat X_i)^2 = \sum_{i=1}^n X_i^2  \leq n$ it holds with probability at least $1 - 2\delta'$ for all $n \in \NN$
        \begin{align}
        \left|\sum_{i=1}^n X_i \right| 
        \leq &  1.44 \sqrt{ V_n \left( 1.4 \llnp(2 V_n) + \log \frac{5.2}{\delta'}\right)} 
        + 2.42  \left(1.4 \llnp(2 V_n) + \log \frac{5.2}{\delta}\right)\\
        \leq & 1.44 \sqrt{ V_n \left( 1.4 \llnp(2n) + \log \frac{5.2}{\delta'}\right)} 
        + 2.42  \left(1.4 \llnp(2n) + \log \frac{5.2}{\delta}\right).
    \end{align}
    Hence, in particular before each episode $k$ when there are $n_k(s,a)$ observations, we have by the identity in Equation~\eqref{eq:XVid} that in the same event as above 
    \begin{align}
    |(\hat P_k(s, a) - P(s, a))^\top V^{\star}_{h+1}| 
        \leq &   \sqrt{\frac{4\range(V^{\star}_{h+1})^2 V_{n_k(s,a)}}{n_k(s,a)}}
        \sqrt{ \frac{0.52}{n_k(s,a)} \left( 1.4 \llnp\left(2 n_k(s,a)\right) + \log \frac{5.2}{\delta'}\right)}\\
        &+  \frac{2.42 \range(V^{\star}_{h+1})}{0.52} \frac{0.52}{n_k(s,a)} \left(1.4 \llnp(2n_k(s,a)) + \log \frac{5.2}{\delta}\right).
    \end{align}
Similar to Equation~\eqref{eq:XVid}, the following identity holds
\begin{align}
        \frac{4\range(V^{\star}_{h+1})^2 V_{n_k(s,a)}}{n_k(s,a)}
        = 4 \hat P_k(s,a)[( V^\star_{h+1} - P(s,a)V^\star_{h+1})^2]
    \end{align}
    and since $|(\hat P_k(s, a) - P(s, a)) V^{\star}_{h+1}| \leq 4.66 \range(V^{\star}_{h+1})$ at any time
        \begin{align}
    |(\hat P_k(s, a) - P(s, a))^\top V^{\star}_{h+1}| \leq \sqrt{4 \hat P_k(s,a)[( V^\star_{h+1} - P(s,a)V^\star_{h+1})^2]} \phi(n_k(s,a))^2 + 4.66 \range(V^{\star}_{h+1}) \phi(n_k(s,a))^2.
    \end{align}
    This shows that $F^{VE}$ has probability at most $1 - 2 \delta'$ for a specific $(s,a,t)$ triple. Hence, with an appropriate union bound, we get the desired bound $\prob(F^{VE}) \leq 2 \numS \numA H \delta'$.
    
\textbf{Failure event $F^R$: }
    For this event, we define $X_i$ as $X_i = (r'_i -  r(s,a)) \one\{\tau_i <
    \infty\}$ where $r'_{i}$ is the immediate reward when $s,a$ was observed the $i$th time (formally $r_{j,l}$ with $j = \lfloor \tau_i / H \rfloor$ and $l= \tau_i \mod H$ or arbitrary (e.g. $1$), if $\tau_i = \infty$).
    
    Similar to above, $X_i$ is a martingale w.r.t. $\Gcal_i^{s,a}$ and by assumption is bounded as $X_i \in [-r(s,a), 1 - r(s,a)]$, i.e., has a range of $1$. Therefore, $S_n = \sum_{i=1}^n X_i$ under the current definition with $V_n = n/4$ satisfies Assumption~1 by \citet{howard2018uniform} and  Corollary~\ref{cor:subgammaboundary} gives that with probability at least $1 - 2 \delta'$ for all $n \in \NN$ 
    \begin{align}
        \left|\sum_{i=1}^n X_i \right| \leq  1.44 \sqrt{ \frac{n}{4} \left( 1.4 \llnp(n/2) + \log \frac{5.2}{\delta'}\right)} .
    \end{align}
    Identical to above, this implies that with probability at least $ 1- 2\delta'$ for all episodes $k \in \NN$ it holds that
    $|\hat r_{k}(s, a) - r(s,a)| \leq  \phi(n_k(s,a))$ for this particular $s,a$. Applying a union bound over $\saspace$ finally yields that $\prob(F^R) \leq 2 \numS \numA \delta'$.
    
\textbf{Failure event $F^P$:}
    In addition to $s,a$, consider a fixed $s' \in \statespace$.
    We here define $X_i$ as $X_i = (\one\{s' = s'_i\} -  P(s'|s,a)) \one\{\tau_i <
    \infty\}$ where $s'_{i}$ is the successor state when $s,a$ was observed the $i$th time (formally $s_{k,j}$ with $k = \lfloor \tau_i / H \rfloor$ and $j= \tau_i \mod H$) or arbitrary, if $\tau_i = \infty$).
    By the Markov property, $X_i$ is a martingale with respect to $\Gcal_{i}^{s,a}$ and is bounded in $[-1,1]$. 
    
    Hence, $S_n = \sum_{i=1}^n X_i$ with $V_n = \sum_{i=1}^n \Ex[X_i^2 | \Gcal_{i-1}^{s,a}] =  P(s,a) (\one\{s' = \cdot\} -  P(s'|s,a))^2 \sum_{i=1}^n \one\{\tau_i < \infty\}\leq n$ satisfies Assumption~1 by \citet{howard2018uniform} (see Bennett entry in Table~2 therein) with sub-Gaussian $\psi_P$. The same is true for the sequence $-S_n$. Using Corollary~\ref{cor:subgammaboundary}, we get that with probability at least $1 - 2 \delta'$ for all $n \in \NN$ 
        \begin{align}
        \left| S_n \right| 
        \leq &  1.44 \sqrt{ V_n \left( 1.4 \llnp(2 V_n) + \log \frac{5.2}{\delta'}\right)} 
        + 0.81  \left(1.4 \llnp(2 V_n) + \log \frac{5.2}{\delta}\right)\\
        \leq & 1.44 \sqrt{ V_n \left( 1.4 \llnp(2n) + \log \frac{5.2}{\delta'}\right)} 
        + 0.81  \left(1.4 \llnp(2n) + \log \frac{5.2}{\delta}\right).
    \end{align}
    Hence, in particular after each episode $k$, we have in the same event because $S_{n_k(s,a)} = n_k(s,a) (\hat P_k(s' | s, a) - P(s' | s, a))$ that 
    \begin{align}
    | \hat P_k(s' | s, a) - P(s' | s, a)| 
    \leq & 
    1.44 \sqrt{ \frac{V_n}{0.52 n_k(s,a)}} \sqrt{\frac{0.52}{n_k(s,a)} \left( 1.4 \llnp(2n_k(s,a)) + \log \frac{5.2}{\delta'}\right)} \\
        & + \frac{0.81}{0.52} \frac{0.52}{n_k(s,a)}  \left(1.4 \llnp(2n_k(s,a)) + \log \frac{5.2}{\delta}\right).
    \end{align}
    Combining this bound with $| \hat P_k(s' | s, a) - P(s' | s, a)| \leq 1.56$ gives
    $| \hat P_k(s' | s, a) - P(s' | s, a)| 
    \leq 
    \sqrt{ \frac{1.44^2 V_n}{0.52 n_k(s,a)}} \phi(n_k(s,a)) + 1.56 \phi(n_k(s,a))^2$.
     It remains to bound the first coefficient as
    \begin{align}
    \frac{1.44^2 V_n}{0.52 n_k(s,a)} \leq & 4 P(s,a) (\one\{s' = \cdot\} -  P(s'|s,a))^2 = 4 P(s,a) \one\{s' = \cdot\}^2 - 4 P(s' | s,a)^2
    \\
    =& 4 P(s' | s,a) - 4 P(s' | s,a)^2 \leq 4 P(s' | s,a).
    \end{align}
    Hence, for a fixed $s', s, a$, with probability at least $1 - \delta'$ the following inequality holds for all episodes $k$
        \begin{align}
    | \hat P_k(s' | s, a) - P(s' | s, a)| 
\leq & 
    \sqrt{4 P(s' | s, a)} \phi(n_k(s,a)) + 1.56 \phi(n_k(s,a))^2.
    \end{align}
    Applying a union bound over $\saspace \times \statespace$, we get that $\prob(F^{P}) \leq 2 \numS^2 \numA \delta'$.
    
\textbf{Failure event $F^{PE}$:}
    The bound in $F^{P}$ uses the predictable variance of $X_i$ which eventually leads to a dependency on the unknown $P(s' | s, a)$ in the bound. In $F^{PE}$, the bound instead depends on the observed $P_k(s' | s, a)$. To achieve that bound, we use Theorem~5 by \citet{howard2018uniform} in combination with Corollary~\ref{cor:subgammaboundary}, similar to event $F^{VE}$. For the same definition of $X_i$ as in $F^P$, we then get that with probability at least $1 - 2\delta'$ for all $n \in \NN$
    \begin{align}
        \left| S_n \right| 
        \leq &  1.44 \sqrt{ V_n \left( 1.4 \llnp(2 V_n) + \log \frac{5.2}{\delta'}\right)} 
        + 2.42  \left(1.4 \llnp(2 V_n) + \log \frac{5.2}{\delta}\right)\\
        \leq & 1.44 \sqrt{ V_n \left( 1.4 \llnp(2n) + \log \frac{5.2}{\delta'}\right)} 
        + 2.42  \left(1.4 \llnp(2n) + \log \frac{5.2}{\delta}\right),
    \end{align}
    where $V_n = \sum_{i=1}^n \left((X_i + P(s'|s,a)) \one\{\tau_i <
    \infty\}\right)^2 \leq n$ (that is, we choose the predictable sequence as $\hat X_i = - P(s'|s,a) \one{\tau_i <
    \infty}$). Analogous to $F^{P}$, we have in the same event for all $k$
    \begin{align}
    | \hat P_k(s' | s, a) - P(s' | s, a)| 
    \leq & 
    \sqrt{ \frac{1.44^2 V_n}{0.52 n_k(s,a)}} \phi(n_k(s,a)) + 4.66 \phi(n_k(s,a))^2
    \end{align}
    and the first coefficient can be written as
    \begin{align}
        \frac{1.44^2 V_n}{0.52 n_k(s,a)} \leq \frac{4}{n_k(s,a)} \sum_{i=1}^{n_k(s,a)} (\one\{s' = s'_i\})^2 \one\{\tau_i <
    \infty\})^2 = 4 \hat P_k(s'| s,a).
    \end{align}
    After applying a union bound over all $s', s, a$, we get the desired failure probability bound
    $\prob(F^{PE}) \leq 2 \numS^2 \numA \delta'$.

\textbf{Failure event $F^{L1}$:}
    Consider a fix $s \in \statespace, a \in \actionspace$ and $\Bcal \subseteq \statespace$ and define
    $X_i = (\one\{s'_i \in \Bcal\} -  P(s' \in \Bcal|s,a)) \one\{\tau_i <
    \infty\}$. In complete analogy to $F^R$, we can show that with probability at least $1 - 2 \delta'$ the following bound holds for all episodes $k$
    \begin{align}
        |\hat P_k(s' \in \Bcal | s, a) - P(s' \in \Bcal | s, a) | \leq  \phi(n_k(s,a)).
    \end{align}
    We can use this result with $\delta' / 2^\numS$ in combination with union bound over all possible subsets $\Bcal \subseteq \statespace$ to get that
    \begin{align}
        \max_{\Bcal \subseteq \statespace}|\hat P_k(s' \in \Bcal | s, a) - P(s' \in \Bcal | s, a) | \leq \sqrt{\numS} \phi(n_k(s,a)).
    \end{align}
    with probability at least $1 - 2 \delta'$ for all $k$.
    Finally, the fact about total variation 
    \begin{align}
        \|p - q\|_1 = 2 \max_{\Bcal \subseteq \statespace} |p(\Bcal) - q(\Bcal)|
    \end{align}
    as well as a union bound over $\saspace$ gives that with probability at least $1 - 2 \numS \numA \delta'$ for all $k,s,a$ it holds that $\|\hat P_k(s, a) - P(s, a)\|_1 
        \leq 2 \sqrt{\numS}\phi(n_k(s,a))$, i.e., $\prob(F^{L1}) \leq 2 \numS \numA \delta'$.
    
\textbf{Failure event $F^{N}$:}
    Consider a fixed $s \in \statespace, a \in \actionspace, t \in [H]$.
    We define $\mathcal F_k$ to be the sigma-field induced by the first $k-1$ episodes and $s_{k,1}$. Let 
    $X_k$ as the indicator whether $s, a$ was observed in episode $k$ at time $t$. The probability $\prob(s=s_{k,t}, a=a_{k,t}|s_{k,1}, \pi_k)$ of whether $X_k = 1$  is $\mathcal F_k$-measurable and hence we can apply Lemma~F.4 by \citet{dann2017unifying} with $W = \ln \frac{\numS \numA H}{\delta'}$ and obtain that
    $\prob( F^N ) \leq \numS \numA H \delta'$ after summing over all statements for $t \in [H]$ and applying the union bound over $s, a, t$.
\end{proof}

\subsection{Admissibility of Certificates}
We now show that the algorithm always gives a valid certificate in all episodes, outside the failure event $F$.  We call its complement, $F^c$, the ``good event''.  The following three lemmas prove the admissibility.
\begin{lemma}[Lower bounds admissible]
\label{lem:validlowerboundall} Consider event $F^c$ and an episode $k$, time step $h \in [H]$ and $s,a \in \saspace$.
Assume that
$\Vub_{k,h+1} \geq V^\star_{h+1} \geq V^{\pi_k}_{h+1} \geq \Vlb_{k,h+1}$ and that the lower confidence bound width is at least
\begin{align}
    \psilb_{k,h}(s,a) \geq &
     \alpha \hat P_k(s,a)( \Vub_{k,h+1} - \Vlb_{k,h+1}) + \beta  \phi(n_k(s,a))^2 
      +  \gamma \phi(n_k(s,a))
\end{align}
where there are four possible choices for $\alpha, \beta$ and $\gamma$:
\begin{align}
    \alpha =& 0 &
    \beta =& 0 &
    \gamma =& 2 \sqrt{\numS} V^{\max}_{h+1} + 1 \qquad \textrm{or}\label{eqn:lowercoeff1}
    \\
        \alpha =& 0 &
    \beta =&  4.66 \|\rho\|_1 &
    \gamma =& 2 \left[\sqrt{\hat P_k(s,a)}\right]\rho + V^{\max}_{h+1} + 1\qquad \textrm{or}\label{eqn:lowercoeff2}
    \\
                \alpha =& 0 &
    \beta =& (8.13 V^{\max}_{h+1} + 4.66 \|\rho\|_1) &
    \gamma =& 1+ \sqrt{12}\sqrt{\sigma^2_{\hat P_k(s,a)}(\Vub_{k,h+1}) + \hat P_k(s,a)\rho^2 } + 2 \left[\sqrt{\hat P_k(s,a)}\right]\rho
    \label{eqn:lowercoeff3}
    \\
            \alpha =& \frac 1 C &
    \beta =& (8.13 V^{\max}_{h+1} + (32 C + 4.66) \|\rho\|_1) &
    \gamma =& 1+ \sqrt{12}\sigma_{\hat P_k(s,a)}(\Vub_{k,h+1}) 
    \label{eqn:lowercoeff4}
\end{align}
with $\rho = \Vub_{k,h+1} - \Vlb_{k,h+1}$ and for any $C > 0$. Then the lower confidence bound at time $h$ is admissible, i.e., 
    $Q^{\pi_k}_h(s,a) \geq \Qlb_{k,h}(s,a)$.
\end{lemma}
\begin{proof}
    We want to show that
    $Q^{\pi_k}_h(s,a) - \Qlb_{k,h}(s,a) \geq 0$.
    Since $Q^{\pi_k}_h \geq 0$, this quantity is non-negative when the Q-value bound is clipped, i.e., $\Qlb_{k,h}(s,a) = 0$. The non-clipped case is left, in which
    \begin{align}
    &Q^{\pi_k}_h(s,a) - \Qlb_{k,h}(s,a) 
        =
        P(s,a) V^{\pi_k}_{h+1} + r(s,a) - \hat r_k(s,a) + \psilb_{k,h}(s,a) - \hat P_k(s,a) \Vlb_{k,h+1}. 
        \label{eqn:Qlbdiff}
        \end{align}
    For the first coefficient choice from Equation~\eqref{eqn:lowercoeff1}, we rewrite this quantity as
    \begin{align}
    &Q^{\pi_k}_h(s,a) - \Qlb_{k,h}(s,a) \\
    & =
        \psilb_{k,h}(s,a) + P(s,a)(V^{\pi_k}_{h+1} - \Vlb_{k,h+1}) + (P(s,a) - \hat P_k(s,a))\Vlb_{k,h+1} + r(s,a) - \hat r_k(s,a) \\
        \shortintertext{using the induction hypothesis for the second term and applying H\"older's inequality to the third term}
        & \geq \psilb_{k,h}(s,a) + 0 - \|P(s,a) - \hat P_k(s,a) \|_1 \| \Vlb_{k,h+1}\|_{\infty} - |r(s,a) - \hat r_k(s,a)| \\
        \shortintertext{applying definition of the good event $F^c$ to the last terms and using the first choice of coefficients for $\psilb_{k,h}$}
        &\geq 2\sqrt{\numS}V^{\max}_h \phi(n_{k}(s,a)) - 2\sqrt{\numS} \phi(n_k(s,a)) V^{\max}_{h+1} - \phi(n_k(s,a))\geq 0\,.
\end{align}
        This completes the proof for the first coefficient choice. It remain to show the same for the second and third coefficient choice. To that end, we rewrite the quantity in Equation~\eqref{eqn:Qlbdiff} as
\begin{align}
        & Q^{\pi_k}_h(s,a) - \Qlb_{k,h}(s,a) \\
        & =
        \psilb_{k,h}(s,a) + (P(s,a) - \hat P_k(s,a))(V^{\pi_k}_{h+1} - V^\star_{h+1}) + \hat P_k(s,a)(V^{\pi_k}_{h+1} - \Vlb_{k,h+1})\\
        & \quad + (P(s,a) - \hat P_k(s,a))V^\star_{h+1}  + r(s,a) - \hat r_k(s,a) \\
        \shortintertext{using the induction hypothesis, we can infer that $\hat P_k(s,a)(V^{\pi_k}_{h+1} - \Vlb_{k,h+1}) \geq 0$ and get}
                & \geq
        \psilb_{k,h}(s,a) -|(P(s,a) - \hat P_k(s,a))(V^{\pi_k}_{h+1} - V^\star_{h+1})| 
        - |(P(s,a) - \hat P_k(s,a))V^\star_{h+1}|  - |r(s,a) - \hat r_k(s,a)| \\
        \shortintertext{applying definition of the good event $F^c$ to the last term and reordering gives}
                & \geq
        -|(P(s,a) - \hat P_k(s,a))(V^{\pi_k}_{h+1} - V^\star_{h+1})| 
         - |(P(s,a) - \hat P_k(s,a))V^\star_{h+1}| - \phi(n_k(s,a)) +  \psilb_{k,h}(s,a).
         \label{eqn:lowerboundall}
    \end{align}
    We now first consider $|(P(s,a) - \hat P_k(s,a))(V^{\pi_k}_{h+1} - V^\star_{h+1})|$ and bound it using Lemma~\ref{lem:lowerorderbound} where we bind 
          $f = V^\star_{h+1} - V^{\pi_k}_{h+1}$ and with $\|f\|_1 \leq \|\Vub_{k,h+1} - \Vlb_{k,h+1}\|_1$
    \begin{align}
       & |(P(s,a) - \hat P_k(s,a))(V^{\pi_k}_{h+1} - V^\star_{h+1})|\\
    & \leq 4.66\|\Vub_{k,h+1} - \Vlb_{k,h+1}\|_1 \phi(n_k(s,a))^2
    + 2 \phi(n_k(s,a)) \sqrt{\hat P_k(s,a)} ( V^\star_{h+1} - V^{\pi_k}_{h+1}) 
    \\
    \shortintertext{and since $ 0 \leq V^\star_{h+1} - V^{\pi_k}_{h+1} \leq \Vub_{k,h+1} - \Vlb_{k,h+1} $ this is upper-bounded by}
    & \leq 4.66\|\Vub_{k,h+1} - \Vlb_{k,h+1}\|_1 \phi(n_k(s,a))^2
    + 2 \phi(n_k(s,a)) \sqrt{\hat P_k(s,a)} (\Vub_{k,h+1} - \Vlb_{k, h+1}) \label{eqn:crossterm1}
    \\
    \shortintertext{and again by Lemma~\ref{lem:lowerorderbound} we can get a nicer form for any $C > 0$ as follows}
    & \leq 
        \frac 1 {C} \hat P_k(s,a)( \Vub_{k,h+1} - \Vlb_{k,h+1})
        -(4 C + 4.66)\|\Vub_{k,h+1} - \Vlb_{k,h+1}\|_1 \phi(n_k(s,a))^2. \label{eqn:crossterm2}
        \end{align}

    After deriving runtime-computable bounds for $|(P(s,a) - \hat P_k(s,a))(V^{\pi_k}_{h+1} - V^\star_{h+1})|$, it remains to upper-bound $|(P(s,a) - \hat P_k(s,a))V^\star_{h+1}|$ in Equation~\eqref{eqn:lowerboundall}. Here, we can apply the definition of the failure event $F^{V}$ and bound $|(P(s,a) - \hat P_k(s,a))V^\star_{h+1}| \leq V^{\max}_{h+1} \phi(n_k(s,a))$. Plugging this bound together with the bound from ~\eqref{eqn:crossterm1} back into~\eqref{eqn:lowerboundall} gives
    \begin{align}
        Q^{\pi_k}_h(s,a) - \Qlb_{k,h}(s,a) 
                & \geq
        - \left(V^{\max}_{h+1} + 1 + 2 \sqrt{\hat P_k(s,a)}( \Vub_{k,h+1} - \Vlb_{k,h+1})\right)\phi(n_k(s,a)) \\
        & \quad - 4.66 \|\Vub_{k,h+1} - \Vlb_{k,h+1}\|_1 \phi(n_k(s,a))^2  +  \psilb_{k,h}(s,a)
    \end{align}
    which is non-negative when we use the second coefficient choice from Equation~\eqref{eqn:lowercoeff2} for $\psilb_{k,h}$.
    Alternatively, we can apply the definition of the failure event $F^{VE}$ which uses an empirical variance instead of the range of $V^\star_{h+1}$ and bound 
    \begin{align}
        &|(P(s,a) - \hat P_k(s,a))V^\star_{h+1}| 
        \\
        \leq&  \sqrt{4 \hat P_k(s,a)[( V^\star_{h+1}(\cdot) - P(s,a)V^\star_{h+1})^2]}
    \phi(n_k(s,a)) + 4.66 V^{\max}_{h+1} \phi(n_k(s,a))^2 \\
    \leq &
        \sqrt{12}\sqrt{ \hat P_k(s,a)(\Vub_{k,h+1} - \Vlb_{k,h+1})^2
                + \sigma^2_{\hat P_k(s,a)}(\Vub_{k,h+1})}\phi(n_k(s,a)) 
           + 8.13 V^{\max}_{h+1} \phi(n_k(s,a))^2\label{eqn:lowermain1}
    \\
    \leq & 
    \sqrt{12}\sigma_{\hat P_k(s,a)}(\Vub_{k,h+1}) \phi(n_k(s,a)) + \frac{1}{C} P_k(s,a) (\Vub_{k,h+1} - \Vlb_{k,h+1})\\
    &+ (8.13V^{\max}_{h+1} + 12 C \|\Vub_{k,h+1} - \Vlb_{k,h+1}\|_1) \phi(n_k(s,a))^2\label{eqn:lowermain2}
    \end{align}
    where we applied Lemma~\ref{lem:variancebound1}.
    Plugging the bound from \eqref{eqn:lowermain1} and \eqref{eqn:crossterm1} into \eqref{eqn:lowerboundall} gives
    \begin{align}
        Q^{\pi_k}_h(s,a) - \Qlb_{k,h}(s,a) 
                & \geq
        - \sqrt{12}\sqrt{ \hat P_k(s,a)(\Vub_{k,h+1} - \Vlb_{k,h+1})^2
                + \sigma^2_{\hat P_k(s,a)}(\Vub_{k,h+1})}\phi(n_k(s,a)) \\
                &\quad - \left(1 + 2 \sqrt{\hat P_k(s,a)}( \Vub_{k,h+1} - \Vlb_{k,h+1})\right)\phi(n_k(s,a)) \\
        & \quad - (8.13 V^{\max}_{h+1} + 4.66 \|\Vub_{k,h+1} - \Vlb_{k,h+1}\|_1 )\phi(n_k(s,a))^2  +  \psilb_{k,h}(s,a).
    \end{align}
    Applying the coefficient choice from Equation~\eqref{eqn:lowercoeff3} for $\psilb_{k,h}$ shows that this bound becomes non-negative as well.
    Finally, we plug the bound from  \eqref{eqn:lowermain2} and \eqref{eqn:crossterm2} into \eqref{eqn:lowerboundall} to get
        \begin{align}
        &Q^{\pi_k}_h(s,a) - \Qlb_{k,h}(s,a) \\
                & \geq
        - \frac 2 {C} \hat P_k(s,a)( \Vub_{k,h+1} - \Vlb_{k,h+1})
        -(8.13V^{\max}_{h+1} + (16 C + 4.66) \|\Vub_{k,h+1} - \Vlb_{k,h+1}\|_1) \phi(n_k(s,a))^2\\ 
        &\quad - (1+\sqrt{12}\sigma_{\hat P_k(s,a)}(\Vub_{k,h+1}) ) \phi(n_k(s,a)) +  \psilb_{k,h}(s,a).
    \end{align}
    We rebind $C \gets 2C$ and use the last coefficient choice from Equation~\eqref{eqn:lowercoeff4} for  $\psilb_{k,h}$ to show the above is non-negative. Hence, we have shown that for all choices for coefficients $Q^{\pi_k}_h(s,a) - \Qlb_{k,h}(s,a) \geq 0$.
\end{proof}

\begin{lemma}[Upper bounds admissible]
\label{lem:validupperbound}
Consider event $F^c$ and an episode $k$, time step $h \in [H]$ and $s,a \in \saspace$.
Assume that
$\Vub_{k,h+1} \geq V^\star_{h+1} \geq V^{\pi_k}_{h+1} \geq \Vlb_{k,h+1}$ and that the upper confidence bound width is at least
\begin{align}
    \psiub_{k,h}(s,a) \geq &
     \alpha \hat P_k(s,a)( \Vub_{k,h+1} - \Vlb_{k,h+1}) + \beta  \phi(n_k(s,a))^2 
      +  \gamma \phi(n_k(s,a))
\end{align}
where there are three possible choices for $\alpha, \beta$ and $\gamma$:
\begin{align}
    \alpha =& 0 &
    \beta =& 0 &
    \gamma =& 1 + V^{\max}_{h+1}  \qquad \textrm{or}\label{eqn:uppercoeff1}
        \\
        \alpha =& 0 &
    \beta =& 8.13 V^{\max}_{h+1} &
    \gamma =& 1+ 3.47 \sqrt{\sigma^2_{\hat P_k(s,a)}(\Vub_{k,h+1})+ \hat P_k(s,a)(\rho^2)}
    \label{eqn:uppercoeff2}
    \\
        \alpha =& \frac 1 C &
    \beta =& (8.13 V^{\max}_{h+1} + 12 C \|\rho\|_1) &
    \gamma =& 1+ 3.47 \sigma_{\hat P_k(s,a)}(\Vub_{k,h+1})
    \label{eqn:uppercoeff3}
\end{align}
with $\rho = \Vub_{k,h+1} - \Vlb_{k,h+1}$ and $C > 0$ arbitrary. Then the upper confidence bound at time $h$ is admissible; that is,
    $Q^\star_h(s,a) \leq \Qub_{k,h}(s,a)$.
\end{lemma}
\begin{proof}
    We want to show that $\Qub_{k,h}(s,a) - Q^\star_h(s,a) \geq 0$. 
    Since $Q^\star_h \leq V^{\max}_h$, this quantity is non-negative when the optimistic Q-value is clipped, i.e., $\Qub_{k,h}(s,a) = V^{\max}_h$. It remains to show that this quantity is non-negative in the non-clipped case in which
    \begin{align}
    \Qub_{k,h}(s,a) - Q^\star_h(s,a) 
    = &
    \hat r_k(s,a) + \psiub_{k,h}(s,a) + \hat P_k(s,a) \Vub_{k,h+1} - P(s,a) V^\star_{h+1} - r(s,a) \\
    = &
    \hat r_k(s,a) - r(s,a)  + \hat P_{k}(s,a)(\Vub_{k,h+1} - V^\star_{h+1})
   + (\hat P_k(s,a) - P(s,a))V^\star_{h+1} + \psiub_{k,h}(s,a)\\
   \shortintertext{by induction hypothesis, we know that $\hat P_{k}(s,a)(\Vub_{k,h+1} - V^\star_{h+1}) \geq 0$, which allows us to bound}
   \geq &
    \hat r_k(s,a) - r(s,a) 
   + (\hat P_k(s,a) - P(s,a))V^\star_{h+1} + \psiub_{k,h}(s,a)\\
    \geq &
    -|\hat r_k(s,a) - r(s,a)|  - |(\hat P_{k}(s,a)- P(s,a)) V^\star_{h+1}| + \psiub_{k,h}(s,a)
    \shortintertext{and applying the definition of the failure event $F^R$ to the first term}
        \geq &
    - \phi(n_k(s,a))  - |(\hat P_{k}(s,a)- P(s,a)) V^\star_{h+1}| + \psiub_{k,h}(s,a).\label{eqn:lb22}
    \end{align}
    It remains to bound the $ |(\hat P_{k}(s,a)- P(s,a)) V^\star_{h+1}|$ term for which we have two ways. First, we can apply the definition of $F^V$ which allows us to use $|(\hat P_{k}(s,a)- P(s,a)) V^\star_{h+1}| \leq V^{\max}_h \phi(n_k(s,a))$. This yields 
    \begin{align}
    \Qub_{k,h}(s,a) - Q^\star_h(s,a) \geq &
    \psiub_{k,h}(s,a) - \phi(n_k(s,a)) - V^{\max}_{h+1} \phi(n_k(s,a))
    \end{align}
    which is non-negative using the first choice of coefficients for $\psiub_{k,h}$ from Equation~\eqref{eqn:uppercoeff1}.
    Second, we can apply the definition of $F^{VE}$ which relies on the empirical variance instead of the range of the optimal value of the successor state. This bound gives 
    \begin{align}
        & |(P(s,a) - \hat P_k(s,a))V^\star_{h+1}| \\
        \leq&  \sqrt{4 \hat P_k(s,a)[( V^\star_{h+1}(\cdot) - P(s,a)V^\star_{h+1})^2]}
    \phi(n_k(s,a)) + 4.66 V^{\max}_{h+1} \phi(n_k(s,a))^2 \label{eqn:ppVstar} \\
    \leq & 
    \sqrt{12}\sqrt{ \hat P_k(s,a)(\Vub_{k,h+1} - \Vlb_{k,h+1})^2
                + \sigma^2_{\hat P_k(s,a)}(\Vub_{k,h+1})}\phi(n_k(s,a)) 
           + 8.13 V^{\max}_{h+1} \phi(n_k(s,a))^2
    \end{align}
    where we applied Lemma~\ref{lem:variancebound1}.
    Plugging that result into the bound in Equation~\eqref{eqn:lb22} yields
    \begin{align}
        \Qub_{k,h}(s,a) - Q^\star_h(s,a) \geq &
    \psiub_{k,h}(s,a) - \phi(n_k(s,a))  
    - 8.13V^{\max}_{h+1} \phi(n_k(s,a))^2\\
    &- \sqrt{12}\sqrt{ \hat P_k(s,a)(\Vub_{k,h+1} - \Vlb_{k,h+1})^2
                + \sigma^2_{\hat P_k(s,a)}(\Vub_{k,h+1})} \phi(n_k(s,a)).
    \end{align}
    This lower bound is non-negative when we use the second coefficient choice from Equation~\eqref{eqn:uppercoeff2} for $\psiub_{k,h}$. 
    Finally, we can also apply the second inequality from Lemma~\ref{lem:variancebound1} to Equation~\eqref{eqn:ppVstar}
    to get 
    \begin{align}
        & |(P(s,a) - \hat P_k(s,a))V^\star_{h+1}| \\
        \leq 
        &\sqrt{12}\sigma_{\hat P_k(s,a)}(\Vub_{k,h+1}) \phi(n_k(s,a)) + (8.13V^{\max}_{h+1} + 12 C \|\one\{\hat P_k(\cdot |s,a) > 0\} (\Vub_{k,h+1} - \Vlb_{k,h+1})\|_2) \phi(n_k(s,a))^2\\
    &+ \frac{1}{C} \hat P_k(s,a) (\Vub_{k,h+1} - \Vlb_{k,h+1})
        \end{align}
        and plugging that result into the bound in Equation~\eqref{eqn:lb22} with $\|\one\{\hat P_k(\cdot |s,a) > 0\} (\Vub_{k,h+1} - \Vlb_{k,h+1})\|_2 \leq \|\Vub_{k,h+1} - \Vlb_{k,h+1}\|_1$ shows that the result holds for the third coefficient choice from Equation~\eqref{eqn:uppercoeff3} for $\psiub_{k,h}$. 
    Hence, we have shown that under either coefficient choice, we have $\Qub_{k,h}(s,a) - Q^\star_h(s,a) \geq 0$.
\end{proof}

\begin{lemma}[Optimality guarantees admissible]
\label{lem:validguarantees}
In the good event $F^c$, for all episodes $k$, the certificate is valid, that is, $\Delta_k \leq \epsilon_k$.
In addition, all Q-value bounds are admissible, i.e., for all $k \in \NN$, $h \in [H]$ and $s \in \statespace, a \in \actionspace$, 
\begin{align}
    \Qlb_{k,h}(s,a) \leq Q^{\pi_k}_{h}(s,a) \leq Q^\star_{h}(s,a) \leq  \Qub_{k,h}(s,a).
\end{align}
\end{lemma}
\begin{proof} Consider the good event $F^c$.
    Since we assume that the initial state is deterministic, we have $\Delta_k = V^\star_1(s_{k,1}) - V^{\pi_k}_1(s_{k,1})$. 
     By induction we can show that $\Vub_{k,h}(s) \geq V^\star_h(s) \geq V^{\pi_k}_h(s) \geq \Vlb_{k,h}(s)$ for all $k,h,s,a$. The induction start is $h=H+1$ which holds by definition and due to the specific values of $\psilb$ and $\psiub$ in the algorithm, we can apply Lemmas~\ref{lem:validlowerboundall} and~\ref{lem:validupperbound} in each induction step. It then follows that in particular $V^{\pi_k}_1(s_{k,1}) \geq \Vlb_1(s_{k,1})$ and the claim follows from
 \begin{align}
        \Delta_k = V^\star(s_{k,1}) - V^{\pi_k}(s_{k,1} \leq& \Vub_{k,1}(s_{k,1})-\Vlb_{k,1}(s_{k,1}) = \epsilon_k\,.
    \end{align}
\end{proof}

The following two lemmas give us upper bounds on the empirical variance terms. The first lemma is used to show that the algorithm produces admissible bounds while the second is relevant for bounding the number of certificate mistakes. 
\begin{lemma}
\label{lem:variancebound1}
    Consider the good event $F^c$ and any episode $k \in \NN$ and time step $h \in [H]$. If $\Vlb_{k,h+1} \leq V^{\pi_k}_{h+1}$ and $V^\star_{h+1} \leq \Vub_{k,h+1}$, then for any $s \in \statespace$ and $a \in \actionspace$ and $C > 0$
\begin{eqnarray}
            &\sqrt{4 \hat P_k(s,a)[( V^\star_{h+1}(\cdot) - P(s,a)V^\star_{h+1})^2]}
    \phi(n_k(s,a)) \label{eqn:4sqrtlemvar} \\
    \leq &
    \sqrt{12}\sqrt{ \hat P_k(s,a)(\Vub_{k,h+1} - \Vlb_{k,h+1})^2 + \sigma^2_{\hat P_k(s,a)}(\Vub_{k,h+1})}\phi(n_k(s,a)) 
           + \sqrt{12} V^{\max}_{h+1} \phi(n_k(s,a))^2\\
    \leq
    &\!\!\!\sqrt{12}\sigma_{\hat P_k(s,a)}(\Vub_{k,h+1}) \phi(n_k(s,a)) + (\sqrt{12}V^{\max}_{h+1} + 12 C \|\one\{\hat P_k(\cdot |s,a) > 0\} (\Vub_{k,h+1} - \Vlb_{k,h+1})\|_2) \phi(n_k(s,a))^2\\
    &+ \frac{1}{C} \hat P_k(s,a) (\Vub_{k,h+1} - \Vlb_{k,h+1}) .
\end{eqnarray}
\end{lemma}
\begin{proof}
We first focus on the inner term 
\begin{align}
    &\hat P_k(s,a)[( V^\star_{h+1}(\cdot) - P(s,a)V^\star_{h+1})^2]\\
    = &
    \hat P_k(s,a)[ (\Vub_{k,h+1} - \hat P_k(s,a) \Vub_{k,h+1}  + V^\star_{h+1} - \Vub_{k,h+1} + \hat P_k(s,a) (\Vub_{k,h+1} - V^\star_{h+1}) + (\hat P_k(s,a) - P(s,a))V^\star_{h+1})^2]\\
    \shortintertext{applying the Cauchy-Schwarz inequality which gives $(a+b+c)^2 \leq 3 a^2 + 3 b^2 + 3c^2$}
    \leq & 
    3\hat P_k(s,a)[ (\Vub_{k,h+1} - \hat P_k(s,a) \Vub_{k,h+1})^2]\\
     &+ 3\hat P_k(s,a)[(V^\star_{h+1} - \Vub_{k,h+1} + \hat P_k(s,a) (\Vub_{k,h+1} - V^\star_{h+1}))^2]
     + 3((\hat P_k(s,a) - P(s,a))V^\star_{h+1})^2.
     \end{align}
The term $\hat P_k(s,a)[(V^\star_{h+1} - \Vub_{k,h+1} + \hat P_k(s,a) (\Vub_{k,h+1} - V^\star_{h+1}))^2]$ is the variance of a r.v. $V^\star_{h+1}(s') - \Vub_{k,h+1}(s')$ when $\Vub_{k,h+1}$ is fixed and $s'$ is drawn from $\hat P_k(s,a)$. We can apply the standard identity of variances $\operatorname{Var}(X) = \Ex[X^2] - \Ex[X]^2$ and rewrite this term as $\hat P_k(s,a)(V^\star_{h+1} - \Vub_{k,h+1})^2 -  (\hat P_k(s,a) (\Vub_{k,h+1} - V^\star_{h+1}))^2$. Plugging this back in gives the bound
\begin{align}
     & 3\hat P_k(s,a)[ (\Vub_{k,h+1} - \hat P_k(s,a) \Vub_{k,h+1})^2]
     + 3\hat P_k(s,a)(V^\star_{h+1} - \Vub_{k,h+1})^2\\ 
      &- 3 (\hat P_k(s,a) (\Vub_{k,h+1} - V^\star_{h+1}))^2
     + 3((\hat P_k(s,a) - P(s,a))V^\star_{h+1})^2\\
     \shortintertext{leveraging $(F^V)^c$ for the final term and dropping the third term which cannot be positive}
          \leq & 
    3\hat P_k(s,a)[ (\Vub_{k,h+1} - \hat P_k(s,a) \Vub_{k,h+1})^2]
     + 3\hat P_k(s,a)(V^\star_{h+1} - \Vub_{k,h+1})^2 + 3(V^{\max}_{h+1})^2 \phi(n_k(s,a))^2.
\end{align}
We substitute this bound on $\hat P_k(s,a)[( V^\star_{h+1}(\cdot) - P(s,a)V^\star_{h+1})^2]$ back into Equation~(\ref{eqn:4sqrtlemvar}):
\begin{align}
        &\sqrt{4 \hat P_k(s,a)[( V^\star_{h+1}(\cdot) - P(s,a)V^\star_{h+1})^2]}
    \phi(n_k(s,a)) \nonumber \\
    \leq 
    &\sqrt{12 \sigma_{\hat P_k(s,a)}^2 (\Vub_{k,h+1}) + 12 \hat P_k(s,a)(V^\star_{h+1} - \Vub_{k,h+1})^2} \phi(n_k(s,a))
    + \sqrt{12} V^{\max}_{h+1} \phi(n_k(s,a))^2.  \\
    \shortintertext{We now leverage that $\Vlb_{k,h+1} \leq V^{\pi_k}_{h+1} \leq V^\star_{h+1} \leq \Vub_{k,h+1}$ to get a computable bound}
    \leq 
    &\sqrt{12 \sigma_{\hat P_k(s,a)}^2 (\Vub_{k,h+1}) + 12 \hat P_k(s,a)(\Vub_{k,h+1} - \Vlb_{k,h+1})^2} \phi(n_k(s,a))
    + \sqrt{12} V^{\max}_{h+1} \phi(n_k(s,a))^2.
\end{align}
This is the first inequality to show. For the second inequality, we first bound this expression further as
\begin{align}
        &\sqrt{12} \sigma_{\hat P_k(s,a)} (\Vub_{k,h+1}) \phi(n_k(s,a)) 
    + \sqrt{12} V^{\max}_{h+1} \phi(n_k(s,a))^2 
    + \sqrt{12 \hat P_k(s,a)(\Vub_{k,h+1} - \Vlb_{k,h+1})^2} \phi(n_k(s,a)).
    \label{eqn:last_term_vstar}
\end{align}
We now treat the last term in Equation~\eqref{eqn:last_term_vstar} separately as
\begin{align}
    &\sqrt{12 \hat P_k(s,a)(\Vub_{k,h+1} - \Vlb_{k,h+1})^2 \phi(n_k(s,a))^2}
    =
    \sqrt{\sum_{s' \in \statespace}
    \frac{12 \phi(n_k(s,a))^2}{\hat P_k(s'|s,a)}\hat P_k(s'|s,a)^2 (\Vub_{k,h+1} - \Vlb_{k,h+1})^2}\\
   \shortintertext{splitting the sum based on whether $\hat P_k(s'| s,a) \leq 12 C^2 \phi(n_k(s,a))^2$ for $C > 0$ and making repeated use of $\sqrt{\sum_i a_i} \leq \sum_i \sqrt{a_i}$}
    \leq&
    \sqrt{\sum_{s' \in \statespace}
    \frac{1}{C^2} \hat P_k(s'|s,a)^2 (\Vub_{k,h+1} - \Vlb_{k,h+1})^2}
    +
    12\sqrt{\sum_{s' \in \statespace} \one\{\hat P_k(s'|s,a) > 0\}
    C^2 \phi(n_k(s,a))^4 (\Vub_{k,h+1} - \Vlb_{k,h+1})^2}\\
    \leq &\frac{1}{C} \hat P_k(s,a) (\Vub_{k,h+1} - \Vlb_{k,h+1})
    + 12 C \|\one\{\hat P_k(\cdot |s,a) > 0\} (\Vub_{k,h+1} - \Vlb_{k,h+1})\|_2 \phi(n_k(s,a))^2.
\end{align}
Plugging this bound for the final term back in to Equation~\ref{eqn:last_term_vstar} gives the desired statement
\begin{align}
            &\sqrt{4 \hat P_k(s,a)[( V^\star_{h+1}(\cdot) - P(s,a)V^\star_{h+1})^2]}
    \phi(n_k(s,a))\\
    \leq 
    &\sqrt{12} \sigma_{\hat P_k(s,a)}(\Vub_{k,h+1} ) \phi(n_k(s,a))\\
    &+ (\sqrt{12}V^{\max}_{h+1} + 12 C \|\one\{\hat P_k(\cdot |s,a) > 0\} (\Vub_{k,h+1} - \Vlb_{k,h+1})\|_2) \phi(n_k(s,a))^2\\
    &+ \frac{1}{C} P_k(s,a) (\Vub_{k,h+1} - \Vlb_{k,h+1}).
\end{align}
\end{proof}

\begin{lemma}
\label{lem:variancebound2}
    Consider the good event $F^c$ and any episode $k \in \NN$ and time step $h \in [H]$. If $\Vlb_{k,h+1} \leq V^{\pi_k}_{h+1}$ and $V^\star_{h+1} \leq \Vub_{k,h+1}$, then for any $s \in \statespace$, $a \in \actionspace$ and $C, D > 0$
\begin{eqnarray}
    \sqrt{D \sigma^2_{\hat P_k(s,a)} (\Vub_{k,h+1})} \phi(n_k(s,a))
    \leq  
    & \sqrt{4D \sigma^2_{P(s,a)} (V^{\pi_k}_{h+1})} \phi(n_k(s,a))  \label{eqn:lemma12full}\\
    & + (6\sqrt{D} + 4CD) V^{\max}_{h+1} \numS \phi(n_k(s,a))^2\\
    &+ \frac{1}{C} \hat P_k(s,a) (\Vub_{k,h+1} - \Vlb_{k,h+1}).
\end{eqnarray}
\end{lemma}
\begin{proof}
Note that the proof proceeds in the same fashion as Lemma~\ref{lem:variancebound1}. We first focus on the inner term 
\begin{align}
    &\hat P_k(s,a)[ (\Vub_{k,h+1} - \hat P_k(s,a) \Vub_{k,h+1})^2]\\
    = &
    \hat P_k(s,a)[ (V^{\pi_k}_{h+1} - P(s,a) V^{\pi_k}_{h+1} + \Vub_{k,h+1} - V^{\pi_k}_{h+1} - \hat P_k(s,a)(\Vub_{k,h+1} - V^{\pi_k}_{h+1})  + (P(s,a) - \hat P_k(s,a))V^{\pi_k}_{h+1})^2]\\
    \shortintertext{applying $(a+b)^2 \leq 2 a^2 + 2 b^2$ twice }
    \leq & 
    2\hat P_k(s,a)[ (V^{\pi_k}_{h+1} - P(s,a) V^{\pi_k}_{h+1})^2]\\
     &+ 4\hat P_k(s,a)[(\Vub_{k,h+1} - V^{\pi_k}_{h+1} - \hat P_k(s,a)(\Vub_{k,h+1} - V^{\pi_k}_{h+1}))^2]
     + 4((P(s,a) - \hat P_k(s,a))V^{\pi_k}_{h+1})^2\\    
     \shortintertext{using the identity of variances applied to the variance of $\Vub_{k,h+1} - V^{\pi_k}_{h+1}$ w.r.t. $\hat{P}_k(s,a)$}
     = & 
    2\hat P_k(s,a)[ (V^{\pi_k}_{h+1} - P(s,a) V^{\pi_k}_{h+1})^2]
     + 4\hat P_k(s,a)(\Vub_{k,h+1} - V^{\pi_k}_{h+1})^2\\ 
      &- 4 (\hat P_k(s,a) (\Vub_{k,h+1} - V^{\pi_k}_{h+1}))^2
     + 4((P(s,a) - \hat P_k(s,a))V^{\pi_k}_{h+1})^2\\
     \shortintertext{dropping the third term which cannot be positive and applying the definition of event $F^{L1}$ to the last term}
          \leq & 
    2\hat P_k(s,a)[ (V^{\pi_k}_{h+1} - P(s,a) V^{\pi_k}_{h+1})^2]
     + 4\hat P_k(s,a)(\Vub_{k,h+1} - V^{\pi_k}_{h+1})^2 + 16(V^{\max}_{h+1} \sqrt{\numS}\phi(n_k(s,a)))^2\\
               = & 
    2 P(s,a)[ (V^{\pi_k}_{h+1} - P(s,a) V^{\pi_k}_{h+1})^2]
    + 2(\hat P_k(s,a)- P(s,a))[ (V^{\pi_k}_{h+1} - P(s,a) V^{\pi_k}_{h+1})^2]\\
    & + 4\hat P_k(s,a)(\Vub_{k,h+1} - V^{\pi_k}_{h+1})^2
    + 16(V^{\max}_{h+1})^2 \numS \phi(n_k(s,a))^2\\
    \shortintertext{applying Lemma~\ref{lem:lowerorderbound} to the second term with $C=1$ and $f = (V^{\pi_k}_{h+1} - P(s,a) V^{\pi_k}_{h+1})^2$}
                   \leq & 
    4 P(s,a)[ (V^{\pi_k}_{h+1} - P(s,a) V^{\pi_k}_{h+1})^2]
    + 12.5 \|(V^{\pi_k}_{h+1} - P(s,a) V^{\pi_k}_{h+1})^2\|_1 \phi(n_k(s,a))^2\\
    & + 4\hat P_k(s,a)(\Vub_{k,h+1} - V^{\pi_k}_{h+1})^2
    + 16(V^{\max}_{h+1})^2 \numS \phi(n_k(s,a))^2\\
                       \leq & 
    4 P(s,a)[ (V^{\pi_k}_{h+1} - P(s,a) V^{\pi_k}_{h+1})^2]
     + 4\hat P_k(s,a)(\Vub_{k,h+1} - V^{\pi_k}_{h+1})^2
    + 28.5 (V^{\max}_{h+1})^2 \numS \phi(n_k(s,a))^2.
\end{align}
We now plug this result into the right hand expression of Equation~\ref{eqn:lemma12full} to get 
\begin{align}
    \sqrt{D \hat P_k(s,a)[ (\Vub_{k,h+1} - \hat P_k(s,a) \Vub_{k,h+1})^2]} \phi(n_k(s,a))
    \leq 
    & \sqrt{4D P(s,a)[ (V^{\pi_k}_{h+1} - P(s,a) V^{\pi_k}_{h+1})^2]} \phi(n_k(s,a))\\
    & + 6\sqrt{D} V^{\max}_{h+1} \sqrt{\numS} \phi(n_k(s,a))^2\\
    &+ \sqrt{4D \hat P_k(s,a)(\Vub_{k,h+1} - V^{\pi_k}_{h+1})^2}\phi(n_k(s,a)).
\end{align}
We now treat the last term separately and start by using the assumption that $\Vlb_{k,h+1} \leq V^{\pi_k}_{h+1} \leq \Vub_{k,h+1}$
\begin{align}
    &\sqrt{4D \hat P_k(s,a)(\Vub_{k,h+1} - V^{\pi_k}_{h+1})^2 \phi(n_k(s,a))^2}
    \leq \sqrt{4D \hat P_k(s,a)(\Vub_{k,h+1} - \Vlb_{k,h+1})^2 \phi(n_k(s,a))^2}\\
    =&
    \sqrt{\sum_{s' \in \statespace}
    \frac{4D \phi(n_k(s,a))^2}{\hat P_k(s'|s,a)}\hat P_k(s'|s,a)^2 (\Vub_{k,h+1} - \Vlb_{k,h+1})^2}\\
   \shortintertext{splitting the sum based on whether $P_k(s'| s,a) \leq 4D C^2 \phi(n_k(s,a))^2$ for $C > 0$ and making repeated use of $\sqrt{\sum_i a_i} \leq \sum_i \sqrt{a_i}$}
    \leq&
    \sqrt{\sum_{s' \in \statespace}
    \frac{1}{C^2}\hat P_k(s'|s,a)^2 (\Vub_{k,h+1} - \Vlb_{k,h+1})^2}
    +
    \sqrt{\sum_{s' \in \statespace}
    (4D)^2 C^2 \phi(n_k(s,a))^4 (\Vub_{k,h+1} - \Vlb_{k,h+1})^2}\\
    \leq &\frac{1}{C} \hat P_k(s,a) (\Vub_{k,h+1} - \Vlb_{k,h+1})
    + 4D C \|\Vub_{k,h+1} - \Vlb_{k,h+1}\|_1 \phi(n_k(s,a))^2.
\end{align}
Plugging this bound for the final term back in gives the desired statement
\begin{align}
    \sqrt{D \hat P_k(s,a)[ (\Vub_{k,h+1} - \hat P_k(s,a) \Vub_{k,h+1})^2]} \phi(n_k(s,a))
    \leq 
    & \sqrt{4D P(s,a)[ (V^{\pi_k}_{h+1} - P(s,a) V^{\pi_k}_{h+1})^2]} \phi(n_k(s,a))\\
    & + (6\sqrt{D} + 4CD) V^{\max}_{h+1} \numS \phi(n_k(s,a))^2\\
    &+ \frac{1}{C} P_k(s,a) (\Vub_{k,h+1} - \Vlb_{k,h+1}).
\end{align}
\end{proof}

\subsection{Bounding the Number of Large Certificates}
We start by deriving an upper bound on each certificate in terms of the confidence bound widths.
\begin{lemma}[Upper bound on certificates]
\label{lem:guaranteebound}
Let $w_{k,h}(s,a) = \prob(s_{k,h}=s, a_{k,h}=a |s_{k,1}, \pi_k)$ be the probability of encountering $s,a$ at time $h$ in the $k$th episode.
In the good event $F^c$, for all episodes $k$, the following bound on the optimality-guarantee holds
	\begin{align}
	    \epsilon_k 
	    \leq & \exp(6) \sum_{s,a \in \saspace}\sum_{h=1}^H w_{k,h}(s,a) 
	(H \wedge (\beta \phi(n_k(s,a))^2 + \gamma_{k,h}(s,a) \phi(n_k(s,a))))
	\end{align}
	with $\beta = 336 \numS H^2$ and $\gamma_{k,h}(s,a) = 14\sigma_{P(s,a)}(V^{\pi_k}_{h+1}) + 2$.
\end{lemma}
\begin{proof}
In this lemma, we use the definition of $\psilb_{k,h} = \psiub_{k,h} = \frac{1}{H} \hat P_k(s,a)( \Vub_{k,h+1} - \Vlb_{k,h+1}) +  45 \numS H^2 \phi(n_k(s,a))^2 + \left(1+ \sqrt{12} \sigma_{\hat P_k(s,a)}(\Vub_{k,h+1})\right) \phi(n_k(s,a))$ from the algorithm in the main text (Algorithm~\ref{alg:ulcr} and note that is an upper bound on the definition of the bonus terms in Algorithm~\ref{alg:ulcr_morecomplicated}. Hence, the lemma holds for both. We start by bounding the sum of confidence widths as
\begin{align}
\psiub_{k,h}(s,a) + \psilb_{k,h}(s,a)
    \leq & \frac 2 {H} \hat P_k(s,a)( \Vub_{k,h+1} - \Vlb_{k,h+1}) + \hat \beta \phi(n_k(s,a))^2 + \hat \gamma_{k,h}(s,a) \phi(n_k(s,a))
\end{align}
where we define
\begin{align}
    \hat \beta &=  90 \numS H^2\\
    \hat \gamma_{k,h}(s,a) &= 2\left(1+ \sqrt{12} \sigma_{\hat P_k(s,a)}(\Vub_{k,h+1})\right).
\end{align}
Before moving on, we further bound the final term using Lemma~\ref{lem:variancebound2} as
\begin{align}
	    \hat \gamma_{k,h}(s,a) \phi(n_k(s,a)) =&
	    2 \phi(n_k(s,a)) + 
	    \sqrt{48 \sigma^2_{\hat P_k(s,a)}(\Vub_{k,h+1})} \phi(n_k(s,a))\\
    \leq &   
     (14 \sigma_{P(s,a)}(V^{\pi_k}_{h+1}) + 2) \phi(n_k(s,a))\\
    & + (42 + 192H) V^{\max}_{h+1} \numS \phi(n_k(s,a))^2
    + \frac{1}{H} \hat P_k(s,a) (\Vub_{k,h+1} - \Vlb_{k,h+1}).
\end{align}
Hence, we can bound the sum of confidence widths as 
\begin{align}
\psiub_{k,h}(s,a) + \psilb_{k,h}(s,a)
    \leq & \frac 3 {H} \hat P_k(s,a)( \Vub_{k,h+1} - \Vlb_{k,h+1}) + \tilde \beta \phi(n_k(s,a))^2 + \gamma_{k,h}(s,a) \phi(n_k(s,a))
\end{align}
where we define
\begin{align}
    \tilde \beta &=   (42 + 192 + 90) \numS H^2 = 324 \numS H^2\\
    \gamma_{k,h}(s,a) &= 14 \sigma_{P(s,a)}(V^{\pi_k}_{h+1}) + 2.
\end{align}
By definition of the upper and lower bound estimates
	\begin{align}
		&\Qub_{k,h}(s,a) - \Qlb_{k,h}(s,a) \\
		& \leq 
        \psiub_{k,h}(s,a) + \psilb_{k,h}(s,a) + \hat P_k(s,a) \Vub_{k,h+1} - \hat P_k(s,a) \Vlb_{k,h+1}\\
        &\leq \left( 1 + \frac 3 {H}\right) \hat P_k(s,a)( \Vub_{k,h+1} - \Vlb_{k,h+1}) 
                    + \tilde \beta \phi(n_k(s,a))^2 
                    + \gamma_{k,h}(s,a) \phi(n_k(s,a)).
        \label{eqn:bound111}
	\end{align}
	We first treat  $\hat P_k(s,a)( \Vub_{k,h+1} - \Vlb_{k,h+1})$ separately as
	\begin{align}
	    & \hat P_k(s,a)( \Vub_{k,h+1} - \Vlb_{k,h+1})
	    \leq 
	    P(s,a)( \Vub_{k,h+1} - \Vlb_{k,h+1}) + (\hat P_k(s,a) - P(s,a))( \Vub_{k,h+1} - \Vlb_{k,h+1})
	    \\
	    \shortintertext{and bound the final term with Lemma~\ref{lem:lowerorderbound} binding $f = \Vub_{k,h+1} - \Vlb_{k,h+1} \in [0, H]$ and $C = H/3$}
	    & \leq \left(1 + \frac 3 H\right) P(s,a)(\Vub_{k,h+1} - \Vlb_{k,h+1}) 
			+ 3\numS H^2 \phi(n_k(s,a))^2.
	\end{align}
	Plugging this result back in the expression in \eqref{eqn:bound111} and setting $\beta = 336 \numS H^2 \geq \tilde \beta + 3(1 + 3/H) \numS H^2$ yields
	\begin{align}
	    		&\Qub_{k,h}(s,a) - \Qlb_{k,h}(s,a) \\
		& \leq 
         \left( 1 + \frac 3 {H}\right)^2 P(s,a)( \Vub_{k,h+1} - \Vlb_{k,h+1})
         +  \beta \phi(n_k(s,a))^2 
         + \gamma_{k,h}(s,a) \phi(n_k(s,a))\\
        & =          \left( 1 + \frac 3 {H}\right)^2 {P_h^{\pi_k}(s,a)}(\Qub_{k,h+1} - \Qlb_{k,h+1}) 
                 +  \beta \phi(n_k(s,a))^2 
         + \gamma_{k,h}(s,a) \phi(n_k(s,a)).
	\end{align}
	Here, $P_h^{\pi_k}(s,a)f = \Ex[f(s_{k,h+1}, \pi_k(s_{k,h+1}, h+1))| s_{k,h} = s, a_{k,h}=a, \pi_k]$ denotes the composition of $P(s,a)$ and the policy action selection operator at time $h+1$. In addition to the bound above, by construction also $0 \leq \Qub_{k,h}(s,a) - \Qlb_{k,h}(s,a) \leq V_h^{\max}$ holds at all times. Resolving this recursive bound yields
	\begin{align}
		&\epsilon_k = (\Vub_{k,1} - \Vlb_{k,1})(s_{k,1}) = (\Qub_{k,1} - \Qlb_{k,1})(s_{k,1},\pi_k(s_{k,1}, 1))\\
		&\leq  \sum_{s,a \in \saspace}\sum_{h=1}^H \left(1 + \frac 3 H \right)^{2h} w_{k,h}(s,a) 
		( 
		V^{\max}_h \wedge +  (\beta \phi(n_k(s,a))^2 
         + \gamma_{k,h}(s,a) \phi(n_k(s,a)))
	 )
		\\
		&\leq  \exp(6) \sum_{s,a \in \saspace}\sum_{h=1}^H w_{k,h}(s,a) 
	(H \wedge (\beta \phi(n_k(s,a))^2 + \gamma_{k,h}(s,a) \phi(n_k(s,a)))).
	\end{align}
	Here we bounded with $x = 2H$
	\begin{align}
	    \left(1 + \frac 3 H \right)^{2h} \leq 
	    \left(1 + \frac 3 H \right)^{2H} = 
	    \left(1 + \frac 6 x \right)^{x} \leq
	    \lim_{x \rightarrow \infty} \left(1 + \frac 6 x \right)^{x} = \exp(6).
	\end{align}
\end{proof}

We now follow the proof structure of \citet{dann2017unifying} and define \emph{nice} episodes, in which all state-action pairs either have low probability of occurring
or the sum of probability of occurring in the previous episodes is large enough
so that outside the failure event we can guarantee that
\begin{align}
    n_{k}(s,a) \geq \frac 1 4 \sum_{i < k} w_{i}(s,a).
\end{align}
This allows us then to bound the number of nice episodes with large certificates by the number of times terms of the form
\begin{align}
    \sum_{s,a \in L_{k}} w_{k}(s,a) \sqrt{\frac{\llnp (n_{k}(s,a)) + D}{n_{k}(s,a)}} 
\end{align}
can exceed a chosen threshold (see Lemma~\ref{lem:mainratelemma} below).
\begin{definition}[Nice Episodes]
An episode $k$ is \emph{nice} if and only if for all $s \in \statespace$, $a \in \actionspace$ the following condition holds:
    \begin{align}
        w_{k}(s,a) \leq \wmin \quad \vee \quad
        \frac 1 4 \sum_{i < k} w_{i}(s,a) \geq H\ln \frac{\numS \numA H}{\delta'}.
    \end{align}
    We denote the set of indices of all nice episodes as $\Ncal \subseteq \NN$.
\end{definition}

\begin{lemma}[Properties of nice episodes]
    If an episode $k$ is nice, i.e., $k \in \Ncal$, then in the good event $F^c$, for all $s \in \statespace$, $a \in \actionspace$ the following statement holds:
    \begin{align}
        w_{k}(s,a) \leq \wmin \quad \vee \quad
        n_{k}(s,a) \geq \frac 1 4 \sum_{i < k} w_{i}(s,a).
    \end{align}
    \label{lem:nicenessep}
\end{lemma}
\begin{proof}
    Since we consider the event ${F^N}^c$, it holds for all $s, a$ pairs with $w_{k}(s,a) > \wmin$
    \begin{align}
        n_{k}(s,a) \geq \frac 1 2 \sum_{i < k} w_{i}(s,a) -  H\ln \frac{\numS \numA H}{\delta'}
        \geq \frac 1 4 \sum_{i < k} w_{i}(s,a)
    \end{align}
    for $k \in N$.
\end{proof}

\begin{lemma}[Number of episodes that are not nice]
    On the good event $F^c$, the number of episodes that are not nice is at most
    \begin{align}
\frac{4\numS^2 \numA(\numA \wedge H) H^2}{c_\epsilon \epsilon}\ln \frac{\numS \numA H}{\delta'}.
    \end{align}
    \label{lem:num_nonnice}
\end{lemma}
\begin{proof}
    If an episode $k$ is not nice, then there is $s,a$ with $w_{k}(s,a) > \wmin$ and \begin{align}
        \sum_{i < k} w_{i}(s,a) <  4H \ln \frac{\numS \numA H}{\delta'}.
    \end{align}
            The sum on the left-hand side of this inequality increases by at least $\wmin$ after the episode while the right hand side stays constant, this situation can occur at most 
            \begin{align}
\frac{4\numS \numA H}{\wmin}\ln \frac{\numS \numA H}{\delta'}
=\frac{4\numS^2 \numA (\numA \wedge H) H^2}{c_\epsilon \epsilon}\ln \frac{\numS \numA H}{\delta'}
\end{align} 
times in total.
\end{proof}

\subsection{Proof of IPOC bound of \ulcr, Theorem~\ref{thm:ulcripoc}}

We are now equipped with all tools to complete the proof of Theorem~\ref{thm:ulcripoc}:
\begin{proof}[Proof of Theorem~\ref{thm:ulcripoc}]
Consider event $F^c$ which has probability at least $1-\delta$ due to Lemma~\ref{lem:failureprob}.
In this event, all optimality guarantees are admissible by Lemma~\ref{lem:validguarantees}.
Further, using Lemma~\ref{lem:guaranteebound}, the optimality guarantees are bounded as
	\begin{align}
	    \epsilon_k 
	    \leq  \exp(6) \sum_{s,a \in \saspace}\sum_{h=1}^H w_{k,h}(s,a) 
	(H \wedge (\beta \phi(n_k(s,a))^2 + \gamma_{k,h}(s,a) \phi(n_k(s,a))))
	\end{align}
	where $\beta = 336 \numS H^2$ and $\gamma_{k,h}(s,a) = 14\sigma_{P(s,a)} (V^{\pi_k}_{h+1}) + 2$.
	It remains to show that for any given threshold $\epsilon > 0$ this bound does not exceed $\epsilon$ except for at most the number of times prescribed by Equation~\ref{eqn:ipocbound_ulcr}. Recall the definition of $L_k$ as the set of state-action pairs with significant probability of occurring, $L_{k} = \{ (s,a) \in \saspace \, : \, w_{k}(s,a) \geq \wmin \}$,
	and split the sum as
		\begin{align}
	    \epsilon_k 
	    \leq &\exp(6) \sum_{s,a \notin L_k} w_k(s,a) H \\
	    & + \exp(6) \beta  \sum_{s,a \in L_k}  w_k(s,a) \phi(n_k(s,a))^2
	    + \exp(6) \sum_{s,a \in L_k} \sum_{h=1}^H w_{k,h}(s,a) \gamma_{k,h}(s,a) \phi(n_k(s,a))
	\end{align}
	and bound each of the three remaining terms individually.
	First, the definition of $L_k$ was chosen such that
	\begin{align}
	    e^6 \sum_{s,a \notin L_k} w_k(s,a) H \leq e^6 H \wmin \numS (\numA \wedge H) = \frac{e^6 H \numS (\numA \wedge H) \epsilon c_\epsilon}{H \numS (\numA \wedge H)} 
	    = c_\epsilon e^6 \epsilon,
	\end{align}
	where we used the fact that the number of positive $w_k(s,a)$  is no greater than $\numS \numA$ or $\numS H$ per episode $k$.

	Second,  
	we use Corollary~\ref{cor:mainratelemma} with $r=1, C=0.728, D = 0.72 \ln \frac{5.2}{\delta'}$ (which satisfies $D > 1$ for any $\delta' \leq 1$) and $\epsilon' = \frac{c_\epsilon \epsilon}{\beta}$ to bound
	\begin{align}
	    \exp(6) \beta\sum_{s,a \in L_k}w_k(s,a) \phi(n_k(s,a))^2 \leq c_\epsilon \epsilon e^6
	\end{align}
	on all but at most 
	\begin{align}
	    O\left(\frac{\numS \numA \beta}{\epsilon}\polylog(\numS, \numA, H, 1/\epsilon, \ln(1/\delta))\right)
	    = O\left(\frac{\numS^2 \numA H^2}{\epsilon}\polylog(\numS, \numA, H, 1/\epsilon, \ln(1/\delta))\right)
	\end{align}
	nice episodes. Third,
	we use Lemma~\ref{lem:mainratelemma} with $r=2, C=0.728, D = 0.72 \ln \frac{5.2}{\delta'}$ and $\epsilon' = c_\epsilon \epsilon$ to bound
	\begin{align}
	\exp(6) \sum_{s,a \in L_k} \sum_{h=1}^H w_{k,h}(s,a) \gamma_{k,h}(s,a) \phi(n_k(s,a))
	    \leq c_\epsilon e^6 \epsilon
	\end{align}
	on all but at most 
	\begin{align}
	    O\left(\frac{\numS \numA B}{\epsilon^2}\polylog(\numS, \numA, H, 1/\epsilon, \ln(1/\delta))\right)
	    = 	    O\left(\frac{\numS \numA H^2}{\epsilon^2}\polylog(\numS, \numA, H, 1/\epsilon, \ln(1/\delta))\right)
	\end{align}
	nice episodes. 
	Here we choose $B = 400H^2 + 4H$ which is valid since 
	\begin{align}
	    & \sum_{s,a \in L_{k}} \sum_{h=1}^H w_{k,h}(s, a) \gamma_{k,h}(s,a)^2
	    \leq 
	    \sum_{s,a \in \saspace} \sum_{h=1}^H w_{k,h}(s, a) \gamma_{k,h}(s,a)^2\\
	    \leq &
	    2 \sum_{s,a \in \saspace} \sum_{h=1}^H w_{k,h}(s, a) 2^2
	    + 2 \sum_{s,a \in \saspace} \sum_{h=1}^H w_{k,h}(s, a) 14^2 \sigma_{P(s,a)}^2(V^{\pi_k}_{h+1})\\
	    \leq &
	    8 H + 400 \sum_{s,a \in \saspace} \sum_{h=1}^H w_{k,h}(s, a)  \sigma_{P(s,a)}^2(V^{\pi_k}_{h+1}) \leq 8 H + 400 H^2.
	\end{align}
	The first inequality comes from the definition of $\gamma_{k,h}$ and applying the fact $(a+b)^2 \leq 2a^2 + 2b^2$. The second inequality follows from the fact that $w_{k,h}$ is a probability distribution over state and actions and hence their total sum over all time steps is $H$. Finally, we applied Lemma~4 by \citet{dann2015sample} which tells us that the sum of variances is simply the variance of the sum of rewards per episode and hence bounded by $H^2$.
	
	Combining the bounds for the three terms above, we obtain that $\epsilon_k \leq 3 c_\epsilon \epsilon \leq \epsilon$ on all nice episodes except at most
	\begin{align}
	    O\left(\left(\frac{\numS \numA H^2}{\epsilon^2} + \frac{\numS^2 \numA H^2}{\epsilon}\right) \polylog(\numS, \numA, H, 1/\epsilon, \ln(1/\delta))\right)
	\end{align}
	nice episodes.
	Further, Lemma~\ref{lem:num_nonnice} states that the number of episodes that are not nice is bounded by
	\begin{align}
	    O\left(\frac{\numS^2 \numA (\numA \wedge H) H^2}{\epsilon}\polylog(\numS, \numA, H, 1/\epsilon, \ln(1/\delta))\right).
	\end{align}
	Taking all these bounds together, we can bound $\epsilon_k \leq 4 c_\epsilon \epsilon \leq \epsilon$ for all episodes $k$ except at most
	\begin{align}
		    O\left(\left(\frac{\numS \numA H^2}{\epsilon^2} + \frac{\numS^2 \numA(\numA \wedge H)H^2}{\epsilon}\right) \polylog(\numS, \numA, H, 1/\epsilon, \ln(1/\delta))\right)
	\end{align}
	which completes the proof.
\end{proof}

\subsection{Technical Lemmas}
\begin{lemma}
\label{lem:llxoverxupperbound}
Let $\tau \in (0, \hat \tau]$ and $D \geq 1$. Then  
    for all $x \geq \check x = \frac{\ln(C / \tau) + D}{\tau}$ with $C = 16 \vee \hat \tau D^2$, the following inequality holds
    \begin{align}
        \frac{\llnp(x) + D}{x} \leq \tau.
    \end{align}
\end{lemma}
\begin{proof}
Since by Lemma~\ref{lem:llnpprop} the function $\frac{\llnp(x) + D}{x}$ is monotonically decreasing in $x$, we can bound 
    \begin{align}
        \frac{\llnp(x) + D}{x} \leq \frac{\llnp(\check x) + D}{\check x} = \frac{\llnp(\check x) + D}{\ln(C / \tau) + D} \tau.
        \end{align}
        It remains to show that $\ln(\check x) \vee 1 \leq \frac{C}{\tau}$. First, note that $\frac C \tau \geq \frac C {\hat \tau} \geq D^2 \geq 1$.
        Also, we can bound using $\ln(x) \leq 2 \sqrt{x}$
        \begin{align}
            \ln(\check x) =& \ln\left(\frac{\ln(C / \tau) + D}{\tau}\right)
            \leq 2 \sqrt{\frac{\ln(C / \tau) + D}{\tau}}
            \leq 2 \sqrt{\frac{2\sqrt{C / \tau} + D}{\tau}}\\
			 \leq & 4 \left(\frac{C}{\tau^3}\right)^{1/4}
			 \leq \left(\frac{C}{\tau}\right)^{3/4}
			 \leq \frac{C}{\tau},
        \end{align}
        since $\sqrt{C} \geq 4$ and $C/\tau \ge 1$.
\end{proof}

\begin{lemma}
\label{lem:lowerorderbound}
Let $f: \statespace \mapsto [0, \infty]$ be a (potentially random) function.
In the good event $F^c$, for all episodes $k$, states $s \in \statespace$ and actions $a \in \actionspace$, the following bound holds for any $C > 0$
\begin{align}
    |(\hat P_k - P)(s,a) f| \leq & 1.56 \|f\|_1  \phi(n_k(s,a))^2 +
    2\phi(n_k(s,a)) \sum_{s \in \statespace} \sqrt{P(s' | s, a)} f(s') \\
    \leq & (4C+1.56) \|f\|_1  \phi(n_k(s,a))^2 + \frac{1}{C} P(s,a) f
\end{align}
and 
\begin{align}
    |(\hat P_k - P)(s,a) f|
    \leq & 4.66 \|f\|_1  \phi(n_k(s,a))^2 +
    2\phi(n_k(s,a)) \sum_{s \in \statespace} \sqrt{\hat P_k(s' | s, a)} f(s')
    \\
    \leq & (4C+4.66) \|f\|_1  \phi(n_k(s,a))^2 + \frac{1}{C} \hat P_k(s,a) f.
\end{align}
\end{lemma}
\begin{proof}
	\begin{align}
		&|(\hat P_k - P)(s,a) f| \leq \sum_{s' \in \statespace} |(\hat P_k-P)(s' | s,a)| f(s')\\
		\shortintertext{We now apply the definition of $F^c$ on each $|(\tilde P_k-P)(s' | s,a)|$ individually.  Specifically, we use $F^{P}$ and $F^{PE}$ for the first and second bound respectively. To unify their treatment, we use $\tilde P$ for $P$ and $\hat P_k$, $c_1=4$ and $c_2$ for $1.56$ and $4.66$ respectively.}
		\leq & \sum_{s' \in \statespace} f(s') c_2  \phi(n_k(s,a)^2 + \sqrt{c_1 \tilde P(s'|s,a)} \phi(n_k(s,a))f(s'))
		\\
		= & c_2 \|f\|_1 \phi(n_k(s,a))^2 + \phi(n_k(s,a)) \sum_{s' \in \statespace} \sqrt{c_1 \tilde P(s'|s,a)} f(s')
		\\
		\shortintertext{This is the first inequality to show but we can further rewrite this to show the second inequality as follows}
		= & c_2 \|f\|_1 \phi(n_k(s,a))^2 + \phi(n_k(s,a)) \sum_{s' \in \statespace} \sqrt{\frac{c_1}{\tilde P(s'|s,a)}} \tilde P(s' |s, a) f(s')\,.
		\\
		\shortintertext{Splitting the last sum based on whether $\sqrt{\tilde P(s'|s,a)}$ is smaller or larger than $\sqrt{c_1} C \phi(n_k(s,a))$}
		\leq & c_2 \|f\|_1 \phi(n_k(s,a))^2 + \frac{1}{C} \tilde P(s,a) f\\
		&+ \phi(n_k(s,a)) \sum_{s' \in \statespace} \sqrt{\frac{c_1}{\tilde P(s'|s,a)}} \tilde P(s' |s, a) f(s')\one\{ \sqrt{\tilde P(s' | s, a)} < \sqrt{c_1} C \phi(n_k(s,a))\}
		\\
				\leq & c_2 \|f\|_1 \phi(n_k(s,a))^2 + \frac{1}{C} \tilde P(s,a) f + \phi(n_k(s,a))^2 \sum_{s' \in \statespace} c_1 C  f(s')
		\\
		\leq & c_2 \|f\|_1 \phi(n_k(s,a))^2 + \frac{1}{C} \tilde P(s,a) f
		+ \phi(n_k(s,a))^2 c_1 C \|f\|_1\\
		\leq &  (c_1 C+c_2) \|f\|_1 \phi(n_k(s,a))^2 + \frac{1}{C} \tilde P(s,a) f \,.
	\end{align}\end{proof}

\begin{lemma}[Rate Lemma, Adaption of Lemma~E.3 by \citet{dann2017unifying}]
Fix $r \geq 1$, $\epsilon' > 0$, $C>0$ and $D \geq 1$, where $C$ and $D$ may depend polynomially on relevant quantities.  Let $w_{k,h}(s,a) = \prob(s_{k,h}=s, a_{k,h}=a |s_{k,1}, \pi_k)$ be the probability of encountering $s,a$ at time $h$ in the $k$th episode. Then for any functions $\gamma_{k,h}~:~ \saspace \rightarrow \reals_+$ indexed by $h \in [H]$
\begin{align}
    \sum_{s,a \in L_{k}} \sum_{h=1}^H
    w_{k,h}(s, a) \gamma_{k,h}(s,a) \left(\frac{C (\llnp(2n_{k}(s,a)) + D)}{n_{k}(s,a)} \right)^{1/r} \leq \epsilon'
\end{align}
on all but at most 
\begin{align}
    \frac{6C \numA \numS B^{r-1}}{\epsilon'^r} \polylog(\numS, \numA, H, \delta^{-1}, \epsilon'^{-1}),
    \quad \textrm{where} \quad B \geq  \sum_{s,a \in L_{k}} \sum_{h=1}^H w_{k,h}(s, a) \gamma_{k,h}(s,a)^{r/(r-1)}
    \end{align}
    nice episodes.
    \label{lem:mainratelemma}
\end{lemma}

\begin{proof}
    Define 
\begin{align}
    \Delta_k =&
    \sum_{s,a \in L_{k}}
    \sum_{h=1}^H
    w_{k,h}(s, a) \gamma_{k,h}(s,a) \left(\frac{C(\llnp(2n_{k}(s,a))+D)}{n_{k}(s,a)}\right)^{1/r}\\
    =&
    \sum_{s,a \in L_{k}} \sum_{h=1}^H 
    (w_{k}(s, a) \gamma_{k,h}(s,a)^{r/(r-1)} )^{1 - \frac 1 r} \left(w_{k,h}(s, a)\frac{C(\llnp(2n_{k}(s,a))+D)}{n_{k}(s,a)}\right)^{1/r}.
\end{align}
    We first bound using H\"older's inequality
    \begin{align}
        \Delta_k \leq &
        \left( \sum_{s,a \in L_{k}} \sum_{h=1}^H w_{k,h}(s, a) \gamma_{k,h}(s,a)^{r/(r-1)} \right)^{1 - \frac 1 r} 
        \left(\sum_{s, a \in L_{k}} \sum_{h=1}^H 
        \frac{C w_{k,h}(s, a)(\llnp(2n_{k}(s,a))+D)}{n_{k}(s,a)}\right)^{\frac 1 r}\\
        \leq & \left(\sum_{s, a \in L_{k}} 
        \frac{C B^{r-1} w_{k}(s, a)(\llnp(2n_{k}(s,a))+D)}{n_{k}(s,a)}\right)^{\frac 1 r}.
    \end{align}
    Using the property in Lemma~\ref{lem:nicenessep} of nice episodes as well as the fact that $w_{k}(s,a) \leq H$ and $\sum_{i < k} w_{i}(s,a) \geq 4 H \ln \frac {\numS \numA H}{\delta'} \geq 4 H \ln(2) \geq 2 H$, we bound
    \begin{align}
        n_{k}(s,a) \geq \frac 1 4 \sum_{i < k} w_{i}(s,a) \geq \frac{1}{6} \sum_{i \leq k} w_{i}(s,a).
    \end{align}
    The function $\frac{\llnp(2x) + D}{x}$ is monotonically decreasing in $x \geq 0$ since $D \geq 1$ (see Lemma~\ref{lem:llnpprop}).
    This allows us to bound
    \begin{align}
        \Delta^r_k \leq &
        \sum_{s, a \in L_{k}}
        \frac{CB^{r-1} w_{k}(s, a)(\llnp(2n_{k}(s,a))+D)}{n_{k}(s,a)}\\
        \leq & 6CB^{r-1}
        \sum_{s, a \in L_{k}}
        \frac{w_{k}(s, a)\left(\llnp\left( \frac 1 3 \sum_{i \leq k} w_{i}(s,a)\right)+D\right)}{\sum_{i \leq k} w_{i}(s,a)}\\
        \leq & 6 CB^{r-1}
        \sum_{s, a \in L_{k}}
        \frac{w_{k}(s, a)\left(\llnp\left(\sum_{i \leq k} w_{i}(s,a)\right)+D\right)}{\sum_{i \leq k} w_{i}(s,a)}\\
        \leq & 6 CB^{r-1}
         \left[ 0 \vee 
        \max_{s, a \in L_{k}} \frac{ \llnp\left(\sum_{i \leq k} w_{i}(s,a)\right)+D}{\sum_{i \leq k} w_{i}(s,a)} \right] \sum_{s, a \in L_{k}} w_{k}(s, a) \\
        \shortintertext{which can be further bounded by leveraging that the sum of all weights $w_k$ sum to $H$ for each episode $k$}
                \leq & 6 C H B^{r-1}
         \left[ 0 \vee
        \max_{s, a \in L_{k}} \frac{ \llnp\left(\sum_{i \leq k} w_{i}(s,a)\right)+D}{\sum_{i \leq k} w_{i}(s,a)} \right].
    \end{align}
    Assume now $\Delta_k > \epsilon'$. In this case, the right-hand side of the inequality above is also larger than $\epsilon'^r$ and there is at least one $(s, a)$ for which $w_{k}(s,a) > w_{\min}$
    and
    \begin{align}
        \frac{ 6 C H B^{r-1}\left(\llnp\left(\sum_{i \leq k} w_{i}(s,a) \right)+D \right)}{\sum_{i \leq k} w_{i}(s, a)} >& \epsilon'^r \\
        \Leftrightarrow
        \frac{\llnp\left(\sum_{i \leq k} w_{i}(s,a) \right)+D}{\sum_{i \leq k} w_{i}(s, a)} >& \frac{\epsilon'^r}{6 C H B^{r-1}}.
    \end{align}
    Let us denote $C' = \frac{6 C H B^{r-1}}{\epsilon'^r}$. Since
    $\frac{\llnp(x) + D}{x}$ is monotonically decreasing and $x = C'^2 + 3C'D$
    satisfies $\frac{\llnp(x) + D}{x} \leq \frac{\sqrt{x} + D}{x} \leq \frac 1
    {C'}$, we know that if $\sum_{i \leq k} w_{i}(s,a) \geq C'^2 + 3C' D$ then
    the above condition cannot be satisfied for $s,a$. Since each time the
    condition is satisfied, it holds that $w_{k}(s,a) > w_{\min}$ and so
    $\sum_{i \leq k} w_{i}(s,a)$ increases by at least $w_{\min}$, it can happen at most  
    \begin{align}
        m \leq \frac{\numS \numA (C'^2 + 3C' D)}{w_{\min}} 
    \end{align}
    times that $\Delta_k > \epsilon'$. Define $K = \{ k : \Delta_k > \epsilon' \} \cap N$ and we know that $|K| \leq m$.
        Now we consider the sum
 \begin{align}
        \sum_{k \in K} \Delta_k^r 
        \leq &
        \sum_{k \in K} 
        6CB^{r-1}
        \sum_{s, a \in L_{k}}
        \frac{w_{k}(s, a) \left(\llnp\left(\sum_{i \leq k} w_{i}(s,a)\right)+D\right)}{\sum_{i \leq k} w_{i}(s,a)}\\
        \leq &
        6CB^{r-1}\left(\llnp\left(C'^2 + 3C'D \right)+D\right)
        \sum_{s, a \in L_{k}} 
            \sum_{k \in K}
        \frac{w_{k}(s, a) }{\sum_{i \leq k} w_{i}(s,a) \one\{w_{i}(s, a) \geq w_{\min}\}}\,.
    \end{align}
    For every $(s, a)$, we consider the sequence of $w_{i}(s, a) \in [w_{\min}, H]$ with $i \in I = \{ i \in \mathbb N \, : \, w_{i}(s, a) \geq w_{\min}\}$ and apply Lemma~\ref{lem:logseq}.
    This yields that
    \begin{align}
            \sum_{k \in K}
        \frac{w_{k}(s, a)}{\sum_{i \leq k} w_{i}(s,a) \one\{w_{i}(s, a) \geq w_{\min}\}}
        \leq 
        1 + \ln(mH / w_{\min}) = \ln\left(\frac{Hme}{w_{\min}}\right)
    \end{align}
    and hence
    \begin{align}
        \sum_{k \in K} \Delta_k^r 
        \leq &
        6 C\numA \numS B^{r-1} \ln\left(\frac{Hme}{w_{\min}}\right)\left(\llnp\left(C'^2 + 3C'D \right)+D\right)\,.
    \end{align}
    Since each element in $K$ has to contribute at least $\epsilon'^r$ to this bound, we can conclude that
    \begin{align}
        \sum_{k \in N} \one\{\Delta_k \geq \epsilon'\} \leq \sum_{k \in K} \one\{\Delta_k \geq \epsilon'\} \leq |K| \leq  
        \frac{6 C\numA \numS B^{r-1}}{\epsilon'^r} \ln\left(\frac{Hme}{w_{\min}}\right)\left(\llnp\left(C'^2 + 3C'D \right)+D\right).
    \end{align}
    Since $\ln\left(\frac{Hme}{w_{\min}}\right)\left(\llnp\left(C'^2 + 3C'D \right)+D\right)$ is $\polylog(\numS, \numA, H, \delta^{-1}, \epsilon'^{-1})$, the proof is complete.
\end{proof}
\begin{corollary}
Fix $r \geq 1$, $\epsilon' > 0$, $C > 0$, and $D \geq 1$, where $C$ and $D$ may depend polynomially on relevant quantities.  
Then,
\begin{align}
    \sum_{s,a \in L_{k}}
    w_{k}(s, a) \left(\frac{C (\llnp(2n_{k}(s,a)) + D)}{n_{k}(s,a)} \right)^{1/r} \leq \epsilon'
\end{align}
on all but at most 
\begin{align}
    \frac{6C \numA \numS H^{r-1}}{\epsilon'^r} \polylog(\numS, \numA, H, \delta^{-1}, \epsilon'^{-1}).
    \end{align}
    nice episodes.
    \label{cor:mainratelemma}
\end{corollary}
\begin{proof}
This corollary follows directly from Lemma~\ref{lem:mainratelemma} with $\gamma_h(s,a) = 1$ and noting that $H \geq  
\sum_{s,a \in \saspace} \sum_{h=1}^H w_{k,h}(s, a) \geq 
\sum_{s,a \in L_{k}} \sum_{h=1}^H w_{k,h}(s, a) \gamma_{k,h}(s,a)^{r/(r-1)}$.
\end{proof}

\begin{lemma}[Lemma~E.5 by \citet{dann2017unifying}]
    Let $a_i$ be a sequence taking values in $[a_{\min}, a_{\max}]$ with $a_{\max} \geq a_{\min} > 0$ and $m > 0$, then
    \begin{align}
        \sum_{k=1}^m \frac{a_k}{\sum_{i=1}^k a_i} \leq \ln\left( \frac{mea_{\max}}{a_{\min}}\right).
    \end{align}
    \label{lem:logseq}
\end{lemma}

\begin{lemma}[Properties of $\llnp$, Lemma~E.6 by \citet{dann2017unifying}]
    The following properties hold:
    \begin{enumerate}
        \item $\llnp$ is continuous and nondecreasing.
        \item $f(x) = \frac{\llnp(nx) + D}{x}$ with $n \geq 0$ and $D \geq 1$ is monotonically decreasing on $\reals_+$.
        \item $\llnp(x y) \leq \llnp(x) + \llnp(y) + 1$ for all $x,y \geq 0$.
    \end{enumerate}
    \label{lem:llnpprop}
\end{lemma}

\begin{corollary}
\label{cor:subgammaboundary}
    Using the terminology by \citet{howard2018uniform}, for any $c > 0$, $\delta \in (0,1)$, the following function is a sub-gamma boundary (and as such also a sub-exponential boundary) with scale parameter $c$ and crossing probability $\delta$:
    \begin{align}
        u_{c, \delta}(v) = 1.44 \sqrt{ v \left( 1.4 \llnp(2v) + \log \frac{5.2}{\delta}\right)} + 2.42 c \left(1.4 \llnp(2v) + \log \frac{5.2}{\delta}\right).
    \end{align}
    Further $u_{0, \delta}$ is a sub-Gaussian boundary with crossing probability $\delta$ and $u_{c/3, \delta}$ is a sub-Poisson boundary with crossing probability $\delta$ for scale parameter $c$.
\end{corollary}
\begin{proof}
This result follows directly from Theorem~1 by \citet{howard2018uniform} instantiated with $h(k) = (k+1)^s \zeta(s), s=1.4$ and $\eta = 2$. The final statements follows from the fact that $\psi_{N}$ is a special case of $\psi_{G}$ with $c=0$ and Proposition~5 in \citet{howard2018uniform}.
\end{proof}

\section{Theoretical analysis of Algorithm~\ref{alg:ofucontext} for finite episodic MDPs with side information}
\subsection{Failure event and bounding the failure probability}
We define the following failure event
\begin{align}
    F = F\fr \cup F\fp \cup F^O 
\end{align}
where
\begin{align}
    F\fr =& \left\{\exists s,a \in \saspace, k \in \NN~:~ 
    \| \hat \theta\fr_{k,s,a} - \theta\fr_{s,a} \|_{N\fr_{k,s,a}} 
    \geq 
    \sqrt{\lambda} \|\theta\fr_{s,a}\|_2
    + \sqrt{\frac 1 2 \log \frac 1 {\delta'} + \frac 1 4 \log \frac{\det N\fr_{k,s,a}}{\det \lambda I}}
    \right\}\,,\\
    F\fp =& \bigg\{\exists s',s,a \in \Scal \times \saspace, k \in \NN~:~ 
    \| \hat \theta\fp_{k,s',s,a} - \theta\fp_{s',s,a} \|_{N\fp_{k,s,a}} 
    \\&\qquad\geq 
    \sqrt{\lambda} \|\theta\fp_{s',s,a}\|_2
    + \sqrt{\frac 1 2  \log \frac 1 {\delta'} + \frac 1 4\log \frac{\det N\fp_{k,s,a}}{\det \lambda I}}
    \bigg\}\,,\\
    F^O =& \bigg \{ \exists T \in \NN~:~
    \sum_{k=1}^T \sum_{s,a \in \saspace}\sum_{h=1}^H [\prob(s_{k,h}=s, a_{k,h}=a |s_{k,1}, \pi_k) - \one\{s = s_{k,h}, a = a_{k,h}\}] 
    \\ & \qquad\geq \numS H \sqrt{T \log \frac{6 \log (2T)}{\delta'}}
    \bigg\}\,, \\
\delta' =& \frac{\delta}{\numS \numA + \numS^2  \numA + \numS H}\,.
\end{align}

\begin{lemma}
\label{lem:context_failureprob}
The failure probability $\prob(F)$ is bounded by $\delta$.
\end{lemma}
\begin{proof}
Consider an arbitrary $s \in \statespace$, $a \in \actionspace$ and define $\Fcal_t$ where $t = H k + h$ with $h \in [H]$ is the running time step index as follows: $\Fcal_t$ is the sigma-field induced by all observations up to $s_{k,h}$ and $a_{k, h}$ including $x_k$ but not $r_{k,h}$ and not $s_{k,h+1}$. Then
$\eta_{t} = 2\one\{ s_{k,h} = s, a_{k,h} = a\}((x\fr_k)^\top \theta\fr_{s,a} - r_{k,h})$ is a martingale difference sequence adapted to $\Fcal_t$.  Moreover, since $\eta_t$ takes values in $[2(x\fr_k)^\top \theta\fr_{s,a} - 2, 2(x\fr_k)^\top \theta\fr_{s,a}]$ almost surely it is conditionally sub-Gaussian with parameter $1$. We can then apply Theorem~20.2 in \citet{lattimore2018bandit} to get 
\begin{align}
    2\| \hat \theta\fr_{k,s,a} - \theta\fr_{s,a} \|_{N\fr_{k,s,a}} 
    \leq 
    \sqrt{\lambda} 2\|\theta\fr_{s,a}\|_2
    + \sqrt{2 \log \frac 1 {\delta'} + \log \frac{\det N\fr_{k,s,a}}{\det \lambda I}}
\end{align}
for all $k \in \NN$ with probability at least $ 1- \delta'$. Similarly for any fixed $s' \in \statespace$, using $\eta_t = 2\one\{ s_{k,h} = s, a_{k,h} = a\}((x\fp_k)^\top \theta\fp_{s',s,a} - \one\{s_{k,h+1} = s'\})$, it holds with probability at least $1 - \delta'$ that 
\begin{align}
        \| \hat \theta\fp_{k,s',s,a} - \theta\fp_{s',s,a} \|_{N\fp_{k,s,a}} 
    \leq 
    \sqrt{\lambda} \|\theta\fp_{s',s,a}\|_2
    + \sqrt{\frac 1 2  \log \frac 1 {\delta'} + \frac 1 4\log \frac{\det N\fp_{k,s,a}}{\det \lambda I}}
\end{align}
for all episodes $k$. Finally, for a fixed $s \in \statespace$ and $h \in [H]$ the sequence 
\begin{align}
    \eta_k = \sum_{a \in \actionspace} [\prob(s_{k,h}=s, a_{k,h}=a |s_{k,1}, \pi_k) - \one\{s = s_{k,h}, a = a_{k,h}\}]
\end{align}
is a martingale difference sequence with respect to $\Gcal_k$, defined as the sigma-field induced by all observations up to including episode $k-1$ and $x_k$ and $s_{k,1}$. 
All but at most one action has zero probability of occurring ($\pi_k$ is deterministic)
and therefore $\eta_k \in [c, c+1]$ with probability $1$ for some $c$ that is measurable in $\Gcal_k$. Hence, $S_t = \sum_{k=1}^t \eta_k$ satisfies Assumption~1 with $V_t = t/4$ and $\psi_N$ and $\Ex L_0 = 1$ (Hoeffding I case in Table~2 of the appendix). This allows us to apply Theorem~1 by \citet{howard2018uniform} where we choose $h(k) = (1 + k)^s \zeta(s)$ with $s = 1.4$ and $\eta = 2$, which gives us (see Eq. (8) and Eq. (9))
specifically) that with probability at least $1 - \delta'$ for all $T \in \NN$
\begin{align}
    \sum_{k=1}^T \sum_{a \in \actionspace} [\prob(s_{k,h}=s, a_{k,h}=a |s_{k,1}, \pi_k) - \one\{s = s_{k,h}, a = a_{k,h}\}] = \sum_{k=1}^T \eta_k 
    \leq \sqrt{T (\log \log (T/2) + \log (6 / \delta'))}.
\end{align}
Setting $\delta' = \frac{\delta}{\numS \numA + \numS^2  \numA + \numS H}$, all statements above hold for all $s', s, a, h$ jointly using a union bound with probability at least $1 - \delta$. This implies that $\prob(F) \leq \delta$.
\end{proof}

Using the bounds on the linear parameter estimates, the following lemma derives bounds on the empirical model.
\begin{lemma}[Bounds on model parameters]
\label{lem:modelbounds}
Outside the failure event $F$, assuming $\|\theta\fp_{s',s,a}\|_2 \leq \xi_{\theta\fp}$ and $\|\theta\fr_{s,a}\|_2 \leq \xi_{\theta\fr}$ for all $s',s \in \statespace$ and $a \in \actionspace$ we have
\begin{align}
    |\hat r_k(s,a) - r_k(s,a)| &\leq 1 \wedge \alpha_{k,s,a} \|x_k\fr\|_{(N\fr_{k,s,a})^{-1}} \\
    |\hat P_k(s' | s,a) - P_k(s,a)\|_1 &\leq 1 \wedge \gamma_{k,s,a} \|x_k\fp\|_{(N\fp_{k,s,a})^{-1}}
\end{align}
where
\begin{align}
    \alpha_{k,s,a} =& \sqrt{\lambda}\xi_{\theta\fr} + \sqrt{\frac 1 2 \log \frac 1 {\delta'} + \frac 1 4 \log \frac{\det N\fr_{k,s,a}}{\det (\lambda I)}}\\
    \gamma_{k,s,a} =& \sqrt{\lambda}\xi_{\theta\fp} + \sqrt{\frac 1 2 \log \frac 1 {\delta'} + \frac 1 4 \log \frac{\det N\fp_{k,s,a}}{\det (\lambda I)}}\,.
\end{align}
\end{lemma}
\begin{proof}
Since $\hat r_k \in [0,1]$ and $r_k \in [0,1]$, we have    
\begin{align}
    |\hat r_k(s,a) - r_k(s,a)| 
    \leq 1 \wedge |(x\fr_k)^\top \hat \theta\fr_{k,s,a} - r_k(s,a)|. 
\end{align}
The last term can be bounded as
\begin{align}
    &|(x\fr_k)^\top \hat \theta\fr_{k,s,a} - r_k(s,a)| = |(x\fr_k)^\top (\hat \theta\fr_{k,s,a} - \theta\fr_{s,a})|
    \leq \|x_k\fr\|_{(N\fr_{k,s,a})^{-1}} \|\hat \theta\fr_{k,s,a} - \theta\fr_{s,a}\|_{N\fr_{k,s,a}} \\
    \leq& \|x_k\fr\|_{(N\fr_{k,s,a})^{-1}}  \left[
    \sqrt{\lambda}\|\theta\fr_{s,a}\|_2 + 
    \sqrt{\frac 1 2 \log \frac{1}{\delta'} + \frac 1 4 \log \frac{\det(N\fr_{k,s,a})}{\det(\lambda I)}}
    \right]\\
    \leq& \alpha_{k,s,a} \|x_k\fr\|_{(N\fr_{k,s,a})^{-1}} 
\end{align}
where we first used H\"older's inequality, then the definition of $F\fr$, and finally the assumption $\|\theta\fr_{s,a}\|_2 \leq \xi_{\theta\fr}$.
This proves the first inequality. Consider now the second inequality, which we bound analogously as
\begin{align}
    &|\hat P_k(s'|s,a) - P_k(s'|s,a)|
    \leq 1 \wedge |(x_k\fp)^\top\hat \theta\fp_{k,s',s,a} - P_k(s'|s,a)|
    \\
    =& 1 \wedge |(x_k\fp)^\top(\hat \theta\fp_{k,s',s,a} - \theta\fp_{s',s,a})| 
    \leq  1 \wedge \|x_k\fp\|_{(N\fp_{k,s,a})^{-1}} \|\hat \theta\fp_{k,s',s,a} - \theta\fp_{s',s,a}\|_{N\fp_{k,s,a}} \\
    \leq& 1 \wedge \|x_k\fp\|_{(N\fp_{k,s,a})^{-1}} \left[
    \sqrt{\lambda} \|\theta_{s',s,a}\|_2 + 
    \sqrt{\frac 1 2 \log \frac{1}{\delta'} + \frac 1 4 \log \frac{\det(N\fp_{k,s,a})}{\det(\lambda I)}}
    \right]\\
    \leq &1 \wedge \gamma_{k,s,a} \|x_k\fp\|_{(N\fp_{k,s,a})^{-1}}\,.
\end{align}
\end{proof}

\subsection{Admissibility of guarantees}
\begin{lemma}[Upper bound admissible]
\label{lem:validupperboundcon}
Outside the failure event $F$, for all episodes $k$, $h \in [H]$ and $s,a \in \saspace$
    \begin{align}
        Q^\star_{k,h}(s,a) \leq \Qub_{k,h}(s,a).
    \end{align}
\end{lemma}
\begin{proof}
        Consider a fixed episode $k$. For $h= H+1$ the claim holds by definition. Assume the claim holds for $h+1$ and consider
        $\Qub_{k,h}(s,a) - Q^\star_{k,h}(s,a)$. Since $Q^\star_{k,h} \leq V^{\max}_h$, this quantity is non-negative when $\Qub_{k,h}(s,a) = V^{\max}_h$. In the other case
    \begin{align}
    &\Qub_{k,h}(s,a) - Q^\star_{k,h}(s,a) \\
    \geq &
    \hat r_k(s,a)  + \hat P_k(s,a) \Vub_{k,h+1}+ \psi_{k,h}(s,a) - P_k(s,a) V^\star_{k,h+1} - r_k(s,a) \\
    = &
    \hat r_k(s,a) - r_k(s,a) + \psi_{k,h}(s,a) + \hat P_{k}(s,a)(\Vub_{k,h+1} - V^\star_{k,h+1})\\
   & 
   + (\hat P_k(s,a) - P_k(s,a))V^\star_{k,h+1} \\
   \shortintertext{by induction hypothesis and $\hat P_k(s'|s,a) \geq 0$ }
   \geq &
    \hat r_k(s,a) - r_k(s,a) + \psi_{k,h}(s,a)
   + (\hat P_k(s,a) - P_k(s,a))V^\star_{k,h+1} \\
   \geq &
    -|\hat r_k(s,a) - r_k(s,a)| + \hat \psi_{kh}(s,a)
   - \sum_{s' \in \statespace} V^\star_{k,h+1}(s')|\hat P_k(s' |s,a) - P_k(s' |s,a)| \\
   \shortintertext{by induction hypothesis}
   \geq &
    -|\hat r_k(s,a) - r_k(s,a)| + \hat \psi_{kh}(s,a)
   - \sum_{s' \in \statespace} \Vub_{h+1}(s')|\hat P_k(s' |s,a) - P_k(s' |s,a)| \\
    \shortintertext{using Lemma~\ref{lem:modelbounds}}
    \geq &
    \psi_{k,h}(s,a) - \alpha_{k,s,a} \|x_k\fr\|_{(N\fr_{k,s,a})^{-1}} - \|\Vub_{h+1}\|_1 
    \gamma_{k,s,a}\|x\fp_k\|_{(N\fp_{k,s,a})^{-1}} = 0\,.
    \end{align}
\end{proof}

Using the same technique, we can prove the following result.
\begin{lemma}[Lower bound admissible]
\label{lem:validlowerboundcon}
Outside the failure event $F$, for all episodes $k$, $h \in [H]$ and $s,a \in \saspace$
    \begin{align}
        Q^{\pi_k}_{k,h}(s,a) \geq \Qlb_{k,h}(s,a).
    \end{align}
\end{lemma}
\begin{proof}
        Consider a fixed episode $k$. For $h= H+1$ the claim holds by definition. Assume the claim holds for $h+1$ and consider
        $Q^{\pi_k}_{k,h}(s,a) - \Qlb_{k,h}(s,a)$. Since $Q^{\pi_k}_{k,h} \geq 0$, this quantity is non-negative when $\Qlb_{k,h}(s,a) = 0$. In the other case
    \begin{align}
    &Q^{\pi_k}_{k,h}(s,a) - \Qlb_{k,h}(s,a)\\
    = &
    P_k(s,a) V^{\pi_k}_{k,h+1} + r_k(s,a)
    - \hat r_k(s,a)  - \hat P_k(s,a) \Vlb_{k,h+1} + \psi_{k,h}(s,a) \\
    = &
    r_k(s,a) - \hat r_k(s,a) +  \psi_{k,h}(s,a) 
    + P_{k}(s,a)(V^{\pi_k}_{k,h+1} - \Vlb_{k, h+1})  + (P_k - \hat P_k)(s,a)\Vlb_{k,h+1}
   \shortintertext{by induction hypothesis and $P_k(s'|s,a) \geq 0$}
   \geq &
   \psi_{k,h}(s,a) - |r_k(s,a) - \hat r_k(s,a)| 
   - |(P_k(s,a) - \hat P_k(s,a))\Vlb_{k,h+1}| \\
   \shortintertext{using Lemma~\ref{lem:modelbounds}}
    \geq &
    \psi_{k,h}(s,a) - \alpha_{k,s,a}\|x\fr_k\|_{(N\fr_{k,s,a})^{-1}} - \|\Vlb_{h+1}\|_1 \gamma_{k,s,a}\|x\fp_k\|_{(N\fp_{k,s,a})^{-1}} = 0\,.
    \end{align}
\end{proof}

\subsection{Cumulative certificate bound}

\begin{lemma}
\label{lem:certbound_context}
Outside the failure event $F$, the cumulative certificates after $T$ episodes for all $T$ are bounded by
\begin{align}
        \sum_{k=1}^T \epsilon_k
     \leq & \tilde O \left(\sqrt{\numS^3 \numA H^2 T}V^{\max}_1 \lambda (\xi^2_{\theta\fp}  + \xi^2_{\theta\fr} + d\fp + d\fr)
    \log \frac{\xi^2_{x\fp}  + \xi^2_{x\fr}}{\lambda \delta} \right).
\end{align}
\end{lemma}
\begin{proof}
Let $\psi_{k,h}(s,a) = \alpha_{k,s,a}\|x\fr_k\|_{(N\fr_{k,s,a})^{-1}} +  V^{\max}_{h+1} \numS \gamma_{k,s,a}\|x\fp_k\|_{(N\fp_{k,s,a})^{-1}}$
We bound the difference between upper and lower Q-estimate as
\begin{align}
    &\Qub_{k,h}(s,a) - \Qlb_{k,h}(s,a)\\
    &\leq 2\psi_{k,h}(s,a) 
         + \hat P_k(s,a)^\top (\Vub_{k,h+1} - \Vlb_{k, h+1})\\
    &= 2\psi_{k,h}(s,a)
        + (\hat P_k(s,a) - P_k(s,a))^\top (\Vub_{k,h+1} - \Vlb_{k, h+1})
         + P_k(s,a)(\Vub_{k,h+1} - \Vlb_{k, h+1})\\
    &\leq 2\psi_{k,h}(s,a) 
         + V^{\max}_{h+1}\|\hat P_k(s,a) - P_k(s,a)\|_1
          + P_k(s,a)(\Vub_{k,h+1} - \Vlb_{k, h+1})
\end{align}
and by construction we also can bound $\Qub_{k,h}(s,a) - \Qlb_{k,h}(s,a) \leq V^{\max}_{h}$. Applying both bounds above recursively, we arrive at
\begin{align}
		&\epsilon_k = (\Vub_{k,1} - \Vlb_{k,1})(s_{k,1}) = (\Qub_{k,1} - \Qlb_{k,1})(s_{k,1},\pi_k(s_{k,1}, 1))\\
		&\leq  \sum_{s,a \in \saspace}\sum_{h=1}^H \prob(s_h=s, a_h=a |s_{k,1}, \pi_k) [V_{h}^{\max} \wedge (2\psi_{k,h}(s,a) 
         + V_{h+1}^{\max}\|\hat P_k(s,a) - P_k(s,a)\|_1)]
		\\
		&\leq  \sum_{s,a \in \saspace}\sum_{h=1}^H \prob(s_h=s, a_h=a |s_{k,1}, \pi_k) [V_{h}^{\max} \wedge
		(2\alpha_{k,s,a}\|x\fr_k\|_{(N\fr_{k,s,a})^{-1}} + 3 V^{\max}_{h+1} \numS\gamma_{k,s,a}\|x\fp_k\|_{(N\fp_{k,s,a})^{-1}})]
\end{align}
where we used Lemma~\ref{lem:modelbounds} in the last step.
We are now ready to bound the cumulative certificates after $T$ episodes as 
\begin{align}
    \lefteqn{\sum_{k=1}^T \epsilon_k} \nonumber \\
     \leq & \sum_{k=1}^T  \sum_{s,a \in \saspace}\sum_{h=1}^H \prob(s_{k,h}=s, a_{k,h}=a |s_{k,1}, \pi_k) [V_{h}^{\max} \wedge
		(2\alpha_{k,s,a}\|x\fr_k\|_{(N\fr_{k,s,a})^{-1}} + 3 V^{\max}_{h+1} \numS \gamma_{k,s,a}\|x\fp_k\|_{(N\fp_{k,s,a})^{-1}})]\\
     \leq & \sum_{k=1}^T \sum_{h=1}^H [V_h^{\max} \wedge
		(2\alpha_{k,s_{k,h},a_{k,h}}\|x\fr_k\|_{(N\fr_{k,s_{k,h},a_{k,h}})^{-1}} + 3 V^{\max}_{h+1} \numS \gamma_{k,s_{k,h},a_{k,h}}\|x\fp_k\|_{(N\fp_{k,s_{k,h},a_{k,h}})^{-1}})]\\
     &+ \sum_{k=1}^T \sum_{s,a \in \saspace}\sum_{h=1}^H [\prob(s_{k,h}=s, a_{k,h}=a |s_{k,1}, \pi_k) - \one\{s = s_{k,h}, a = a_{k,h}\}] V^{\max}_h\\
     \intertext{applying definition of failure event $F^O$}
     \leq & \sum_{k=1}^T \sum_{h=1}^H [V_h^{\max} \wedge
		(2\alpha_{k,s_{k,h},a_{k,h}}\|x\fr_k\|_{(N\fr_{k,s_{k,h},a_{k,h}})^{-1}} + 3 V^{\max}_{h+1} \numS \gamma_{k,s_{k,h},a_{k,h}}\|x\fp_k\|_{(N\fp_{k,s_{k,h},a_{k,h}})^{-1}})]\\
		&+ V_1^{\max} \numS H \sqrt{T  \log \frac{6 \log (2T)}{\delta'}}\\
	\shortintertext{splitting reward and transition terms}
	     \leq & \sum_{k=1}^T \sum_{h=1}^H [V_h^{\max} \wedge
		2\alpha_{k,s_{k,h},a_{k,h}}\|x\fr_k\|_{(N\fr_{k,s_{k,h},a_{k,h}})^{-1}}]
		\label{eqn:term1_ellipsoid_potential}\\
		&+ \sum_{k=1}^T \sum_{h=1}^H [V_h^{\max} \wedge
		 3 V^{\max}_{h+1} \numS\gamma_{k,s_{k,h},a_{k,h}}\|x\fp_k\|_{(N\fp_{k,s_{k,h},a_{k,h}})^{-1}}]
		 \label{eqn:term2_ellipsoid_potential}
		 \\
		&+ V_1^{\max} \numS H \sqrt{T  \log \frac{6 \log (2T)}{\delta'}}.
		\label{eqn:term3_martingale}
\end{align}
Before bounding the first two terms further, we first derive the following useful inequality using AM-GM inequality which holds for any $s \in \actionspace$ and $s \in \statespace$
\begin{align}
\log \frac{\det N\fr_{k,s,a}}{\det(\lambda I)}
\leq \log \frac{ \left(\frac 1 d\fr \tr N\fr_{k, s, a}\right)^{d\fr}}{\lambda^{d\fr}}
= d\fr \log \frac{ \tr N\fr_{k, s, a}}{d\fr \lambda}
\leq 
 d\fr\log \frac{ d\fr \lambda + \xi_{x\fr}^2 (k-1) H }{d\fr \lambda}\label{eqn:logdetineq}
\end{align}
where in the last inequality we used the fact that $N\fr_{k, s, a}$ is the sum of $\lambda I$ and at most $H(k-1)$ outer products of feature vectors. Analogously, the following inequality holds for the covariance matrix of the transition features
\begin{align}
    \log \frac{\det N\fp_{k,s,a}}{\det(\lambda I)}
\leq 
 d\fp\log \frac{ d\fp \lambda + \xi_{x\fp}^2 (k-1) H }{d\fp \lambda}.
\end{align}
This inequality allows us to upper-bound for $k \leq T$
\begin{align}
    \alpha_{k,s_{k,h},a_{k,h}} 
    =& \sqrt{\lambda}\xi_{\theta\fr} + \sqrt{\frac 1 2 \log \frac 1 {\delta'} + \frac 1 4 \log \frac{\det N\fr_{k,s_{k,h},a_{k,h}}}{\det (\lambda I)}}\\
    \leq & \sqrt{\lambda}\xi_{\theta\fr} + \sqrt{\frac 1 2 \log \frac 1 {\delta'} + \frac 1 4 d\fr\log \frac{ d\fr \lambda + \xi_{x\fr}^2 (k-1) H }{d\fr \lambda}}\\
        \leq & \sqrt{\lambda}\xi_{\theta\fr} + \sqrt{\frac 1 2 d\fr\log \frac{ d \lambda + \xi_{x\fr}^2 HT }{d\fr \lambda \delta'}} \\
        \shortintertext{using the fact that $\sqrt{a} + \sqrt{b} \leq 2 \sqrt{a + b}$ for all $a,b \in \RR_+$}
        \leq & 2\sqrt{\lambda \xi_{\theta\fr}^2 + \frac 1 2 d\fr\log \frac{ d \lambda + \xi_{x\fr}^2 HT }{d\fr \lambda \delta'}} \\
        \leq & 2 V_1^{\max}\sqrt{\frac 1 4 + \lambda \xi_{\theta\fr}^2 + \frac 1 2 d\fr\log \frac{ d\fr \lambda + \xi_{x\fr}^2 HT }{d\fr \lambda \delta'}} 
        =: \alpha_T. \label{eqn:alphaT}
\end{align}
Note that the last inequality ensures $\alpha_T \geq V_1^{\max}$.
We now use $\alpha_T$ to bound the term in Equation~\eqref{eqn:term1_ellipsoid_potential}
\begin{align}
    & \sum_{k=1}^T \sum_{h=1}^H [V_h^{\max} \wedge
	2\alpha_{k,s_{k,h},a_{k,h}}\|x\fr_k\|_{(N\fr_{k,s_{k,h},a_{k,h}})^{-1}}]
	    \leq 
	        \sum_{k=1}^T \sum_{h=1}^H [V_h^{\max} \wedge
	2\alpha_T\|x\fr_k\|_{(N\fr_{k,s_{k,h},a_{k,h}})^{-1}}]\\
	    \leq 
	        & 2 \alpha_T \sum_{k=1}^T \sum_{h=1}^H [1 \wedge
	\|x\fr_k\|_{(N\fr_{k,s_{k,h},a_{k,h}})^{-1}}]\\
	\shortintertext{using Cauchy-Schwarz inequality}
		    \leq 
	        & \sqrt{4 \alpha_T^2 TH \sum_{k=1}^T \sum_{h=1}^H [1 \wedge
	\|x\fr_k\|^2_{(N\fr_{k,s_{k,h},a_{k,h}})^{-1}}]}\,.
	\label{eqn:sqrtpotential1}
\end{align}
Leveraging Lemma~\ref{lem:ellipticalpotential}, we can bound the elliptical potential inside the square-root as
\begin{align}
    \sum_{k=1}^T \sum_{h=1}^H [1 \wedge
	\|x\fr_k\|^2_{(N\fr_{k,s_{k,h},a_{k,h}})^{-1}}]
= & \sum_{s,a \in \saspace}\sum_{h=1}^H \sum_{k=1}^T
\one\{s = s_{k,h}, a = a_{k,h}\}[1 \wedge
		\|x\fr_k\|^2_{(N\fr_{k,s,a})^{-1}}]\\
		\leq &   \sum_{s,a \in \saspace} H 
		\sum_{k=1}^T  [1 \wedge
		\|x\fr_k\|^2_{(N\fr_{k,s,a})^{-1}} ]
		\leq \sum_{s,a \in \saspace} 2H \log \frac{\det N\fr_{k,s,a}}{\det \lambda I}\\
		\shortintertext{applying Equation~\eqref{eqn:logdetineq}}
		\leq &
		2\numS \numA H d\fr \log \frac{d\fr \lambda + \xi_{x\fr}^2 HT}{d\fr \lambda}\\
		\shortintertext{and applying the definition of $\alpha_T$}
		\leq &
		2\numS \numA H \frac{\alpha_T^2}{2(V^{\max}_1)^2} \leq \frac{ \numS \numA H \alpha_T^2}{(V^{\max}_1)^2}\,.
\end{align}
We plug this bound back in \eqref{eqn:sqrtpotential1} to get
\begin{align}
\lefteqn{\sum_{k=1}^T \sum_{h=1}^H [V_h^{\max} \wedge 2\alpha_{k,s_{k,h},a_{k,h}}\|x\fr_k\|_{(N\fr_{k,s_{k,h},a_{k,h}})^{-1}}]} \nonumber \\
	\leq &
	\sqrt{\frac{4\alpha_T^4 \numS \numA H^2 T}{(V^{\max}_1)^2}}
	= 
	\sqrt{\numS \numA H^2 T}\frac{2 \alpha_T^2}{V^{\max}_1}
	\\
	\leq &
	\sqrt{\numS \numA H^2 T} V_1^{\max} \left[2 +  8\lambda \xi^2_{\theta\fr} + 4 d\fr\log \frac{d\fr \lambda + \xi_{x\fr}^2 HT}{d\fr \lambda \delta'}\right].
	\label{eqn:term1_ifinal}
\end{align}
After deriving this upper bound on the term in Equation~\eqref{eqn:term1_ellipsoid_potential}, we bound the term in Equation~\eqref{eqn:term2_ellipsoid_potential} in similar fashion. We start with an upper bound on $\numS \gamma_{k, s_{k,h}, a_{k,h}}$ which holds for $k \leq T$:
\begin{align}
    \numS \gamma_{k, s_{k,h}, a_{k,h}}
            \leq &\sqrt{1 + 4\lambda \numS^2 \xi_{\theta\fp}^2 + 2 \numS^2 d\fp\log \frac{ d\fp \lambda + \xi_{x\fp}^2 HT }{d\fp \lambda \delta'} }
        =: \gamma_T \,, \label{eqn:gammaT}
\end{align}
which is by construction at least $1$. We now use this definition to bound as above
\begin{align}
    & \sum_{k=1}^T \sum_{h=1}^H [V_h^{\max} \wedge
		 3 V^{\max}_{h+1} \numS \gamma_{k,s_{k,h},a_{k,h}}\|x\fp_k\|_{(N\fp_{k,s_{k,h},a_{k,h}})^{-1}}]
		 \leq
	    3V_1^{\max} \gamma_T \sum_{k=1}^T \sum_{h=1}^H [1 \wedge
		 \|x\fp_k\|_{(N\fp_{k,s_{k,h},a_{k,h}})^{-1}}]	 \\
		 \leq &
	    3V_1^{\max} \gamma_T \sqrt{TH \sum_{k=1}^T \sum_{h=1}^H [1 \wedge
		 \|x\fp_k\|^2_{(N\fp_{k,s_{k,h},a_{k,h}})^{-1}}]}
		 \leq
	    3V_1^{\max} \gamma_T \sqrt{TH 2\numS \numA H d\fp \log \frac{d\fp \lambda + \xi^2_{x\fp} HT}{d\fp \lambda}}\\
	    		 \leq &
	    3V_1^{\max} \gamma_T \sqrt{2\numS \numA H^2 T  \frac{\gamma_T^2}{2\numS^2}}
	    \leq 
	    	\sqrt{\numS^3 \numA H^2 T} V_1^{\max} \left[3 +  12\lambda \xi^2_{\theta\fp} + 6 d\fp\log \frac{d\fp \lambda + \xi_{x\fp}^2 HT}{d\fp \lambda \delta'}\right].
	    	\label{eqn:term2_final}
\end{align}
Combining \eqref{eqn:term3_martingale}, \eqref{eqn:term1_ifinal} and \eqref{eqn:term2_final}, the cumulative certificates after $T$ episodes are bounded by
\begin{align}
    \sum_{k=1}^T \epsilon_k
    \leq & \sqrt{\numS^3 \numA H^2 T} V_1^{\max} \left[14 +  12\lambda (\xi^2_{\theta\fp}  + \xi^2_{\theta\fr}) + 6 (d\fp + d\fr)\log \frac{(d\fp + d\fr) \lambda + (\xi_{x\fp}^2 + \xi_{x\fr}^2) HT}{(d\fp \wedge d\fr) \lambda \delta'}\right]\\
    & + V_1^{\max} \numS H \sqrt{T  \log \frac{6 \log (2T)}{\delta'}}\\
    = & \tilde O \left(\sqrt{\numS^3 \numA H^2 T}V^{\max}_1 \lambda (\xi^2_{\theta\fp}  + \xi^2_{\theta\fr} + d\fp + d\fr)
    \log \frac{\xi^2_{x\fp}  + \xi^2_{x\fr}}{\lambda \delta}\right)\,.
\end{align}
\end{proof}

\subsection{Proof of Theorem~\ref{thm:ufolsic_certs}}
We are now ready to assemble the arguments above and prove the cumulative IPOC bound for Algorithm~\ref{alg:ofucontext}:
\begin{proof}
By Lemma~\ref{lem:context_failureprob}, the failure event $F$ has probability at most $\delta$. Outside the failure event, for every episode $k$, the upper and lower Q-value estimates are valid upper bounds on the optimal Q-function and lower bounds on the Q-function of the current policy $\pi_k$, respectively (Lemmas~\ref{lem:validupperboundcon} and \ref{lem:validlowerboundcon}). Further, Lemma~\ref{lem:certbound_context} shows that the cumulative certificates grow at the desired rate
\begin{align}
    \tilde O \left(\sqrt{\numS^3 \numA H^2 T}V^{\max}_1 \lambda (\xi^2_{\theta\fp}  + \xi^2_{\theta\fr} + d\fp + d\fr)
    \log \frac{\xi^2_{x\fp}  + \xi^2_{x\fr}}{\lambda \delta}\right).
\end{align}
\end{proof}

\subsection{Technical Lemmas}
We now state two existing technical lemmas used in our proof.
\begin{lemma}[Elliptical confidence sets; Theorem~20.1 in \citet{lattimore2018bandit}]
\label{lem:ellipticalconf}
Let $\lambda > 0$, $\theta \in \RR^d$ and $(r_i)_{i \in \NN}$ and $(x_i)_{i \in \NN}$ random processes adapted to a filtration $\Fcal_i$ so that $r_i - x_i^\top \theta$ are conditionally 1-sub-Gaussian. Then with probability at least $1 - \delta$ for all $k \in \NN$
\begin{align}
    \|\theta - \tilde \theta_k\|_{N_{k}(\lambda)} \leq 
    \sqrt{\lambda}\|\theta\|_2 + 
    \sqrt{2 \log \frac{1}{\delta} + \log \frac{\det(N_{k}(\lambda))}{\det(\lambda I)}}
\end{align}
where $N_{k}(\lambda) = \lambda I + \sum_{i = 1}^k x_{i} x_i^\top$ is the covariance matrix
and $\tilde \theta_k = N_k(\lambda)^{-1} \sum_{i=1}^k r_i x_i$ is the least-squares estimate.
\end{lemma}
\begin{lemma}[Elliptical potential; Lemma~19.1 in \citet{lattimore2018bandit}]
\label{lem:ellipticalpotential}
Let $x_1, \dots, x_n \in \RR^d$ with $L \geq \max_{i} \|x_i\|_2$ 
and $N_i = N_0 + \sum_{j=1}^i x_j x_j^\top$ with $N_0$ being psd. Then
\begin{align}
    \sum_{i=1}^n 1 \wedge \|x_i\|_{N_{i-1}^{-1}} 
    \leq 2 \log \frac{\det N_n}{\det N_0}
    \leq 2 d \log \frac{\tr(N_0) + n L^2}{d\det(N_0)^{1/d}}.
\end{align}
\end{lemma}
\section{Mistake IPOC Bound for Algorithm~\ref{alg:ofucontext}?}
\label{sec:nomistakebound}
By Proposition~\ref{prop:ipoc_properties_mis}, a mistake IPOC bound is stronger than the cumulative version we proved for Algorithm~\ref{alg:ofucontext}. One might wonder whether Algorithm~\ref{alg:ofucontext} also satisfies this stronger bound, but this is not the case:
\begin{proposition}
For any $\epsilon < 1$, there is an MDP with linear side information such that Algorithm~\ref{alg:ofucontext} outputs certificates $\epsilon_k \geq \epsilon$ infinitely often with probability $1$.
\label{prop:nomistakebound}
\end{proposition}

\begin{proof}
    Consider a two-armed bandit where the two-dimensional context is identical to the deterministic reward for both actions. The context alternates between $x_A := \begin{bmatrix}(1 + \epsilon) / 2 \\ (1 - \epsilon) / 2
    \end{bmatrix}$ and $x_B := \begin{bmatrix}(1 - \epsilon) / 2 \\ (1 + \epsilon) / 2
    \end{bmatrix}$. That means in odd-numbered episodes, the agent receives reward $\frac{1 + \epsilon}{2}$ for action $1$ and reward $\frac{1 - \epsilon}{2}$ for action $2$ (bandit A) and conversely in even-numbered episodes (bandit B).  Let $n_{A,i}$ and $n_{B, i}$ be the current number of times action $i$ was played in each bandit and $N_i = \operatorname{diag}(n_{A,i} + \lambda, n_{B, i} + \lambda)$ the covariance matrix.  One can show that
    the optimistic Q-value of action $2$ in bandit A is lower bounded as
    \begin{align}
        \Qub(2) \geq & \sqrt{\ln \det N_{2}} \|x_A\|_{N_2^{-1}} \wedge 1 \\
        = & \sqrt{\frac{\ln (\lambda + n_{A,2}) + \ln (\lambda + n_{B, 2})}{n_{A, 2}}}\wedge 1 .
        \label{eqn:bb1}
    \end{align}
    Assume now the agent stops playing action 2 in bandit A and playing action 1 in bandit B at some point. Then the denominator in Eq~\eqref{eqn:bb1} stays constant but the numerator grows unboundedly as $n_{B, 2} \rightarrow \infty$. That implies that $\Qub(2) \rightarrow 1$ but the optimistic Q-value for the other action $\Qub(1)  \rightarrow \frac{1 + \epsilon}{2} \leq 1$ approaches the true reward. Eventually $\Qub(2) > \Qub(1)$ and the agent will play the $\epsilon$-suboptimal action $2$ in bandit A again. Hence, Algorithm~\ref{alg:ofucontext} has to output infinitely many $\epsilon_k \geq \epsilon$.
\end{proof}

This negative result is due to the non-decreasing nature of the ellipsoid confidence intervals. It does not rule out alternative algorithms with mistake IPOC bounds for this setting, but they would likely require entirely different estimators and confidence bounds. 

\section{Additional Experimental Results}
\label{sec:experiemntalsetup}

\subsection{More Details on Experimental Results in Contextual Problems}
The results presented in the main paper are generated on the following MDP with side information. It has $\numS = 10$ states,  $\numA = 40$ actions, horizon of $H=5$, reward context dimension $d\fr=10$, and transition context dimension $d\fp = 1$. The transition context $x\fp_k$ is always constant $1$. We sample the reward parameters independently for all $s \in \statespace$, $a \in \actionspace$ and $i \in [d\fr]$ as
\begin{align}
    \theta\fr_{i, s, a} = X_{i,s,a} Y_{i,s,a}, \quad
    X_{i, s, a} \sim \operatorname{Bernoulli}(0.5), \quad
    Y_{i, s, a} \sim \operatorname{Unif}(0,1).
\end{align}
and the transition kernel for each $s \in \statespace, a \in \actionspace$ as
\begin{align}
    P(s,a) = \theta\fp_{s,a} \sim \operatorname{Dirichlet}(\alpha\fp)
\end{align}
where $\alpha\fp \in \RR^{\numS}$ with $\alpha\fp_i = 0.3$ for $i \in [\numS]$.
The reward context is again sampled from a Dirichlet distribution with parameter $\alpha\fr \in \RR^{d\fr}$ where $\alpha\fr_i = 0.01$ for $i \leq 4$ in the first $2$ million episodes and all other times $\alpha\fr_i = 0.7$. This shift in context distribution after $2$ million episodes simulates rare contexts becoming more frequent.

In addition, we applied Algorithm~\ref{alg:ofucontext} to randomly generated contextual bandit problems ($\numS = H = 1$) with $d\fr = 10$ dimensional context and $40$ actions.
We sample the reward parameters independently for all $s=1$, $a \in \actionspace$ and $i \in [d\fr]$ as
\begin{align}
    \theta\fr_{i, s, a} = X_{i,s,a} Y_{i,s,a}, \quad \quad X_{i, s, a} \sim \operatorname{Bernoulli}(0.9), \quad Y_{i, s, a} \sim \operatorname{Unif}(0,1).
\end{align}
The context in each episode is sampled from a Dirichlet distribution with parameter $\alpha \in \RR^{d\fr}$ where $\alpha_i = 0.7$ for $i \leq 7$ and $\alpha_i = 0.01$ for $i \geq 10$. This choice was made to simulate both frequent as well as a few rare context dimensions.
The \ofulsic algorithm was run for $8$ million episodes and we changed context, certificate and policy only every $1000$ episodes for faster experimentation. Figure~\ref{fig:bandit_res} shows certificates and optimality gaps of a representative run. Note that we sub-sampled the number of episodes shown for clearer visualization. 

Certificates and optimality gaps have a correlation of $0.88$ which confirms that certificates are informative about the policy's return. If one for example needs to intervene when the policy is more than $0.2$ from optimal (e.g., by reducing the price for that customer), then in more than $42\%$ 
of the cases where the certificate is above $0.2$, the policy is worse than $0.2$ suboptimal.

In both experiments, we use a slightly more complicated version of \ofulsic listed in Algorithm~\ref{alg:ofucontext_pen} which computes the optimistic and pessimistic Q estimates $\Qub$ and $\Qlb$ using subroutine ProbEstNorm in Algorithm~\ref{alg:l1shift}. For the sake of clarity, we presented a simplified version with the same guarantees in the main text. While this simplified version of \ofulsic does not leverage that the true transition kernel $P_{k}(s,a)$ has total mass $1$, Algorithm~\ref{alg:ofucontext_pen} adds this as a constraint (see Lemma~\ref{lem:probestnorm} below) similar to \citet{abbasi2014online}. This change yielded improved estimates empirically in our simulation. Note that this does not harm the theoretical properties. One can show the same cumulative IPOC bound for Algorithm~\ref{alg:ofucontext_pen} by slightly modifying the proof for Algorithm~\ref{alg:ofucontext}.
\begin{figure}[t]
    \centering
    \includegraphics[width=.5\textwidth]{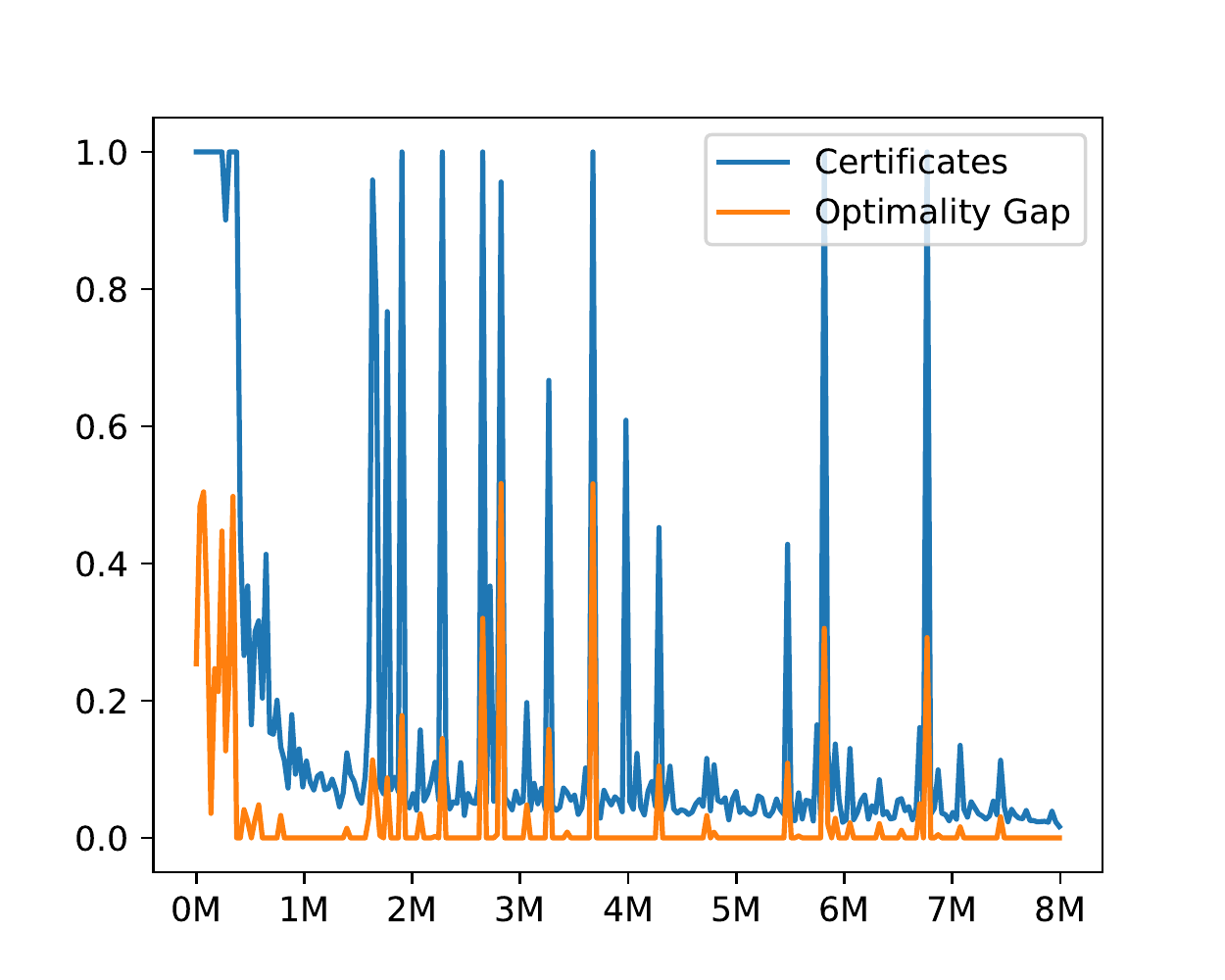}
    \caption{Results of \ofulsic for 8M episodes on a linear contextual bandit problem; certificates are shown in blue and the true (unobserved) optimality gap in orange for increasing number of episodes.}
    \label{fig:bandit_res}
\end{figure}

\begin{algorithm}[ht!]
\SetKwInOut{Inputa}{Input}
\Inputa{failure prob. $\delta \in (0,1]$, regularizer $\lambda > 0$}
$\forall s, s' \in \Scal, a \in \Acal, h \in [H]:$\\
$N\fp_{s,a} \gets \lambda I_{d\fp \times d\fp};\quad N\fr_{s,a} \gets \lambda I_{d\fr \times d\fr}$\;
$M\fr_{s,a} \gets \vec 0_{d\fr}; \quad M\fp_{s', s, a} \gets \vec 0_{d\fp}$\;
$\Vub_{H+1} \gets \vec 0_{\numS} \quad \Vlb_{H+1} \gets \vec 0_{\numS}\quad V^{\max}_{h} \gets (H-h+1)$\;
$\xi_{\theta\fr} \gets \sqrt{d}; ~ \xi_{\theta\fp} \gets \sqrt{d} \quad \delta' \gets \frac{\delta}{\numS (\numS\numA + \numA + H) }$\;
$\phi(N, x, \xi) \defeq \left[\sqrt{\lambda}\xi + \sqrt{\frac 1 2 \ln \frac{1 }{\delta'} + \frac 1 4 \ln \frac{\det N}{\det(\lambda I)}}\right]\|x\|_{N^{-1}}$
\label{lin:confsize_context2}
\;

\For{$k=1, 2, 3, \dots$}{
    Observe current contexts $x_k\fr$ and $x_k\fp$\;
    \tcc{estimate model with least squares}
    \For{$s,s' \in \Scal, a \in \Acal$}{
    $\hat \theta\fr_{s,a} \gets (N\fr_{s,a})^{-1} M\fr_{s,a}$\;
    $\hat r(s,a) \gets 0 \vee (x_k\fr)^\top \hat \theta\fr_{s,a} \wedge  1$\;
    $\hat \theta\fp_{s',s,a} \gets (N\fp_{s,a})^{-1} M\fp_{s',s,a}$\;
    $\hat P(s'|s,a) \gets 0 \vee (x_k\fp)^\top \hat \theta\fp_{s',s,a} \wedge 1$\;
    }
    \tcc{optimistic planning}
    \For{$h=H$ \KwTo $1$ \textrm{\textbf{and}} $s \in \statespace$}
        {
            \For{$a \in \actionspace$}
            {
            $\psiub_h(s,a) \gets 
            \phi(N\fr_{s,a},x\fr_k, \xi_{\theta\fr})$\;
            $\psilb_h(s,a) \gets 
            \phi(N\fr_{s,a},x\fr_k, \xi_{\theta\fr})$\;
            $\Qub_h(s,a) \gets \hat r(s,a) + \operatorname{ProbEstNorm}(\hat P(s,a), \phi(N\fp_{s,a},x\fp_k, \xi_{\theta\fp}), \Vub_{h+1}) + \psiub_h(s,a)$\;
            $\Qlb_h(s,a) \gets \hat r(s,a) -\operatorname{ProbEstNorm}(\hat P(s,a), \phi(N\fp_{s,a},x\fp_k, \xi_{\theta\fp}), -\Vlb_{h+1}) - \psilb_h(s,a)$\;
            \tcp{clip values}
            $\Qub_h(s,a) \gets 0 \vee \Qub_h(s,a) \wedge V^{\max}_h$\;
            $\Qlb_h(s,a) \gets 0 \vee \Qlb_h(s,a) \wedge V^{\max}_h$\;
             }
 $\pi_k(s, h) \gets \argmax_{a} \Qub_h(s, a)$\;
 $\Vub_h(s) \gets \Qub_h(s, \pi_k(s, t))$\;
 $\Vlb_h(s) \gets \Qlb_h(s, \pi_k(s, t)) $\;
}
    \tcc{Execute policy for one episode}
    $s_{k,1} \sim P_0$\; 
$\epsilon_k \gets \Vub_1(s_{k,1}) - \Vlb_1(s_{k,1})$\;
\textbf{output policy $\pi_k$ with certificate $\epsilon_k$}\;
    \For{$h=1$ \KwTo $H$}
    {
        $a_{k,h} \gets \pi_{k}(s_{k,h}, h)$\;
        $r_{k,h} \sim P_R(s_{k,h}, a_{k,h});\quad s_{k,h+1} \sim P(s_{k,h},a_{k,h})$\;
        \tcp{Update statistics}
        $N\fp_{s_{k,h},a_{k,h}} \gets N\fp_{s_{k,h}, a_{k,h}} + x_k\fp (x_k\fp)^\top$\;
        $N\fr_{s_{k,h},a_{k,h}} \gets N\fr_{s_{k,h}, a_{k,h}} + x_k\fr (x_k\fr)^\top$\;
        $M\fp_{s_{k,h+1}, s_{k,h}, a_{k,h}} \gets M\fp_{s_{k,h+1}, s_{k,h}, a_{k,h}} + x_k\fp$\;
        $M\fp_{s_{k,h}, a_{k,h}} \gets M\fp_{s_{k,h}, a_{k,h}} + x_k\fp$\;
    }
}

\caption{\ofulsic algorithm with probability mass constraints}
\label{alg:ofucontext_pen}
\end{algorithm}

\begin{algorithm}[t!]
\SetKwInOut{Inputa}{Input}
\SetKwInOut{Reta}{Return}

\Inputa{estimated probability vector $\hat p \in [0,1]^{\numS}$}
\Inputa{confidence width $\psi \in \RR_{+}$}
\Inputa{value vector $v \in \RR^{\numS}$}
Compute sorting $\sigma$ of $v$ so that $v_{\sigma_i} \geq v_{\sigma_j}$ for all $i \leq j$\;
$p \gets \hat p - \psi \vee 0$\;
$m \gets p^\top \one$\;
$r \gets 0$\;
\For{$i \in [\numS]$}
{
    $s \gets m \wedge ((\hat p_{\sigma_i} + \psi \wedge 1) - p_{\sigma_i})$\;
    $m \gets m - s$\;
    $r \gets r + v_{\sigma_i} (p_{\sigma_i} + s)$\;
}
\Reta{$r$}

\caption{ProbEstNorm($\hat p$, $\psi$, $v$) function to compute normalized estimated expectation of $v$}
\label{alg:l1shift}
\end{algorithm}

\begin{lemma}
\label{lem:probestnorm}
Let $\hat p \in [0, 1]^d$, $\psi \geq 0$ and $v \in \RR^d$ and define
$\Pcal_{\hat p} = \{ p \in [0,1]^d ~ : ~\leq \hat p - \psi \one_d \leq  p \leq \hat p + \psi \one_d  ~\wedge ~ \|p\|_1 = 1 \}$. Then, as long as $\Pcal_{\hat p} \neq \emptyset$, the value returned by Algorithm~\ref{alg:l1shift} satisfies
\begin{align}
    \operatorname{ProbEstNorm}(\hat p, \psi, v) =& \max_{p \in \Pcal_{\hat p}} p^\top v\\
    -\operatorname{ProbEstNorm}(\hat p, \psi, -v) =& \min_{p \in \Pcal_{\hat p}} p^\top v
\end{align}
and for any two $p, \tilde p \in \Pcal_{\hat p}$ it holds that $|p^\top v - \tilde p^\top v| \leq \|v\|_1 \| p - \tilde p \|_\infty = 2 \psi \|v\|_1$.
\end{lemma}

\subsection{Empirical Comparison of Sample-Efficiency in Tabular Environments}
\label{sec:tabular_experiments}
The simulation study above and in Section~\ref{sec:experiments} demonstrates that policy certificates can be a useful predictor for the (expected) performance of the algorithm in the next episode and the comparison of theoretical guarantees in Table~\ref{tab:tabbounds} indicates the improved sample-efficiency of the tabular algorithm \ulcr compared to existing approaches. However, do these tighter regret and PAC bounds indeed translate to an improved sample efficiency empirically? To answer this question, we compare \ulcr against \texttt{UCBVI-BF} and \texttt{UBEV}, the methods with tightest regret and PAC bounds respectively and which we expect to perform the best among existing approaches.

We evaluate the methods on tabular MDPs which are randomly generated as follows: With probability $0.85$ the average immediate reward $r(s,a)$ for any $(s,a)$ is $0$ and with probability $0.15$ it is drawn from a uniform distribution $r(s,a) \sim \textrm{Unif}[0,1]$. The transition kernel for each $P(s,a) \sim \textrm{Dirichlet}(0.1)$ are drawn from a Dirichlet distribution with parameter $\alpha=0.1$. 

Figure~\ref{fig:tabular_experiments} shows the performance of each method on MDPs with $\numS=20$ states and $\numA=4$ actions. The left plot shows the sum of rewards of each algorithm (averaged over a window of $1000$ episodes) on a problem with small horizon $H=10$. \ulcr\footnote{We use Algorithm~\ref{alg:ulcr_morecomplicated} with more refined bonuses compared to the simplified version in the main text.} converges to the optimal policy much faster than both \texttt{UCBVI-BF} and \texttt{UBEV}. Note that we adjusted \texttt{UBEV} to time-independent MDPs (the rewards and transition kernel do not depend on the time index within an episode) to make the comparison fair. On problems with larger horizon $H=50$ the performance gap between \ulcr and the competitors increases. Hence, even for problems of moderate horizon length compared $H \leq \numS \numA$, \ulcr outperforms \texttt{UCBVI-BF} despite both methods having minimax-optimal regret bounds (in the dominant term) in this case. This difference can likely be attributed to the tighter optimism bonuses of \ulcr as opposed to those of \texttt{UCBVI-BF} which are derived from an explicit regret-like bound with several additional approximations.

\begin{figure}[t]
    \centering
    \includegraphics[width=0.45\textwidth]{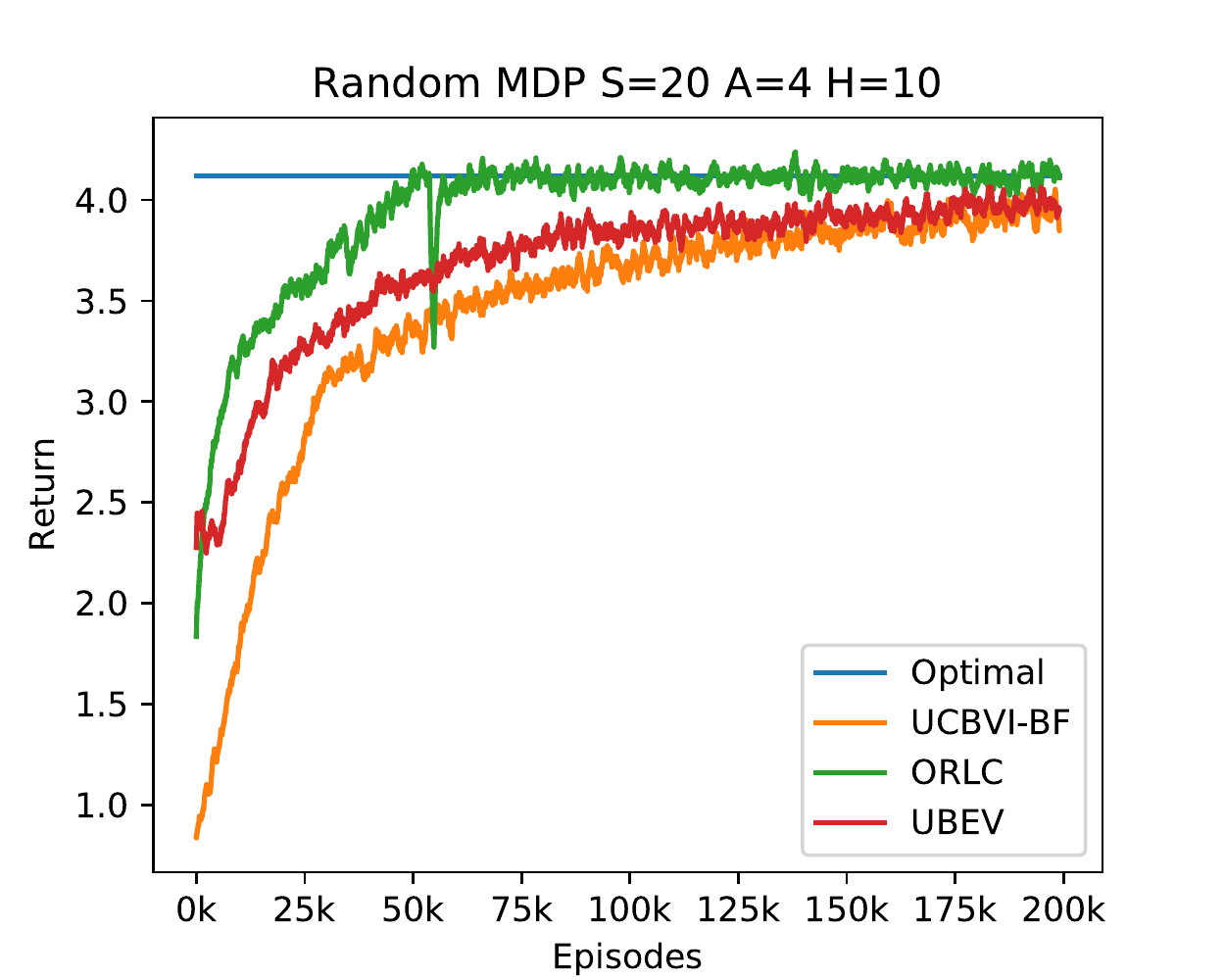}
        \includegraphics[width=0.45\textwidth]{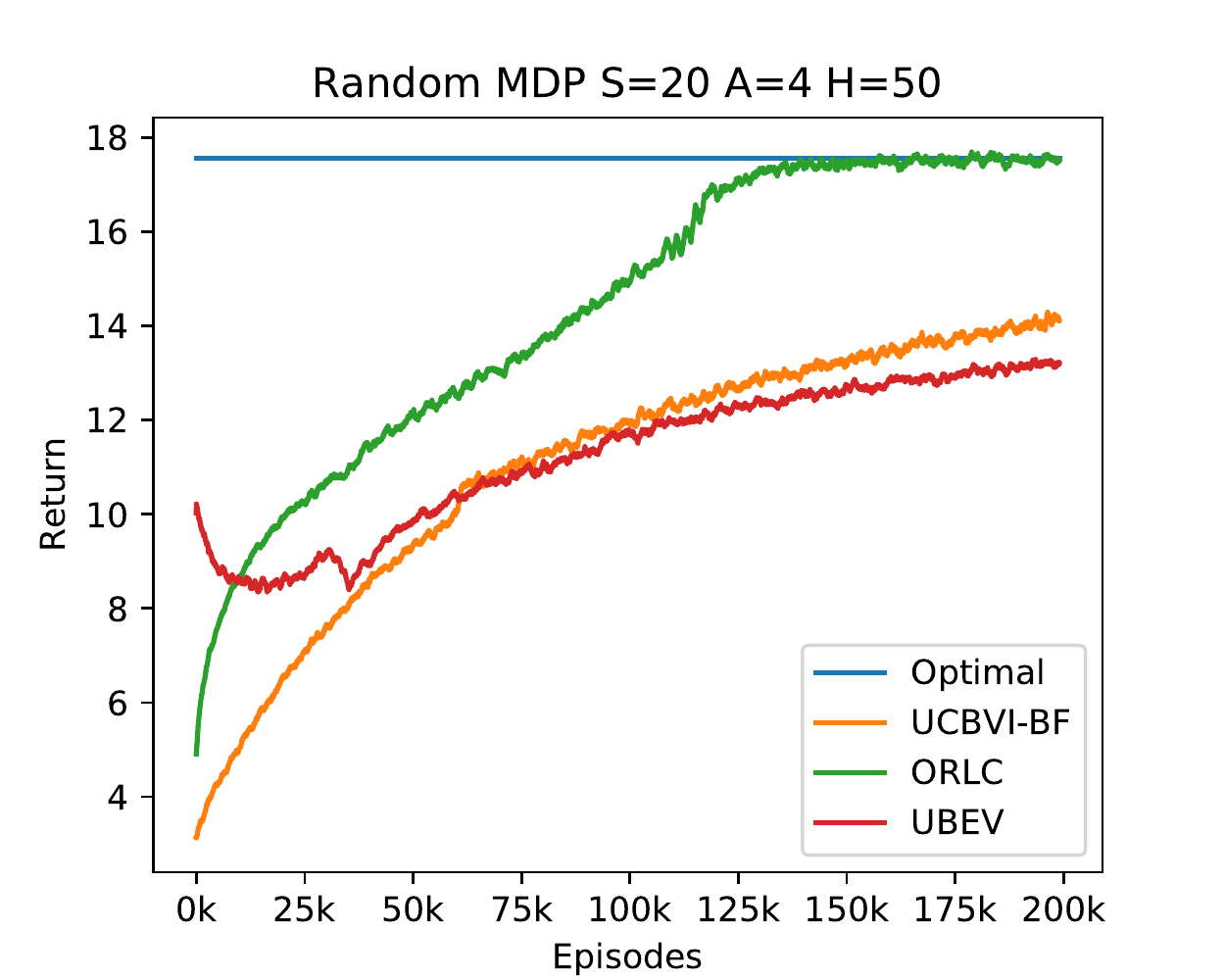}
    \caption{Experimental comparison of \ulcr and existing approaches. The graph show the achieved sum of rewards per episode averaged over $1000$ episodes each to generate smoothed curves. These results show representative single runs of each method on the same MDPs. Results are consistent across different random MDPs and different runs of the methods.}
    \label{fig:tabular_experiments}
\end{figure}

\subsection{Policy Certificates in Problems with no Context}
The simulation study above and in Section~\ref{sec:experiments} demonstrate that policy certificates can be a useful predictor for the optimality gap of the algorithm in the next episode. However, the experimental results only consider problems with context and might therefore wonder whether simple baselines that only consider a certainty measure over contexts can produce similar results. We would like to emphasize that this is not the case as such baselines can only detect performance drops due to unfamiliar contexts but are blind to performance drops due to exploration on familiar contexts. To illustrate this point, consider a multi-armed bandit problem without context. The algorithm periodically (and less and less frequently) plays suboptimal arms until it can be sufficiently certain that these arms cannot be optimal. Consider Figure~\ref{fig:mab_experiment} where we plot the optimality gaps and optimality certificates of \ulcr on a multi-armed bandit problem with $\numA = 100$ arms. We see occasional performance drop and as in the contextual case, our algorithm's certificates is able to predict them. A baseline that is only measuring the familiarity of context would completely fail in this case.

\begin{figure}[t]
    \centering
    \includegraphics[width=0.45\textwidth]{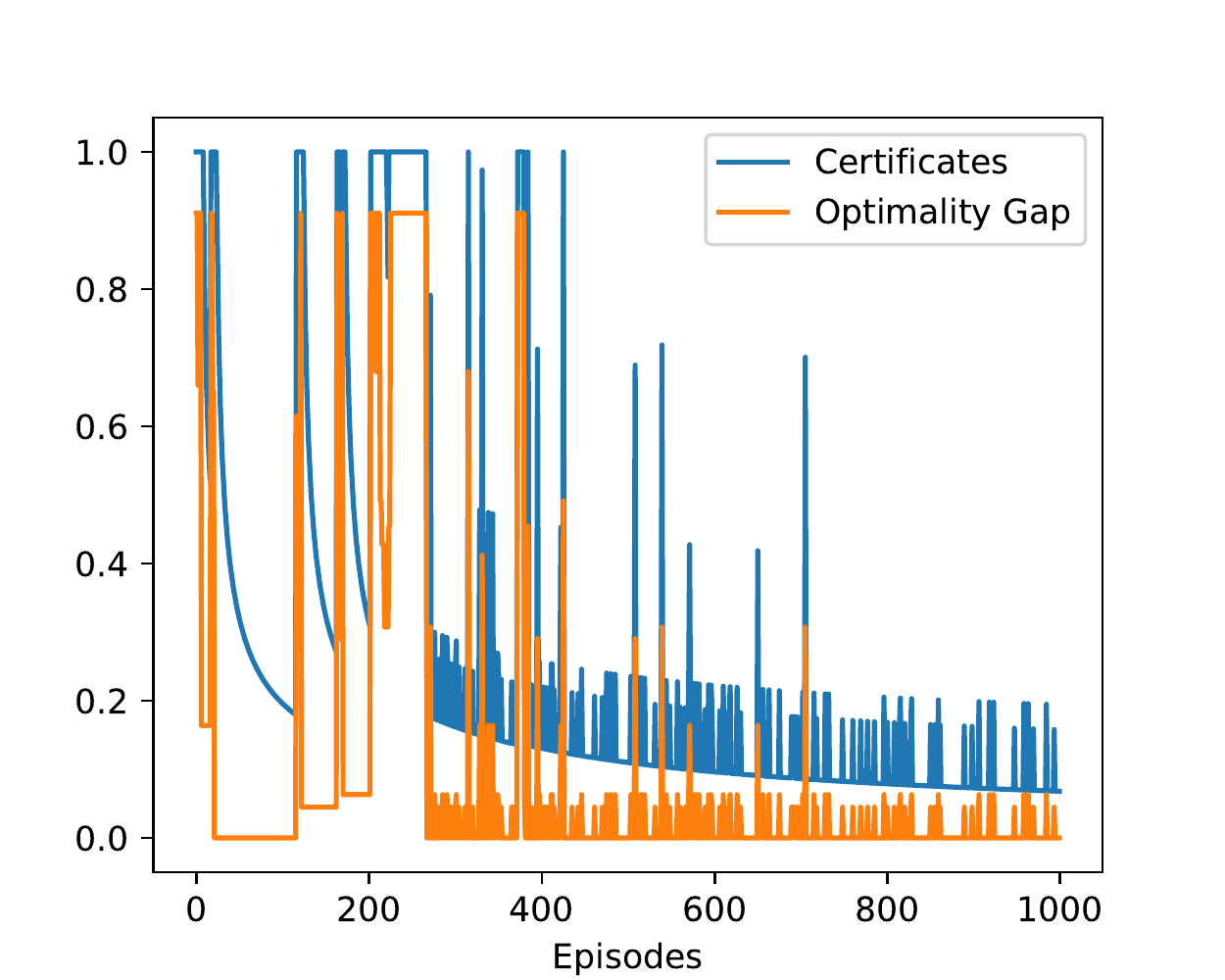}
    \caption{Performance and certificates of \ulcr on a multi-armed bandit problem with $100$ arms, generated randomly in the same way as tabular MDP instances above. Only every $10$th episode is plotted to improve visibility of individual spikes.}
    \label{fig:mab_experiment}
\end{figure}
\end{document}